\documentclass[twoside,11pt]{article}
\usepackage{eprint} 

\usepackage{amsmath}

\usepackage{graphicx}
\usepackage{subfigure} 

\usepackage{lscape}
\usepackage[normalem]{ulem}

\usepackage[vlined,linesnumbered,ruled,titlenotnumbered]{algorithm2e}
\usepackage[noend]{algpseudocode}

\usepackage{subfigure}
\usepackage{color}

\newcommand{\vect}[1]{\boldsymbol{\mathbf{#1}}}

\def\abovestrut#1{\rule[0in]{0in}{#1}\ignorespaces}
\def\belowstrut#1{\rule[-#1]{0in}{#1}\ignorespaces}
\def\abovespace{\abovestrut{0.20in}}

\def\belowspace{\belowstrut{0.10in}}

\firstpageno{1}

\begin{document}

\title{How to Center Binary Deep Boltzmann Machines}

\author{\name Jan Melchior \email Jan.Melchior@ini.rub.de \\ 
		\addr Theory of Neural Systems \\
		Institut f\"ur Neuroinformatik \\
		Ruhr Universit\"at Bochum\\
		44780 Bochum, Germany
\AND 
		\name Asja Fischer \email Asja.Fischer@rub.de \\ 
		\addr Theory of Machine Learning \\
		Institut f\"ur Neuroinformatik \\
		Ruhr Universit\"at Bochum\\
		44780 Bochum, Germany 
\AND 
		\name Laurenz Wiskott \email Laurenz.Wiskott@ini.rub.de \\ 
		\addr Theory of Neural Systems \\
		Institut f\"ur Neuroinformatik \\
		Ruhr Universit\"at Bochum\\
		44780 Bochum, Germany
}

\maketitle

\begin{abstract} 
This work analyzes centered binary Restricted Boltzmann Machines (RBMs) and binary Deep Boltzmann Machines (DBMs), where centering is done by subtracting offset values from visible and hidden variables. 
We show analytically that (i) centering results in a different but equivalent parameterization for artificial neural networks in general, (ii) the expected performance of centered binary RBMs/DBMs is invariant under simultaneous flip of data and offsets, for any offset value in the range of zero to one, (iii) centering can be reformulated as a different update rule for normal binary RBMs/DBMs, and (iv) using the enhanced gradient is equivalent to setting the offset values to the average over model and data mean. 
Furthermore, numerical simulations suggest that (i) optimal generative performance is achieved by subtracting mean values from visible as well as hidden variables, (ii) centered RBMs/DBMs reach significantly higher log-likelihood values than normal binary RBMs/DBMs, (iii) centering variants whose offsets depend on the model mean, like the enhanced gradient, suffer from severe divergence problems, (iv) learning is stabilized if an exponentially moving average over the batch means is used for the offset values instead of the current batch mean, which also prevents the enhanced gradient from diverging, (v) centered RBMs/DBMs reach higher LL values than normal RBMs/DBMs while having a smaller norm of the weight matrix,
(vi) centering leads to an update direction that is closer to the natural gradient and that the natural gradient is extremly efficient for training RBMs, (vii) centering dispense the need for greedy layer-wise pre-training of DBMs, (viii) furthermore we show that pre-training often even worsen the results independently whether centering is used or not, and (ix) centering is also beneficial for auto encoders.
\end{abstract} 


\section{Introduction}

In the last decade Restricted Boltzmann Machines (RBMs) got into the focus of attention because they can be considered as building blocks of deep neural networks~\citep{HintonOsinderoEtAl-2006a, Bengio-2009a}. 
RBM training methods are usually based on gradient ascent on the log-Likelihood (LL) of the model parameters given the training data. 
Since the gradient is intractable, it is often approximated using Gibbs sampling only for a few steps~\citep{HintonOsinderoEtAl-2006a, Tieleman-2008a, TielemanHinton-2009a}

Two major problems have been reported when training RBMs.
Firstly, the bias of the gradient approximation introduced by using only a few steps of Gibbs sampling may lead to a divergence of the LL during training \citep{FischerIgel-2010a, Schulz2010}. 
To overcome the divergence problem, \citet{DesjardinsCourvilleEtAl-2010a} and \citet{cho:10} have proposed to use parallel tempering, which is an advanced sampling method that leads to a faster mixing Markov chain and thus to a better approximation of the LL gradient. 

Secondly, the learning process is not invariant to the data representation. For example training an RBM on the \emph{MNIST} dataset leads to a better model than training it on \emph{1-MNIST} (the dataset generated by flipping each bit in \emph{MNIST}).  
This is due to missing invariance properties of the gradient with respect to these flip transformations and not due to the model's capacity, since an RBM trained on \emph{MNIST} can be transformed in such a way  that it models \emph{1-MNIST} with the same LL. 
Recently, two approaches have been introduced that address the invariance problem. 
The enhanced gradient \citep{ChoRaikoEtAl-2011a,ChoRaikoEtAl-2013} has been designed as an invariant alternative to the true LL gradient of binary RBMs and has been derived by calculating a weighted average over the gradients one gets by applying any possible bit flip combination on the dataset. 
Empirical results suggest that the enhanced gradient leads to more distinct features and thus to better classification results based on the learned hidden representation of the data. 
Furthermore, in combination with an adaptive learning rate the enhanced gradient leads to more stable training in the sense that good LL values are reached independently of the initial learning rate. 
 \citet{TangSutskever-2011}, on the other hand have shown empirically that subtracting the data mean from the visible variables leads to a model that can reach similar LL values on the \emph{MNIST} and the \emph{1-MNIST} dataset and comparable results to those of the enhanced gradient.\footnote{Note, that changing the model such that the mean of the visible variables is removed is not equivalent to just removing the mean of the data.}
Removing the mean from all variables is known as the `centering trick' which was originally proposed for feed forward neural networks~\citep{LeCun1998,Schraudolph98}. 
It has recently also been applied to the visible and hidden variables of Deep Boltzmann Machines~(DBM)~\citep{MontavonMueller-2012} where it has been shown to lead to an initially better conditioned optimization problem. 
Furthermore, the learned features have shown better discriminative properties and centering has improved the generative properties of locally connected DBMs.
A related approach applicable to multi-layer perceptrons where the activation functions of the neurons are transformed to have zero mean and zero slope on average was proposed by \citet{RaikoValpolaEtAl-2012}. 
The authors could show that the gradient under this transformation became closer to the natural gradient, which is desirable since the natural gradient follows the direction of steepest ascent in the manifold of probability distributions. 
Furthermore, the natural gradient is independent of the concrete parameterization of the distributions and is thus clearly the update direction of choice~\citep{Amari-1998}. However, it is intractable already for rather small RBMs. 
\citet{Schwehn-2010} and \citet{OllivierArnoldEtAl-2013} trained binary RBMs and \citet{DesjardinsPascanuEtAl-2013} binary DBMs using approximations of the natural gradient obtained by Markov chain Monte Carlo methods. 
Despite the theoretical arguments for using the natural gradient, the authors concluded that the computational overhead is extreme and it is rather questionable that the natural gradient is efficient for training RBMs or DBMs. In this work we show that the natural gradient, if tractable is extremely efficient for training RBMs and that centering leads to an update direction that is closer to the natural gradient. This was already part of the previous version in \citep{Fischer2014} and was recently confirmed by \citet{Grosse2015}.
Another related contribution lately proposed by \citet{Ioffe2015} for feed forward neural networks (batch normalization) aims to remove first and second order statistics in the network.

In this work\footnote{Previous versions of this work have been published as eprint~\citep{Melchior2013} and as part of the thesis~\citep{Fischer2014}.}
 we give a unified view on centering where we analysis in particular the properties and performance of centered binary RBMs and DBMs.
We begin with a brief overview over binary RBMs, the standard learning algorithms, the natural gradient of the LL of RBMs, and the basic ideas used to construct the enhanced gradient in Section~\ref{sec:RBMintro}. 
In Section~\ref{sec:CenteredRBMs}, we discuss the theoretical properties of centered RBMs, show that centering can be reformulated as a different update rule for normal binary RBMs, that the enhanced gradient is a particular form of centering and finally that centering in RBMs and its properties naturally extend to DBMs. Furthermore, in Section~\ref{sec:generalization}, we show that centering is an alternative parameterization for arbitrary Artificial Neural Networks (ANNs) in general and we discusses how the parameters of centered and normal binary ANNs should be initialized. 
Our experimental setups are described in Section~\ref{sec:Experiments} before we empirically analyze the performance of centered RBMs with different initializations, offset parameters, sampling methods, and learning rates in Section~\ref{sec:Results}. The empirical analysis includes experiments on 10 real world problems, a comparison of the centered gradient with the natural gradient, and experiments on Deep Boltzmann machines and Auto encoders (AEs).
Finally our work is concluded in Section~\ref{sec:conclusion}.

\section{Restricted Boltzmann Machines}\label{sec:RBMintro}
An RBM \citep{Smolensky-1986a} is a bipartite undirected graphical 
model with a set of \emph{N} visible and \emph{M} hidden variables taking values 
$\vect{x} = (x_1, ... ,x_N)$ 
and 
$\vect{h} = (h_1, ... ,h_M)$,
respectively.
Since an RBM is a Markov random field, its joint probability distribution is given by a Gibbs distribution
\begin{eqnarray}
p\left( \vect{x},\vect{h} \right) &=& \frac{1}{Z}\mathrm{e}^{-E(\vect{x},\vect{h})}\enspace,\nonumber
\end{eqnarray}
with partition function $Z$ and energy $E(\vect{x},\vect{h})$. 
For binary RBMs,  
$\vect{x} \in \lbrace 0,1 \rbrace^N$ 
,  
$\vect{h} \in \lbrace 0,1 \rbrace^M$, and
the energy, which defines the bipartite structure, is given by
\begin{eqnarray}
  E\left( \vect{x}, \vect{h} \right)
  &=&
  	- \vect{x}^T\vect{b} - \vect{c}^T\vect{h} - \vect{x}^T\vect{W}\vect{h}\enspace,\nonumber
\end{eqnarray}
where the weight matrix $\vect{W}$, the visible bias vector $\vect{b}$ and the hidden bias vector $\vect{c}$ are the parameters of the model, jointly denoted by $\vect{\theta}$. 
The partition function which sums over all possible visible and hidden states is given by 
\begin{eqnarray}
Z &=& {\sum_{\vect{{x}}} \sum_{\vect{{h}}} {\mathrm{e}^{-E\left(\vect{{x}},\vect{{h}} \right)}}}\enspace.\nonumber
\end{eqnarray}

RBM training is usually based on gradient ascent using approximations of the LL gradient
\begin{eqnarray}
	\nabla\vect{\theta}
	  &=&
	\frac{
	\partial \left\langle \log \left( p(\vect{x} | \vect{\theta}) \right) \right\rangle_d
	}{
		  \partial \vect{\theta}
	  } 	   
	  =
	  -\left\langle
 			\frac{\partial E( \vect{x},\vect{h} )} 
 				{ \partial \vect{\theta}}\right\rangle_d
	  + \left\langle
 			\frac{\partial E( \vect{x},\vect{h} )} 
 				{ \partial \vect{\theta}}\right\rangle_m \enspace,\nonumber
\end{eqnarray}
where $\left\langle \cdot \right\rangle_m$ is the expectation under $p(\vect{h},\vect{x})$ and 
$\left\langle \cdot \right\rangle_d$ is the expectation under $p(\vect{h}|\vect{x})p_e(\vect{x})$ with empirical distribution $p_e$. 
We use the notation $\nabla\vect{\theta}$ for the derivative of the LL with respect to $\vect{\theta}$ in order to be consistent with the notation used by \citet{ChoRaikoEtAl-2011a}.
For binary RBMs the gradient becomes  
\begin{eqnarray}
	\nabla\vect{W}
	  &=&  \langle \vect{x}\vect{h}^T\rangle_d - \langle \vect{x}\vect{h}^T\rangle_m
	\enspace ,\nonumber\\
	\nabla\vect{b}
	  &=&  \langle \vect{x}\rangle_d - \langle \vect{x}\rangle_m
	\enspace ,\nonumber\\
	\nabla\vect{c}
	  &=&  \langle \vect{h}\rangle_d - \langle \vect{h}\rangle_m
	\enspace\nonumber.
\end{eqnarray}
Common RBM training methods approximate $\left\langle \cdot \right\rangle_m$ by samples gained by different Markov chain Monte Carlo methods. 
Sampling $k$ (usually $k=1$) steps from a Gibbs chain initialized with a data sample yields the Contrastive Divergence (CD)~\citep{HintonOsinderoEtAl-2006a} algorithm. 
In stochastic maximum likelihood~\citep{Younes-1991}, in the context of RBMs also known as Persistent Contrastive Divergence (PCD)~\citep{Tieleman-2008a}, the chain is not reinitialized with a data sample after parameter updates. 
This has been reported to lead to  better gradient approximations if the learning rate is chosen sufficiently small. 
Fast Persistent Contrastive Divergence (FPCD) ~\citep{TielemanHinton-2009a} tries to further speed up learning by introducing an additional set of parameters, which is only used for Gibbs sampling during learning.
The advanced sampling method Parallel Tempering (PT) introduces additional `tempered'
Gibbs chains corresponding to smoothed versions of $p(\vect x,\vect h)$. The energy of these distributions is multiplied by $\frac{1}{T}$, where $T$ is referred to as temperature. 
The higher the temperature of a chain is, the `smoother'  the corresponding distribution and the faster the chain mixes. 
Samples may swap between chains with a probability given by the Metropolis Hastings ratio, which leads to better mixing of the original chain (where $T=1$). 
We use PT$_{c}$ to denote the RBM training algorithm that uses Parallel Tempering with $c$ tempered chains as a sampling method.
Usually only one step of Gibbs sampling is performed in each tempered chain before allowing samples to swap, and a deterministic even odd algorithm~\citep{MartinDenschlagEtAl-2009} is used as a swapping schedule.
PT$_{c}$ increases the mixing rate and has been reported to achieve better gradient approximations than CD and (F)PCD ~\citep{DesjardinsCourvilleEtAl-2010a,cho:10} with the drawback of having a higher computational cost.

See the introductory paper of \citet{fischer:14} for a recent review of RBMs and their training algorithms.

\subsection{Enhanced Gradient}\label{sec:Enhanced_Gradient}

\citet{ChoRaikoEtAl-2011a} proposed a different way to update  parameters during training of binary RBMs, which is invariant to the data representation. 

When transforming the state $(\vect{x},\vect{h})$ of a binary RBM by flipping some of its variables (that is $\tilde{x_i}=1-x_i$ and $\tilde{h_j}=1-h_j$ for some $i,j$), yielding a new state $(\vect{\tilde{x}},\vect{\tilde{h}})$, one can transform the parameters $\vect{\theta}$ of the RBM to $\vect{\tilde{\theta}}$ such that $E(\vect{x},\vect{h}|\vect{\theta}) = E(\vect{\tilde{x}},\vect{\tilde{h}}|\vect{\tilde{\theta}}) + const $ and thus $p(\vect{x},\vect{h}|\vect{\theta}) = p(\vect{\tilde{x}},\vect{\tilde{h}}|\vect{\tilde{\theta}})$ holds. 
However, if we update the parameters of the  transformed model based on the corresponding LL gradient to $\vect{\tilde{\theta}}' = \vect{\tilde{\theta}} + \eta \nabla \vect{\tilde{\theta}}$ and apply the inverse parameter transformation to $\vect{\tilde{\theta}}'$, the result will differ from $\vect{\theta}' = \vect{\theta} + \eta \nabla \vect{\theta}$. 
The described procedure of transforming, updating, and transforming back can be regarded as a different way to update $\vect \theta$.

Following this line of thought there exist $2^{N+M}$ different parameter updates corresponding to the $2^{N+M}$ possible binary flips of $(\vect{x},\vect{h})$.  
\citet{ChoRaikoEtAl-2011a} proposed the enhanced gradient as a weighted sum of these  $2^{N+M}$  parameter updates, which for their choice of weighting is given by
\begin{eqnarray}
\nabla_e \vect{W}
	 &=&
	  \langle ( \vect{x} - \langle  \vect{x} \rangle_{d} ) (\vect{h} - \langle  \vect{h} \rangle_{d})^T \rangle_{d} - \langle ( \vect{x} - \langle  \vect{x} \rangle_{m} ) (\vect{h} - \langle  \vect{h} \rangle_{m})^T \rangle_{m}
	\enspace,  \label{enhanced_W}  \\
\nabla_e \vect{b}
	  &=&
	  \langle  \vect{x} \rangle_{d} - \langle  \vect{x} \rangle_{m} - \nabla_e \vect{W} \frac{1}{2} \left(\langle  \vect{h} \rangle_{d} + \langle  \vect{h} \rangle_{m} \right)
	\enspace, \label{enhanced_b} \\ 
\nabla_e \vect{c}
	  &=&
	  \langle \vect{h}  \rangle_{d} - \langle \vect{h} \rangle_{m} - {\nabla_e \vect{W}}^T \frac{1}{2} \left(\langle  \vect{x} \rangle_{d} + \langle  \vect{x} \rangle_{m} \right)\enspace. \label{enhanced_c}  
\end{eqnarray}
It has been shown that the enhanced gradient is invariant to arbitrary bit flips of the variables and therefore invariant under the data representation, which has been demonstrated on the \emph{MNIST} and \emph{1-MNIST} dataset. 
Furthermore, the authors reported more stable training under various settings in terms of the LL estimate and classification accuracy.

\subsection{Natural Gradient}\label{sec:theoNatGrad}

Following the direction of steepest ascent in the Euclidean parameter space 
(as given by the standard gradient) does not necessarily correspond to the direction of steepest ascent in the manifold of probability distributions $\{p(\vect{x} | \vect{\theta}), \vect{\theta} \in \Theta\}$, which we are actually interested in. 
To account for the local geometry of the manifold, the Euclidean metric should be replaced by the Fisher information metric defined by 
$|| \vect{\theta}  ||_{\mathcal{I}(\vect\theta)} = \sqrt{\sum \theta_k \mathcal{I}_{kl}\left( \vect \theta \right) \theta_l}$, where $\mathcal{I}(\vect\theta)$ is the 
Fisher information matrix \citep{Amari-1998}.
The $kl$-th entry of the Fisher information matrix for a parameterized distribution $p(\vect{x} | \vect{\theta})$ is given by
\begin{eqnarray}
	\vect{ \mathcal{I}}_{kl}\left( \theta \right)
	  &=&
	  \left\langle
	  \left( 
	\frac{
	\partial  \log \left( p(\vect{x} | \vect{\theta}) \right)
	}{
		  \partial \theta_k
	  } 	
	  \right)
	  \left(  
\frac{
	\partial  \log \left( p(\vect{x} | \vect{\theta}) \right)
	}{
		  \partial \theta_l
	  } 	  
	  \right)
	  \right\rangle_m \enspace \nonumber,
\end{eqnarray}
where $\langle \cdot \rangle_m$ denotes the expectation under $p(\vect{x} | \vect{\theta})$.
The gradient associated with the Fisher metric is called the natural gradient and is given by
\begin{eqnarray}
	\nabla_n\vect{\theta}
	  &=& \vect{ \mathcal{I}}\left( \theta \right)^{-1} \nabla\vect{\theta} \enspace \nonumber.
\end{eqnarray}
The natural gradient points in the direction $\delta \vect \theta$ achieving the largest change of the objective function (here the LL) for an infinitesimal small distance $\delta \vect \theta$ between $p(\vect{x} | \vect{\theta})$ and $p(\vect{x} | \vect{\theta +\delta \vect \theta })$ in terms of the Kullback-Leibler divergence \citep{Amari-1998}. 
This makes the natural gradient independent of the parameterization including the invariance to flips of the data as a special case. 
Thus, the natural gradient is clearly the update direction of choice.

For binary RBMs the entries of the Fisher information matrix~\citep{AmariKurataEtAl-1992,DesjardinsPascanuEtAl-2013,OllivierArnoldEtAl-2013} are given by
\begin{eqnarray}
	\vect{ \mathcal{I}}_{w_{ij},w_{uv}}\left( \theta \right) = \vect{ \mathcal{I}}_{,w_{uv},w_{ij}}\left( \theta \right)
	  &=&
	  \langle x_i h_j x_u h_v \rangle_m - \langle x_u h_v \rangle_m \langle x_u h_v \rangle_m \nonumber\\
	  &=&
	  Cov_m \left( x_i h_j , x_u h_v \right)  \enspace,\nonumber\\
	\vect{ \mathcal{I}}_{w_{ij},b_{u}}\left( \theta \right) = \vect{ \mathcal{I}}_{b_{u}, w_{ij}}\left( \theta \right)
	  &=&
	  Cov_m \left( x_i h_j , x_u \right) \enspace,\nonumber\\
\vect{ \mathcal{I}}_{w_{ij},c_{v}}\left( \theta \right) = \vect{ \mathcal{I}}_{c_{v}, w_{ij}}\left( \theta \right)
	  &=&
	  Cov_m \left( x_i h_j , h_v \right)\enspace, \nonumber\\
	\vect{ \mathcal{I}}_{b_{i},b_{u}}\left( \theta \right) = \vect{ \mathcal{I}}_{b_{u},b_{i}}\left( \theta \right)
	&=&
	  Cov_m \left( x_i  , x_u \right)\enspace, \nonumber\\
	\vect{ \mathcal{I}}_{c_{j},c_{v}}\left( \theta \right) = \vect{ \mathcal{I}}_{c_{v},c_{j}}\left( \theta \right)
	  &=&
	  Cov_m \left( h_j  , h_v \right) \enspace\nonumber.
\end{eqnarray}
Since these expressions involve expectations under the model distribution they are not tractable in general, but can be approximated using MCMC methods \citep{OllivierArnoldEtAl-2013,DesjardinsPascanuEtAl-2013}. 
Furthermore, a diagonal approximation of the Fisher information matrix could be used. 
However, the approximation of the natural gradient is still computationally very expensive so that the practical usability remains questionable \citep{DesjardinsPascanuEtAl-2013}.

\section{Centered Restricted Boltzmann Machines}\label{sec:CenteredRBMs}

Inspired by the centering trick proposed by \citet{LeCun1998}, \citet{TangSutskever-2011} have addressed the flip-invariance problem by changing the energy of the RBM in a way that the mean of the input data is removed. 
\citet{MontavonMueller-2012} have extended the idea of centering to the visible and hidden variables of DBMs and have shown that centering improves the conditioning of the underlying optimization problem, leading to models with better discriminative properties for DBMs in general and better generative properties in the case of locally connected DBMs.

Following their line of thought, the energy for a centered binary RBM where the visible and hidden variables are shifted by the offset parameters $\vect{\mu} = (\mu_1,\dots ,\mu_N)$ and $\vect{\lambda} = (\lambda_1,\dots ,\lambda_M)$, respectively, can be formulated as 
\begin{eqnarray}
  E\left( \vect{x}, \vect{h} \right)
  &=&
  	- \left(\vect{x} - \vect{\mu} \right)^T\vect{b}
  	- \vect{c}^T\left(\vect{h} - \vect{\lambda} \right)
  	- \left(\vect{x} - \vect{\mu} \right)^T\vect{W}\left(\vect{h} - \vect{\lambda} \right)
  \label{eqn:energy}\enspace.
\end{eqnarray}
By setting both offsets to zero one retains the normal binary RBM. 
Setting $\vect{\mu} = \langle \vect{x}  \rangle_{d}$ and $\vect{\lambda} = \vect{0}$ leads to the model introduced by \citet{TangSutskever-2011}, and by setting  $\vect{\mu} = \langle \vect{x}  \rangle_{d}$ and $\vect{\lambda} = \langle \vect{h}  \rangle_{d}$ we get a shallow variant of the centered DBM analyzed by \citet{MontavonMueller-2012}. 

The conditional probabilities for a variable taking the value one are given by
\begin{eqnarray}
p\left(x_i = 1 | \vect{h}\right)
	 & = &
	   \sigma(\vect{w_{i*}}\left(\vect{h}-\vect{\lambda}\right) + b_i)\enspace ,
	\label{eqn:cond_x} \\
p\left(h_j = 1 | \vect{x}\right)
	 & = &
	  \sigma(\left(\vect{x}-\vect{\mu}\right)^T \vect{w_{*j}} + c_j)\enspace,\label{eqn:cond_h} 
\end{eqnarray}
where $ \sigma\left(\cdot\right)$ is the sigmoid function, $\vect{w_{i*}}$ represents the $i$th row, and $\vect{w_{*j}}$ the $j$th column of the weight matrix $\vect{W}$.

The LL gradient now takes the form
\begin{eqnarray}
\nabla \vect{W}
	 & =  &
	  \langle ( \vect{x} - \vect{\mu} ) (\vect{h} - \vect{\lambda})^T \rangle_{d}  
	  - \langle ( \vect{x} - \vect{\mu} ) (\vect{h} - \vect{\lambda})^T \rangle_{m} \enspace,
	\label{eqn:grad_W} \\
\nabla \vect{b}
	  & =& 
	  \langle  \vect{x} - \vect{\mu}  \rangle_{d} - \langle  \vect{x} - \vect{\mu}  \rangle_{m} = \langle  \vect{x} \rangle_{d} - \langle  \vect{x} \rangle_{m}\enspace,
	\label{eqn:grad_b} \\
\nabla \vect{c}
	  &  =  &
	  \langle \vect{h} - \vect{\lambda}  \rangle_{d} - \langle \vect{h} - \vect{\lambda} \rangle_{m}
	  =
	  \langle \vect{h}  \rangle_{d} - \langle \vect{h} \rangle_{m}\enspace. 
	\label{eqn:grad_c} 
\end{eqnarray}
$\nabla \vect{b}$ and $\nabla \vect{c}$ are independent of the choice of $ \vect{\mu} $ and $\vect{\lambda}$ and thus centering only affects $\nabla \vect{W}$.  
It can be shown (see  Appendix~\ref{sec:ProofInv}) that the gradient of a centered RBM is invariant to flip transformations if a flip of $x_i$ to $1-x_i$ implies a change of $\mu_i$ to $1-\mu_i$, and a flip $h_j$ to $1-h_j$ implies a change of $\lambda_j$ to $1-\lambda_j$. 
This obviously holds for $\mu_i = 0.5$, $\lambda_j=0.5$ but also for any expectation value over $x_i$ and $h_j$ under any distribution. 
Moreover, if the offsets are set to an expectation centered RBMs get also invariant to shifts of variables (see Section~\ref{sec:generalization}). Note that the properties of centered RBMs naturally extend to centered DBMs (see Section~\ref{sec:dbm}).

If we set $\vect{\mu} $ and $\vect{\lambda}$ to the expectation values of the variables, these values may depend on the RBM parameters (think for example about $ \langle \vect{h}  \rangle_{d}$) and thus they might change during training. 
Consequently, a learning algorithm for centered RBM needs to update the offset values 
to match the expectations under the distribution that has changed with a parameter update.
When updating the offsets one needs to transform the RBM parameters such that the modeled probability distribution stays the same.
An RBM with offsets $\vect{\mu}$ and $\vect{\lambda}$ can be transformed to an RBM with offsets $\vect{\mu}'$ and $\vect{\lambda}'$ by
\begin{eqnarray}
  \vect{W}' &=& \vect{W}\enspace,
  \label{eqn:transform_W}\\
  \vect{b}' &=& \vect{b} + \vect{W} \left(\vect{\lambda}' - \vect{\lambda}\right)\enspace,
  \label{eqn:transform_b}\\
  \vect{c}' &=& \vect{c} + \vect{W}^T \left(\vect{\mu}' - \vect{\mu}\right)\enspace,
  \label{eqn:transform_c}
\end{eqnarray}
such that $E(\vect{x},\vect{h}|\vect{\theta},\vect{\mu},\vect{\lambda}) = E(\vect{x},\vect{h}|\vect{\theta}',\vect{\mu}',\vect{\lambda}') + const $, is guaranteed. 
Obviously, this can be used to transform a centered RBM to a normal RBM and vice versa, highlighting that centered and normal RBMs are just different parameterizations of the same model class.

If the intractable model mean is used for the offsets, they have to be approximated by samples.
Furthermore, when  $\vect \lambda$ is chosen to be $\langle \vect h \rangle_d$ or $\langle\vect  h \rangle_m$ or 
when  $\vect \mu$ is chosen to be $\langle \vect x \rangle_m$ 
one could either approximate the mean values using the sampled states or the corresponding conditional probabilities. 
But due to the Rao–-Blackwell theorem an estimation based on the probabilities has lower variance and therefore is the approximation of choice\footnote{This can be proven analogously to the proof of proposition 1 in the work of \citet{SwerskyChenEtAl-2010a}.}.

Algorithm~\ref{alg:algorithm2} shows pseudo code for training a centered binary RBM, where we use
 $\langle \cdot \rangle$ to denote the average over samples from the current batch. 
Thus, for example, we write $\langle \vect x_d \rangle$ for  the average value of data samples $\vect x_d$ in the current batch, which is used as an approximation for the expectation of $\vect x$ under the data distribution that is $\langle x \rangle_d$. Similarly, $\langle \vect h_d \rangle$ approximates $\langle \vect h \rangle_d$ using the hidden samples $\vect h_d$ in the current batch.

\begin{algorithm}
\caption{Training centered RBMs \label{alg:algorithm2}}
$Initialize~\vect{W}$ \tcc*{$i.e. ~ \vect{W} \gets \mathcal{N}(0,0.01)^{N \times M}$}
$Initialize~\vect{\mu},\vect{\lambda}$ \tcc*{$ i.e. ~\vect{\mu} \gets \langle $data$ \rangle, \vect \lambda \gets \vect {0.5} $} 
$Initialize~\vect{b},\vect{c}$ \tcc*{$ i.e. ~\vect{b} \gets \sigma^{-1}(\vect{\mu}), \vect{c} \gets \sigma^{-1}(\vect{\lambda})$} 
$Initialize~\eta,\nu_{\mu}, \nu_{\lambda}$ \tcc*{$ i.e. ~\eta,\nu_{\mu}, \nu_{\lambda} \in \left\lbrace {0.001,...,0.1} \right\rbrace$} 
\Repeat{stopping criteria is met}
{
\ForEach{\emph{batch \textbf{in} data}}
{
	\ForEach{sample $\vect x_d$ \emph{\textbf{in}} \emph{batch}}
		{
		 $Calculate ~ \vect{h}_d =  p ( h_j = 1 | \vect x_d)$ \tcc*{$\triangleright$ Eq.~\eqref{eqn:cond_h}}
		$Sample ~ \vect{x}_m ~ from\,RBM$ \tcc*{$\triangleright$ Eqs.~\eqref{eqn:cond_x}, \eqref{eqn:cond_h}}
		$Calculate ~ \vect{h}_m =  p (h_j = 1 | \vect x_m)$ \tcc*{$\triangleright$ Eq.~\eqref{eqn:cond_h}}
		$Store ~ \vect{x}_m,\vect{h}_d,\vect{h}_m$
		}
	$Estimate ~ \vect{\mu}_{new}$ \tcc*{$ i.e. ~\vect{\mu}_{new} \gets \langle \vect{x}_d\rangle $}
	$Estimate ~ \vect{\lambda}_{new}$ \tcc*{$ i.e. ~\vect{\lambda}_{new} \gets \langle \vect{h}_d\rangle $} 
	\tcc{Transform parameters with respect to the new offsets}
	$\vect{b} \gets \vect{b} + \nu_{\lambda} \vect{W} \left(\vect{\lambda}_{new} - \vect{\lambda}\right)$ \tcc*{$\triangleright$ Eq.~\eqref{eqn:transform_b}} 
	$\vect{c} \gets \vect{c} + \nu_{\mu} \vect{W}^T \left(\vect{\mu}_{new} - \vect{\mu}\right) $ \tcc*{$\triangleright$ Eq.~\eqref{eqn:transform_c}}
	 \tcc{Update offsets using exp.~moving averages with sliding factors $\nu_{\mu}$ and $\nu_{\lambda}$} 
	$\vect{\mu} \gets (1-\nu_{\mu})\vect{\mu} + \nu_{\mu} \vect{\mu}_{new}$ \\
	$\vect{\lambda} \gets (1-\nu_{\lambda})\vect{\lambda} + \nu_{\lambda} \vect{\lambda}_{new}$ \\
	\tcc{Update parameters using gradient ascent with learning rate $\eta$} 
    $\nabla \vect{W} \gets \langle ( \vect{x}_d - \vect{\mu} ) (\vect{h}_d - \vect{\lambda})^T \rangle - \langle ( \vect{x}_m - 	\vect{\mu} ) (\vect{h}_m - \vect{\lambda})^T \rangle$  \tcc*{ $\triangleright$ Eq.~\eqref{eqn:grad_W}} 
	$\nabla \vect{b} \gets \langle \vect{x}_{d} \rangle - \langle \vect{x}_{m} \rangle$ \tcc*{$\triangleright$ Eq.~\eqref{eqn:grad_b}}
	$\nabla \vect{c} \gets \langle \vect{h}_{d} \rangle - \langle \vect{h}_{m} \rangle$ \tcc*{$\triangleright$ Eq.~\eqref{eqn:grad_c}}  
    $\vect{W} \gets \vect{W} + \eta \nabla \vect{W}$ \\
	$\vect{b} \gets \vect{b} + \eta \nabla \vect{b}$ \\
	$\vect{c} \gets \vect{c} + \eta \nabla \vect{c}$ 
	 }
	 }
\tcc{Transform network to a normal binary RBM if desired} 
$\vect{b} \gets \vect{b} - \vect{W} \vect{\lambda}$ \tcc*{$\triangleright$ Eq.~\eqref{eqn:transform_b}} 
$\vect{c} \gets \vect{c} - \vect{W}^T \vect{\mu} $  \tcc*{$\triangleright$ Eq.~\eqref{eqn:transform_c}} 
$\vect{\mu} \gets \vect{0}$  \\	   
$\vect{\lambda} \gets \vect{0}$
\end{algorithm}

Note that in Algorithm~\ref{alg:algorithm2} the update of the offsets is performed before the gradient is calculated, such that gradient and reparameterization both use the current samples.
This is in contrast to the algorithm for centered DBMs proposed by \citet{MontavonMueller-2012}, where the update of the offsets and the reparameterization follows after the gradient update. Thus, while the gradient still uses the current samples the reparameterization is based on samples gained from the model of the previous iteration.
However, the proposed DBM algorithm smooths the offset estimations by an exponentially moving average over the sample means from many iterations, so that the choice of the sample set used for the offset estimation should be less relevant. 
In Algorithm~\ref{alg:algorithm2} an exponentially moving average is obtained if the sliding factor $\nu$ is set to $0<\nu<1$ and prevented if $\nu=1$. 
The effects of using an exponentially moving average are empirically analyzed in Section~\ref{sec:initExp}.

\subsection{Centered Gradient}\label{sec:centeredGradient}

We now use the centering trick to derive a centered parameter update, which can replace the gradient during the training of normal binary RBMs.
Similar to the derivation of the enhanced gradient we can transform a normal binary to a centered RBM, perform a gradient update, and transform the RBM back (see Appendix~\ref{sec:DerivationCentered} for the derivation). 
This yields the following parameter updates, which we refer to as centered gradient
\begin{eqnarray}
	\nabla_c \vect{W} 
& = & \langle ( \vect{x} - \vect{\mu} ) (\vect{h} - \vect{\lambda})^T \rangle_{d} - \langle ( \vect{x} - \vect{\mu} ) (\vect{h} - \vect{\lambda})^T \rangle_{m} \enspace, \label{eqn:cgradW}\\
	\nabla_c \vect{b}  &=& \langle \vect{x} \rangle_{d} - \langle  \vect{x} \rangle_{m} - \nabla_c \vect{W}\vect{\lambda} \enspace,\label{eqn:cgradb}\\
	\nabla_c \vect{c}  &=& \langle \vect{h} \rangle_{d} - \langle  \vect{h} \rangle_{m} - \nabla_c \vect{W}^T\vect{\mu}\enspace.\label{eqn:cgradc}
\end{eqnarray}
Notice that by setting $\vect{\mu} = \frac{1}{2} \left(\langle  \vect{x} \rangle_{d} + \langle  \vect{x} \rangle_{m} \right)$ and $\vect{\lambda} = \frac{1}{2} \left(\langle  \vect{h} \rangle_{d} + \langle  \vect{h} \rangle_{m} \right)$ the centered gradient becomes equal to the enhanced gradient
(see Appendix~\ref{sec:CenteredToEnhanced}). 
Thus, it becomes clear that the enhanced gradient is a special case of centering. 
This can also be concluded from the derivation of the enhanced gradient for Gaussian visible variables in \citep{ChoIlin-2013}. 

\begin{algorithm}
\caption{Training RBMs using the centered gradient \label{alg:algorithm1}}
$Initialize~\vect{W}$ \tcc*{$i.e. ~ \vect{W} \gets \mathcal{N}(0,0.01)^{N \times M} $}
$Initialize~\vect{\mu},\vect{\lambda}$ \tcc*{$ i.e. ~\vect{\mu} \gets \langle $data$ \rangle, \vect \lambda \gets \vect {0.5} $}
$Initialize~\vect{b},\vect{c}$ \tcc*{$ i.e. ~\vect{b} \gets \sigma^{-1}(\vect{\mu}), \vect{c} \gets \sigma^{-1}(\vect{\lambda})$}
$Initialize~\eta,\nu_{\mu}, \nu_{\lambda}$ \tcc*{$ i.e. ~\eta,\nu_{\mu}, \nu_{\lambda} \in \left\lbrace {0.001,...,0.1} \right\rbrace$}
\Repeat{stopping criteria is met}
{ 
\ForEach{\emph{batch \textbf{in} data}}
{
	\ForEach{\emph{ample $\vect v_d$ \textbf{in} batch}}
	{
		$Calculate ~ \vect{h}_d = p ( h_j = 1 | \vect x_d)$ \tcc*{$\triangleright$ Eq.~\eqref{eqn:cond_h}}
		$Sample ~ \vect{x}_m\,\,from\,RBM$ \tcc*{$\triangleright$ Eqs.~\eqref{eqn:cond_x}, \eqref{eqn:cond_h}}
		$Calculate ~ \vect{h}_m = p (h_j = 1 | \vect v_m)$ \tcc*{$\triangleright$ Eq.~\eqref{eqn:cond_h}}
		$Store ~ \vect{x}_m,\vect{h}_d,\vect{h}_m$
	}
	$Estimate ~ \vect{\mu}_{new}$ \tcc*{$ i.e. ~\vect{\mu}_{new} \gets \langle \vect{x}_d\rangle $} 
	$Estimate ~ \vect{\lambda}_{new}$ \tcc*{$ i.e. ~\vect{\lambda}_{new} \gets \langle \vect{h}_d\rangle $}
	\tcc{Update offsets using exp. moving averages with sliding factors $\nu_{\mu}$ and $\nu_{\lambda}$}
	$\vect{\mu} \gets (1-\nu_{\mu})\vect{\mu} + \nu_{\mu} \vect{\mu}_{new}$ \\
	$\vect{\lambda} \gets (1-\nu_{\lambda})\vect{\lambda} + \nu_{\lambda} \vect{\lambda}_{new}$ \\
	\tcc{Update parameters using the centered gradient with learning rate $\eta$}
	$\nabla_c \vect{W} \gets \langle ( \vect{x}_d - \vect{\mu} ) (\vect{h}_d - \vect{\lambda})^T \rangle - \langle ( \vect{x}_m - 	\vect{\mu} ) (\vect{h}_m - \vect{\lambda})^T \rangle$  \tcc*{$\triangleright$ Eq.~\eqref{eqn:cgradW}}
	$\nabla_c \vect{b} \gets \langle \vect{x}_{d} \rangle - \langle  \vect{x}_{m} \rangle - \nabla_c \vect{W}\vect{\lambda}$ \tcc*{$\triangleright$ Eq.~\eqref{eqn:cgradb}}
	 $\nabla_c \vect{c} \gets \langle \vect{h}_{d} \rangle - \langle  \vect{h}_{m} \rangle - \nabla_c \vect{W}^T\vect{\mu}$ \tcc*{$\triangleright$ Eq.~\eqref{eqn:cgradc}}
	$\vect{W} \gets \vect{W} + \eta \nabla_c \vect{W}$ \\
	$\vect{b} \gets \vect{b} + \eta \nabla_c \vect{b}$\\
	$\vect{c} \gets \vect{c} + \eta \nabla_c \vect{c}$
}
}
\end{algorithm}
The enhanced gradient has been designed such that the weight updates become the difference of the covariances between one visible and one hidden variable under the data and the model distribution. 
Interestingly, one gets the same weight update for two other choices of offset parameters: either $\vect{\mu} = \langle  \vect{x} \rangle_{d}$ and $\vect{\lambda} = \langle  \vect{h} \rangle_{m}$ or
$\vect{\mu} = \langle  \vect{x} \rangle_{m}$ and $\vect{\lambda} = \langle  \vect{h} \rangle_{d}$. 
However, these offsets result in different update rules for the bias parameters.

Algorithm~\ref{alg:algorithm1} shows pseudo code for training a normal binary RBM using the centered gradient, which is equivalent to training a centered binary RBM using Algorithm~\ref{alg:algorithm2}.
Both algorithms can easily be extended to RBMs with other types of units and DBMs.

\subsection{Centered Deep Boltzmann Machines}\label{sec:dbm}

A DBM~\citep{SalakhutdinovHinton-2009a} is a deep undirected graphical model with several hidden layers where successive layers have a bipartite connection structure. 
Therefore, a DBM can be seen as a stack of several RBMs and thus as natural extension of RBMs.
A centered binary DBM with $L$ layers $\vect{h}_{(0)}, \cdots ,\vect{h}_{(L)}$ (where $\vect{h}_{(0)}$ corresponds to the visible layer) represents a Gibbs distribution with energy
\begin{eqnarray}
  E\left(\vect{h}_{\left(0\right)}, \cdots ,\vect{h}_{\left(L\right)} \right)
  &=&
  	- \sum_{l=0}^L \left(\vect{h}_{\left(l\right)} - \vect{\lambda}_{\left(l\right)}\right)^T\vect{b}_{\left(l\right)}
  	- \sum_{l=0}^{L-1} \left(\vect{h}_{\left(l\right)} - \vect{\lambda}_{\left(l\right)}\right)^T\vect{W}_{\left(l\right)}\left(\vect{h}_{\left(l+1\right)} - \vect{\lambda}_{\left(l+1\right)}\right)\enspace\nonumber
\,,
\end{eqnarray}
where each layer $l$ has a bias $\vect b_{(l)}$, an offset $\vect \lambda_{(l)}$ and is connected to layer $l+1$ by weight matrix $\vect W_l$.

The derivations, proofs and algorithms given in this work for RBMs automatically extend to DBMs since each DBM can be transformed to an RBM with restricted connections and partially unknown input data. This is illustrated for a DBM with four layers in Figure~\ref{fig:deep_flat}.
As a consequence of this relation DBMs can essentially be trained in the same way as RBMs but also suffer from the same problems as described before. The only difference when training DBMs is that the expectation under the data distribution in the LL gradient cannot be calculated exactly as it is the case for RBMs. Instead the term is approximated by running a mean field estimation until convergence~\citep{SalakhutdinovHinton-2009a}, which corresponds to approximating the gradient of a lower variational bound of the LL.
Furthermore, it is common to pre-train DBMs in a greedy layer wise fashion using RBMs~\citep{SalakhutdinovHinton-2009a, HintonSalakhutdinov-2012}.
\begin{figure}\label{fig:deep_flat}
\begin{center}
\subfigure[Deep network version.]{\label{subfig:deep_net}
\includegraphics[scale=1.2, clip=true, trim=0 20 0 0]{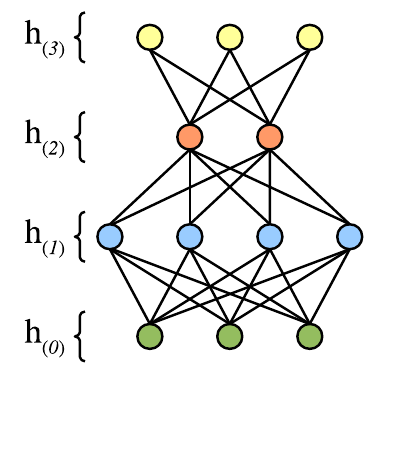}}
\hspace{10mm}
\subfigure[Shallow network version.]{\label{subfig:flat_net}
\includegraphics[scale=1.2, clip=true, trim=0 -8 0 0] {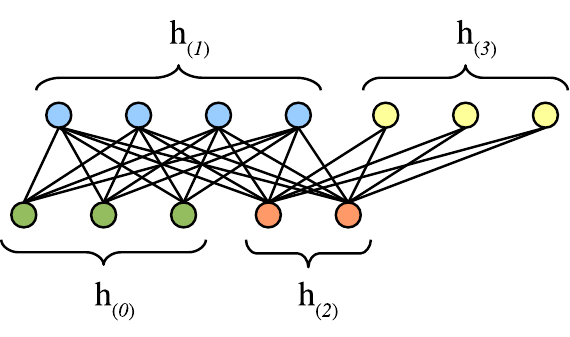}}
\caption{Example for (a) a deep neural network with four layers $h_{(0)}, \cdots ,h_{(3)}$ and (b) the equivalent two layer shallow version of the same network with restricted connections and unknown input $h_{(2)}$.}
\end{center}
\end{figure}  

\section{Centering in Artificial Neural Networks in General }\label{sec:generalization}

Removing the mean from visible and hidden units has originally been proposed for feed forward neural networks~\citep{LeCun1998,Schraudolph98}. When this idea was applied to RBMs ~\citep{TangSutskever-2011} the model was reparameterized such that the  probability distribution defined by the normal and centered RBM stayed the same. 
In this section we generalize this concept to show that centering is an alternative parameterization for arbitrary ANN architectures in general, if the network is reparameterized accordingly
This holds independently of the chosen activation functions and connection types including directed, undirected and recurrent connections.
To show the correctness of this statement, let us consider the centered artificial neuron model
\begin{eqnarray}
o_j &=& \phi_j\left( \sum_i w_{ij} \left(a_i - \mu_i \right) + c_j\right),\label{eqn:general_neuron}
\end{eqnarray}
where the output $o_j$ of the $j$th neuron depends on its activation function $\phi_j$, bias term $c_j$ and weights $w_{ij}$ with associated inputs $a_i$ and their corresponding offsets $\mu_i$. Such neurons can be used to construct arbitrary network architectures using undirected, directed and recurrent connections, which can then be optimized with respect to a chosen loss. 

Two ANNs that represent exactly the same functional input-output mapping can be considered as different parameterizations of the same model. Thus, a centered ANN is just a different parameterization of an uncentered ANN if we can show that their functional input-output mappings are the same. This can be guaranteed in general if all corresponding units in a centered and an uncentered ANN have the same mapping from inputs to outputs. If the offset $\mu_i$ is changed to $\mu_i' = \mu_i + \nabla \mu_i$ then the output of the centered artificial neuron~\eqref{eqn:general_neuron} becomes
\begin{eqnarray}
\phi_j\left( \sum_i w_{ij} \left(a_i - \mu_i' \right) + c_j\right) 
&=& \phi_j\left( \sum_i w_{ij} \left(a_i - \left(\mu_i + \nabla \mu_i \right) \right)+c_j\right)\nonumber\,\\
&=& \phi_j\left( \sum_i w_{ij} \left(a_i - \mu_i \right)+c_j - \sum_i w_{ij} \nabla \mu_i \right)\nonumber\,,
\end{eqnarray}
showing that the units output does not change when changing the offset $\mu_i$ to $\mu_i'$ if the units bias parameter $c_j$ is reparameterized to $c_j' = c_j + \sum_i w_{ij} \nabla \mu_i$.

This generalizes the proposed reparameterization for RBMs given by Equation~\eqref{eqn:transform_c} to ANNs. Note that the originally centering algorithm~\citep{LeCun1998, Schraudolph98} did not reparameterize the network, which can cause instabilities especially if the learning rate is large.
By setting $\mu_i$ or $\mu_i'$ to zero it now follows that for each normal ANN there exists a centered ANN  and vice verse such that the output of each neuron and thus the functional mapping from input to output of the whole network stays the same. If the input to output mapping does not change this also holds for an arbitrary loss depending on this output. 

Moreover, if we guarantee that a shift of $a_i$ implies a shift of $\mu_i$ by the same value (that is a shift of $a_i$ to $a_i + \delta_i$ implies a shift of $\mu_i$ to $ \mu_i + \delta_i$) the neuron's output $o_j$ gets  invariant to shifts of $a_i$. 
This is easy to see since $\delta_i$ cancels out in Equation~\eqref{eqn:general_neuron} if the same shift  is applied to both $a_i$ and $\mu_i$, which holds for example if we set the offsets to the mean values of the corresponding variables since  $\langle a_i+\delta_i\rangle = \delta_i + \langle a_i\rangle$. 

\subsection{Auto Encoders}\label{sec:autoencoders}

An AE or auto-associator~\citep{RumelhartMcClellandEtAl-1986} is a type of neural network that has originally been proposed for unsupervised dimensionality reduction. Like RBMs, AEs have also been used for unsupervised feature extraction and greedy layer-wise pre-training of deep neural networks~\citep{BengioLamblinEtAl-2007}. In general, an AE consists of a deterministic encoder $encode(\vect{x})$, which maps the input $\vect{x} = (x_1, ... ,x_N)$ to a hidden representation $\vect{h} = (h_1, ... ,h_M)$ and a deterministic decoder $decode(\vect{h})$, which maps the hidden representation to the reconstructed input representation $\vect{\tilde{x}}$. The network is optimized such that the reconstructed input $\vect{\tilde{x}}$ gets as close as possible to the original input $\vect{x}$ measured by a chosen loss $\mathcal{L}(\vect{x},\vect{\tilde{x}})$. Common choices for the loss are the mean squared error $ \langle \sum^N_{i=1} \left(x_i-\tilde{x}_i\right)^2 \rangle $ for arbitrary input and the average cross entropy $ \langle - \sum^N_{i=1}x_i \log \tilde{x}_i + (1 - x_i)\log(1 - \tilde{x}_i) \rangle $ for binary data. AEs are usually trained via back-propagation \citep{Kelley-1960, RumelhartWilliams-1986} and they can be seen as feed-forward neural networks where the input patterns are also the labels.
 We can therefore define a centered AE by centering the encoder and decoder, which for a centered three layer AE corresponds to
\begin{eqnarray}
	 encode(\vect{x}) &=&\phi^{enc}\left( \vect{W}' \left(\vect{x} - \vect{\mu} \right) + \vect{c}\right)\,\,\, = \,\,\,\,  \vect{h},\nonumber\\
	decode(\vect{h}) &=& \phi^{dec}\left( \vect{W} \left(\vect{h} - \vect{\lambda} \right) + \vect{b}\right) \,\,\,\,\, = \,\,\,\, \vect{\tilde{x}} \nonumber,
\end{eqnarray}
with encoder matrix $\vect{W}'$, decoder matrix $\vect{W}$, encoder bias $\vect{c}$, decoder bias $\vect{b}$, encoder offset $\vect{\mu}$, decoder offset $\vect{\lambda}$, encoder activation function $\phi^{enc}$ and decoder activation function $\phi^{enc}$. It is common to assume tied weights, which means that the encoder is just the transpose of the decoder matrix ($\vect{W}'$ = $\vect{W}^T$). When choosing the activation functions for the encoder and decoder (i.e. sigmoid, tangens-hyperbolicus, radial-basis, linear, linear-rectifier,  ... ), we have to ensure that the encoder activation function is appropriate for the input data (e.g. a sigmoid cannot represent negative values). Worth mentioning that when using the sigmoid function for $\phi^{enc}$ and $\phi^{dec}$ the encoder becomes equivalent to Equation~\ref{eqn:cond_x} and the decoder becomes equivalent to Equation~\ref{eqn:cond_h}. The networks structure therefore becomes equivalent to an RBM such that the only difference is the training objective.

\subsection{Initialization of the Model Parameters}\label{sec:initialization}

It is a common way to initialize the weight matrix of ANNs to small random values to break the symmetry. 
The bias parameters are often initialized to zero. 
However, we argue that there exists a more reasonable initialization for the bias parameters.

\citet{Hinton-2010b} proposed to initialize the RBM's visible bias parameter $b_i$ to $\ln( p_i/(1- p_i))$, where $p_i$ is the proportion of the data points in which unit $i$ is on
(that is $p_i= \langle x_i \rangle_d$). 
He states that if this is not done, the hidden units are used to activate the $i$th visible unit with a probability of approximately $p_i$ in the early stage of training. 

We argue that this initialization is in fact reasonable since it corresponds to the Maximum Likelihood Estimate (MLE) of the visible bias given the data for an RBM with zero weight matrix, given by

\begin{equation}\label{eqn:MLE}
	\vect b^* = 
	\ln\left(\frac{\langle  \vect{x} \rangle_d}{ \vect 1 - \langle  \vect{x} \rangle_d} \right)
 =  -\ln\left(\frac{ \vect 1}{\langle  \vect{x} \rangle_d} - \vect 1 \right)
 =  \sigma^{-1}(\langle  \vect{x} \rangle_d)\enspace ,
\end{equation}
where $\sigma^{-1}$ is the inverse sigmoid function. 
Notice that the MLE of the visible bias for an RBM with zero weights is the same whether the RBM is centered or not. 
The conditional probability of the visible variables~\eqref{eqn:cond_x} of an RBM with this initialization is then given by $p\left(\vect{x=1}|\vect{h}\right) = \sigma(\sigma^{-1}(\langle  \vect{x} \rangle_d)) = \langle \vect{x} \rangle_d $, where $p\left(\vect{x=1}|\vect{h}\right)$ denotes the vector containing the elements $p\left(x_i=1|\vect{h}\right)$. 
Thus the mean of the  data is initially modeled only by the bias values and the weights are free to model higher order statistics in the beginning of training.
For the unknown hidden variables it is reasonable to assume an initial mean of $0.5$ so that the 
MLE of the hidden bias for an RBM with zero weights is given by $\vect c^*= \sigma^{-1}(0.5) = 0.0$. 
These considerations still hold approximately if the weights are not zero but initialized to small random values.
 
\citet{MontavonMueller-2012} suggested to initialize the bias parameters to the inverse sigmoid of the initial offset parameters. 
They argue that this initialization leads to a good starting point, because it guarantees that the Boltzmann machine is initially centered.
Actually,  if the initial offsets are set to $\mu_i = \langle x_i \rangle_d$ and $\lambda_j = 0.5$ the initialization suggested by~\citet{MontavonMueller-2012} is equal  to the initialization to the MLEs as follows from Equation~\eqref{eqn:MLE}. 

Note that this initialization is not restricted to RBMs or sigmoid activation function. Independent of the initial weight matrix we can always set the bias in ANNs to the inverse activation function of the corresponding mean value .

\section{Methods}\label{sec:Experiments}

As shown in the previous section the algorithms described by
\citet{ChoRaikoEtAl-2011a}, \citet{TangSutskever-2011} and \citet{MontavonMueller-2012}
can all be viewed as different ways of applying the centering trick. 
They differ in the choice of the offset parameters and in the  way of approximating them, either based on the samples gained from the model in the previous learning step or from the current one, using an exponentially moving average or not. 
The question arises, how RBMs should be centered to achieve the best performance in terms of the LL.
In the following we analyze the different ways of centering empirically 
and try to derive a deeper understanding of why centering is beneficial.

For simplicity we introduce the following shorthand notation. 
We use $d$ to denote the data mean $\langle \cdot \rangle_d$, $m$ for the model mean $\langle \cdot \rangle_m$, $a$ for the average of the means $\frac{1}{2}\langle \cdot \rangle_d + \frac{1}{2}\langle \cdot \rangle_m$, and \emph{0} if the offsets is set to zero. 
We indicate the choice of $\vect{\mu}$ in the first and the choice of $\vect \lambda$ in the second place, for example $dm$ translates to $\vect \mu=\langle \vect x \rangle_d $ and  $\vect \lambda=\langle \vect h \rangle_m$.
We add a superscribed $b$ (before) or $l$ (later) to denote whether the reparameterization is performed before or after the gradient update. 
If the sliding factor in Algorithm \ref{alg:algorithm2}  or \ref{alg:algorithm1} is set to a value smaller than one and thus an exponentially moving average is used, a subscript $s$ is added.
 Thus, we indicate the variant of \citet{ChoRaikoEtAl-2011a} by $aa^b$, the one of \citet{MontavonMueller-2012} by $dd^l_s$, the data normalization of \citet{TangSutskever-2011} by \emph{d0}, and the normal binary RBM simply by \emph{00}.
Table~\ref{tab:algo_overview} summarizes the abbreviations most frequently used in this paper.

We begin our analysis with RBMs, where one layer is small enough to guarantee that the exact LL is still tractable. 
In a first set of experiments we analyze the four algorithms described above in terms of the evolution of the LL during training. 
In a second set of experiments we analyze the effect of the initialization described in Section~\ref{sec:initialization}. 
We proceed with a comparison of the effects of estimating offset values and reparameterizing the parameters before or after the gradient update. 
Afterwards we analyze the effects of using an exponentially moving average to approximate the offset values in the different algorithms and of choosing other offset values. We continue with comparing the normal and the centered gradient with the natural gradient. To verify whether the results scale to more realistic problem sizes we compare the RBMs, DBMs and AE on ten large datasets.

\begin{table}[t]
\begin{center}
\begin{small}
\begin{sc}
\begin{tabular}{l l l l}
\hline
\abovespace\belowspace
Abbr. & $\mu$ & $\lambda$ & Description\\
\hline
\abovespace
$00$ & $\vect 0$ & $\vect 0$ & Normal binary RBM \\ & & & \citep{HintonOsinderoEtAl-2006a}\\
\hline
\abovespace\belowspace
$d0$ & $\langle \vect x \rangle_d$ & $\vect 0$ & Data Normalization RBM \\ & & & \citep{TangSutskever-2011} \\
\hline
\abovespace\belowspace
$dd^l_s$ & $\langle \vect x \rangle_d$ & $\langle \vect h \rangle_d$ & Original Centered RBM \\ & & & \citep{MontavonMueller-2012} \\ & & & reparam. after gradient update, \\ & & & use of an exp. moving average \\
\hline
\abovespace\belowspace
$aa^b$ & $0.5 \left( \langle \vect x \rangle_d + \langle \vect x \rangle_m \right)$ & $0.5 \left( \langle \vect h \rangle_d + \langle \vect h \rangle_m \right)$  & Enhanced gradient RBM \\ & & & \citep{ChoRaikoEtAl-2011a} \\ & & & reparam. before gradient update, \\ & & & no exp. moving average \\
\hline
\abovespace\belowspace
$dd^b_s$ & $\langle \vect x \rangle_d$ & $\langle \vect h \rangle_d$  & Centering using the data mean, \\ & & & reparam. before gradient update, \\ & & & use of an exp. moving average \\
\hline
\abovespace\belowspace
$mm^b_s$ & $\langle \vect x \rangle_m$ & $\langle \vect h \rangle_m$  & Centering using the model mean, \\ & & & reparam. before gradient update, \\ & & & use of an exp. moving average \\
\hline
\abovespace\belowspace
$dm^b_s$ & $\langle \vect x \rangle_d$ & $\langle \vect h \rangle_m$  & Centering using the data mean \\ & & & for the visible and the model mean \\ & & &for hidden units, \\ & & & reparam. before gradient update, \\ & & & use of an exp. moving average \\
\end{tabular}
\end{sc}
\end{small}
\end{center}
\caption {Look-up table: Abbreviations for the most frequently used algorithms.}
\label{tab:algo_overview}
\end{table}

\subsection{Benchmark Problems}\label{sec:Benchmark_problems}

We consider four different benchmark problems in our detailed analysis.\\
The {\bf{\emph{Bars \& Stripes}}}~\citep{MacKay2003} 
problem consists of quadratic patterns of size $D \times D$ that can be generated as follows. 
First, a vertical or horizontal orientation is chosen randomly with equal probability. 
Then the state of all pixels of every row or column is chosen uniformly at random.
This leads to $N = 2^{D+1}$ patterns 
(see Figure~\ref{subfig:BAS} for some example patterns)
where the completely uniform patterns occur twice as often as the others. 
The dataset is symmetric in terms of the amount of zeros and ones and thus the flipped and unflipped problems are equivalent. 
An upper bound of the LL is given by $\left(N-4\right)\ln \left(\frac{1}{N}\right) + 4 \ln \left(\frac{2}{N}\right)$. 
For our experiments we used $D=3$ or $D=2$ (only in Section~\ref{sec:natural_grad_exp}) leading to an upper bound of $-41.59$ and $-13.86$, respectively.

\begin{figure}[t]\label{fig:BASexample}
\begin{center}
\subfigure[8 out of 16 patterns from the \emph{Bars \& Stripes} dataset.]{\label{subfig:BAS}
\centerline{\includegraphics[scale=0.71]{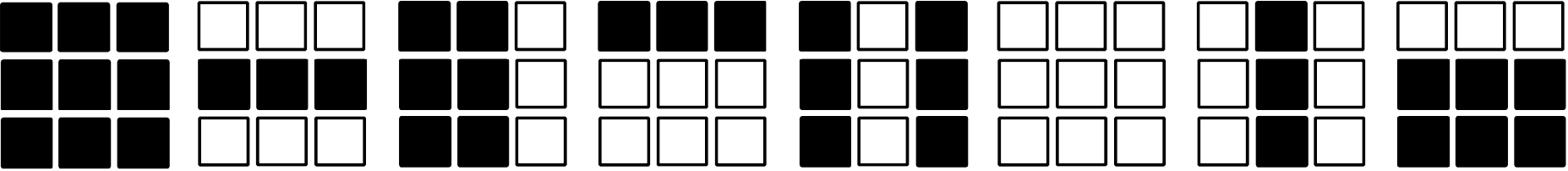}}}
\vspace{5mm}
\subfigure[All patterns of the \emph{Shifting Bar} dataset with $N=9$ and $B=1$.]{\label{subfig:Shifter}
\centerline{\includegraphics[scale=0.67]{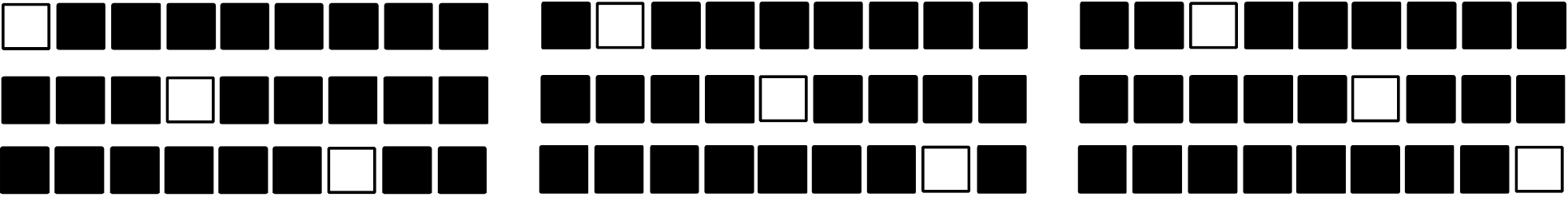}}}
\vspace{2mm}
\subfigure[Example patterns from the \emph{MNIST} dataset after binarization.]
{\label{subfig:MNIST}
\centerline{\includegraphics[scale=1.3]{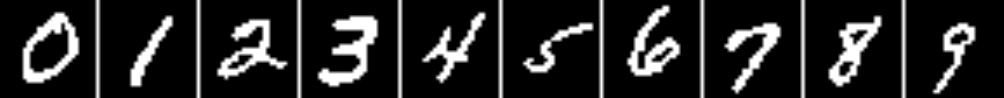}}}
\vspace{5mm}
\subfigure[Example patterns from the \emph{Caltech 101 Silhouette} dataset.]{
\label{subfig:Caltech}
\centerline{\includegraphics[scale=1.3]{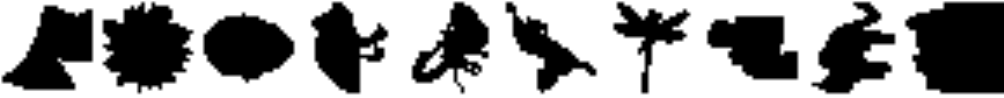}}}
\vspace{-5mm}
\caption{Some patterns from the different benchmark problems.}
\end{center}
\end{figure}  

The {\bf{\emph{Shifting Bar}}} dataset is an artificial benchmark problem we have designed to be asymmetric in terms of the amount of zeros and ones in the data.  
For an input dimensionality $N$, a bar of length $0<B<N$ has to be chosen, where $\frac{B}{N}$ expresses the percentage of ones in the dataset. 
A position $0 \le p<N$ is chosen uniformly at random and the states of the following $B$ pixels are set to one, where a wrap around is used if $p+B \ge N$. The states of the remaining pixels are set to zero. This leads to $N$ different patterns
(see Figure~\ref{subfig:Shifter}) with equal probability and an upper bound of the LL of $N\ln\left(\frac{1}{N}\right)$.
For our experiments we used $N=9$, $B=1$ and its flipped version {\bf{\emph{Flipped~Shifting~Bar}}}, which we get for $N=9$, $B=8$, both having an upper LL bound  of $-19.78$.

The {\bf{\emph{MNIST}}}~\citep{LeCun1998} dataset of handwritten digits has become a standard benchmark problem for RBMs. It consists of $60,000$ training and $10,000$ testing examples of gray value handwritten digits of size $28 \times 28$. See Figure~\ref{subfig:MNIST} for some example patterns. 
After binarization (with a threshold of 0.5) the dataset contains $13.3\%$ ones, similar to the \emph{Shifting Bar} problem, which for our choice of $N$ and $B$ contains $11.1\%$ ones. We refer to the dataset where each bit of MNIST is flipped (that is each one is replaced by a zero and \emph{vice versa}) as \mbox{\bf{\emph{1-MNIST}}}.
To our knowledge, the best reported performance in terms of the average LL per sample of an RBM with 500 hidden units on \emph{MNIST} test data is \mbox{-84}~\citep{Salakhutdinov-2008, SalakhutdinovMurray-2008, TangSutskever-2011, ChoRaikoEtAl-2013}.

The {\bf{\emph{CalTech 101 Silhouettes}}}  \citep{MarlinSwerskyEtAl-2010a} dataset consists of $4.100$ training, $2.307$ validation,  and $2264$ testing examples of binary object silhouettes of size $28 \times 28$. See Figure~\ref{subfig:Caltech} for some example patterns. The dataset contains $55.1\%$ ones, and thus (like in the \emph{Bars \& Stripes problem}) the amount of  zeros and ones is almost  the same. The background pixels take the value one which is in contrast to \emph{MNIST} where the background pixels are set to zero.
To our knowledge, the best reported performance in terms of the average LL per sample of an RBM with 500 hidden units on \emph{CalTech 101 Silhouettes} test data is \mbox{-114.75}~\citep{ChoRaikoEtAl-2013}. 
 
In some experiments we considered eight {\bf{\emph{additional binary datasets}}} from different domains compromising  biological, image, text and game-related data~\citep{LarochelleBengioEtAl-2010}. The datasets differ in dimensionality (112 to 500) and size (a few hundred to several thousand examples) and    
have been separated into training, validation and test sets.
The average test LL for binary RBMs with 23 hidden units and related models can be found in~\citep{LarochelleMurray-2011}. All datasets contain less ones than zeros,  where the percentage of ones lies between $3.9\%$ and $36.8\%$.

\section{Results}\label{sec:Results}
 
For all models in this work the weight matrices were initialized with random values sampled from a
Gaussian with zero mean and a standard deviation of $0.01$. If not stated otherwise the visible and hidden biases, and offsets were initialized as described in Section~\ref{sec:initialization}. 
 
We begin our analysis with experiments on small RBMs where the LL can be calculated exactly.
We used 4 hidden units when modeling \emph{Bars \& Stripes} and \emph{Shifting Bar} and 16 hidden units when modeling \emph{MNIST}. For training we used CD and PCD with $k$ steps of Gibbs sampling (CD-$k$, PCD-$k$) and PT$_{c}$ where the $c$ temperatures were distributed uniformly form $0$ to $1$. For \emph{Bars \& Stripes} and \emph{Shifting Bar} full-batch training was performed for $50,000$ gradient updates, where the LL was evaluated every 50th gradient update.
For modeling \emph{MNIST} mini-batch training with a batch size of 100 was performed for 100 epochs, each consisting of 600 gradient updates and the exact LL was evaluated after each epoch.

The following tables containing the results for RBMs show the maximum average LL and the corresponding standard deviation reached during training with different learning algorithms over the 25 trials.
In some cases the final average LL reached at the end of training is given in parenthesis to indicate a potential divergence of the LL. 
For reasons of readability, the average LL was divided by the number of training samples in the case of \emph{MNIST}.
In order to check if the result of the best method within one row differs significantly from the others we performed pairwise signed Wilcoxon rank-sum tests (with $p=0.05$).
The best results are highlighted in bold. 
This can be more than one value if the significance test between these values was negative.

\subsection{Comparison of the Standard Methods}

The comparison of the learning performance of the previously described algorithms $dd^l_s$,  $aa^b$,  \emph{d0}, and \emph{00} (using their originally proposed initializations) shows that training a centered
 RBM leads to significantly higher LL values than training a normal binary RBM (see Table~\ref{tab:originalMethods1} for the results for \emph{Bars \& Stripes} and  \emph{MNIST} and Table~\ref{tab:originalMethods2} for the results for  \emph{Shifting Bar} and \emph{Flipped Shifting Bar}).
Figure~\ref{subfig:OrignalMethods_B} illustrates on the \emph{Bars \& Stripes} dataset that centering both the visible and the hidden variables ($dd^l_s$ and $aa^b$) compared to centering only the visible variables (\emph{d0}) accelerates the learning and leads to a higher LL when using PT. The same holds for PCD as can be seen from Table~\ref{tab:originalMethods1}. 
Thus centered RBMs can form more accurate models of the data distribution than normal RBMs.
This is different to the observations made for DBMs by \citet{MontavonMueller-2012}, which found a better generative performance of centering only in the case of locally connected DBMs. 

It can also be seen in Figure~\ref{fig:OrignalMethods} that all methods show divergence in combination with CD and PCD \citep[as described before by][for normal RBMs]{FischerIgel-2010a}, which can be prevented for $dd^l_s$, \emph{d0}, and \emph{00} when using PT as shown in Figure~\ref{subfig:OrignalMethods_D}.
This can be explained by the fact that PT leads to faster mixing Markov chains and thus less biased gradient approximations. The $aa$ algorithm however suffers from severe divergence of the LL when PCD or PT is used, which is even worse than with CD. 
This divergence problem arises independently of the choice of the learning rate as indicated by the LL values reached at the end of training (given in parentheses) in Table~\ref{tab:originalMethods1} and Table~\ref{tab:originalMethods2} and which can also be seen by comparing Figure~\ref{subfig:OrignalMethods_C} and Figure~\ref{subfig:OrignalMethods_D}.
The divergence occurs the earlier and faster the bigger the learning rate, while for the other algorithms we never observed divergence in combination with PT even for very big learning rates and long training time. The reasons for this divergence will be discussed in detail in Section~\ref{sec:divergence}.
The results in Table~\ref{tab:originalMethods2} also demonstrate the flip invariance of the centered RBMs on the \emph{Shifting Bar} dataset empirically. 
While \emph{00} fails to model the flipped version of the dataset correctly $dd^l_s$,  $aa^b$, \emph{d0} have approximately the same performance on the flipped and unflipped dataset.

\begin{figure}[t]
\centering
\subfigure[CD-1 - learning rate 0.05]{\label{subfig:OrignalMethods_A}
\includegraphics[trim=0.6cm 0cm 1.55cm 0cm, clip=true,scale=0.412]{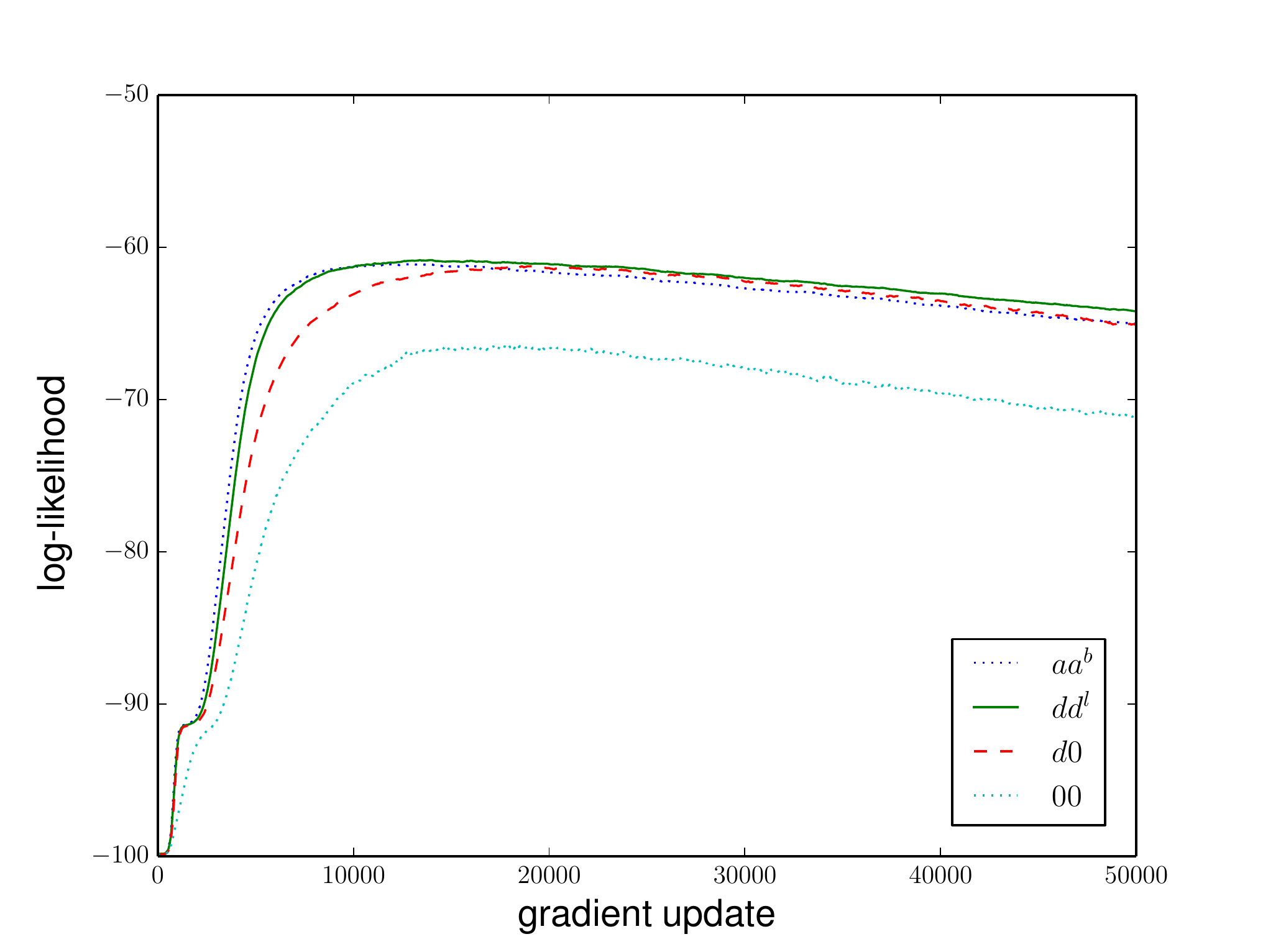}}
\subfigure[PT$_{10}$ - learning rate 0.05]{\label{subfig:OrignalMethods_B}
\includegraphics[trim=1.56cm 0cm 1.55cm 0cm, clip=true,scale=0.412]{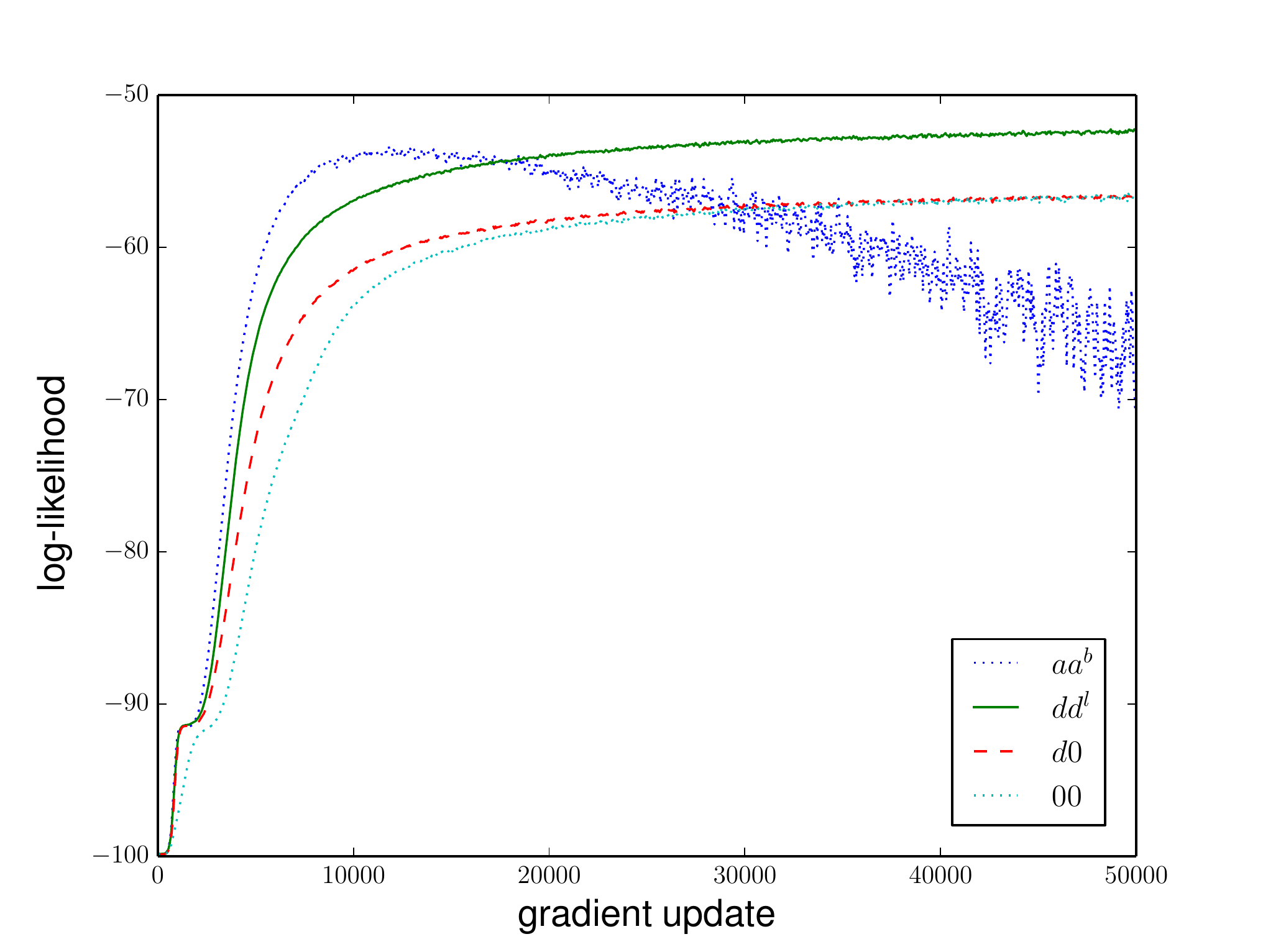}}
\subfigure[PCD-1 - learning rate 0.05]{\label{subfig:OrignalMethods_C}
\includegraphics[trim=0.6cm 0cm 1.55cm 0cm, clip=true,scale=0.412]{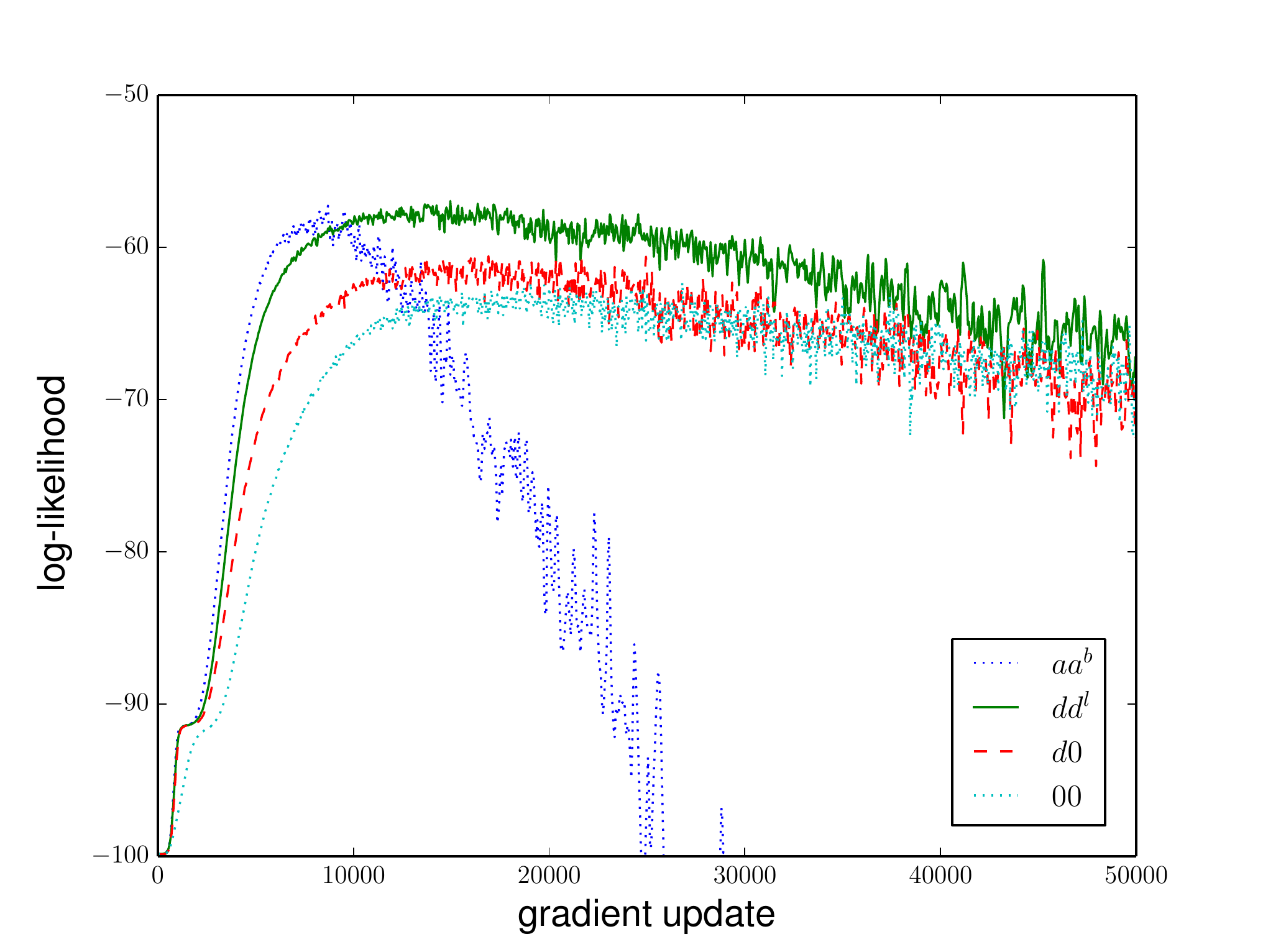}}
\subfigure[PCD-1 - learning rate 0.01]{\label{subfig:OrignalMethods_D}
\includegraphics[trim=1.56cm 0cm 1.55cm 0cm, clip=true,scale=0.412]{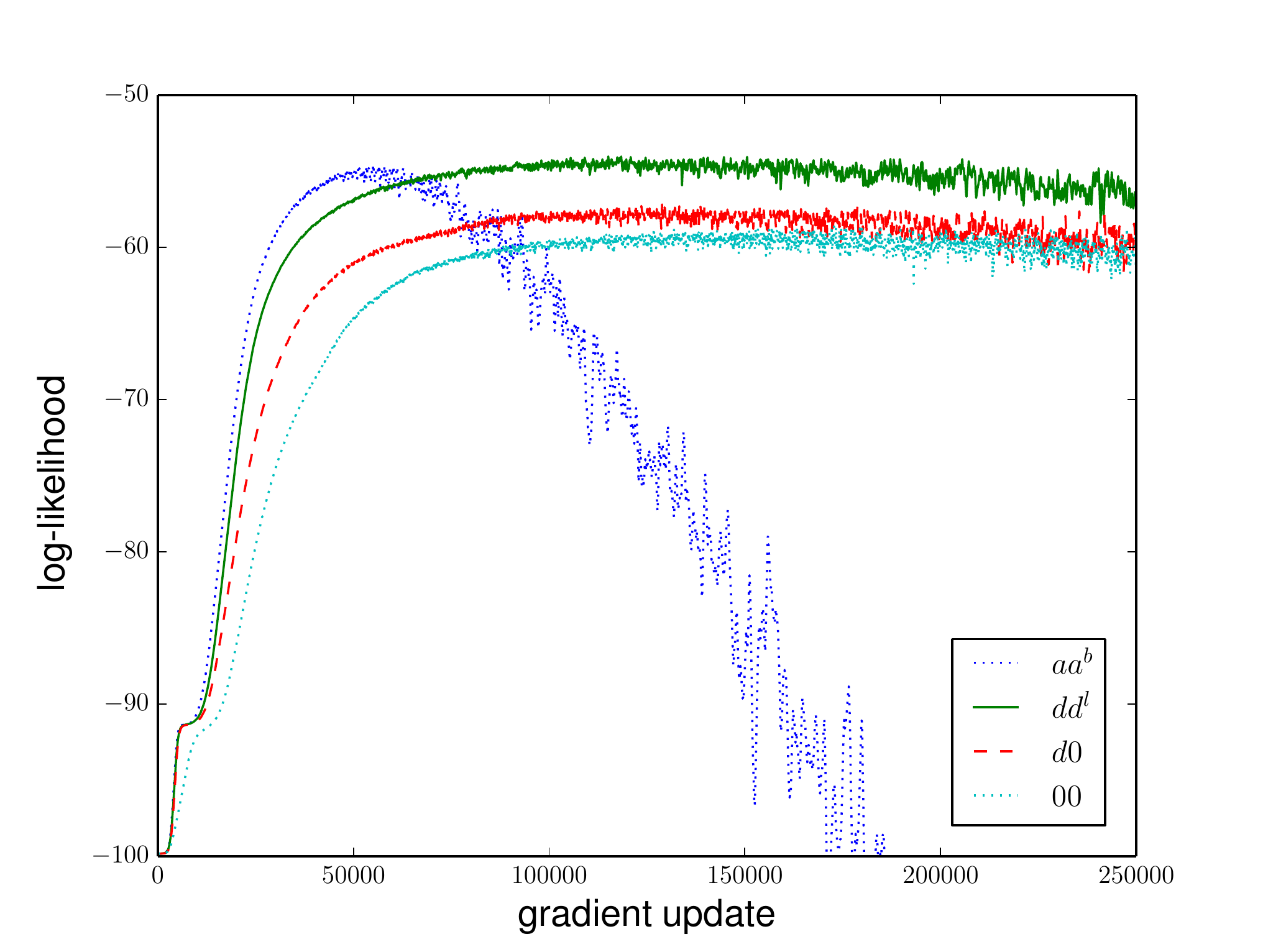}}
\caption{Evolution of the average LL during training on the \emph{Bars \& Stripes} dataset for the standard centering methods. 
(a) When CD-1 is used for sampling and a learning rate of $\eta=0.05$,
(b) when PT$_{10}$ is used for sampling and a learning rate of $\eta=0.05$,
(c) when PCD-1 is used for sampling and a learning rate of $\eta=0.05$, and
(d) when PCD-1 is used for sampling and a learning rate of $\eta=0.01$. }
\label{fig:OrignalMethods}
\end{figure} 

\begin{landscape}
\begin{table*}
\setlength{\tabcolsep}{3pt}
\begin{center}
\begin{small}
\begin{sc}
\hspace*{-2.1cm}
\begin{tabular}{l@{\hskip -0.07in} l l l l}
\hline
\abovespace\belowspace
Algorithm-$\eta$ & \multicolumn{1}{c}{$aa^b$} & \multicolumn{1}{c}{$dd_s^l$} & \multicolumn{1}{c}{\emph{d0}} & \multicolumn{1}{c}{\emph{00}} \\
\hline
\abovespace
Bars \& Stripes  &  &  &  &  \\
\abovespace
CD-1-0.1& \textbf{-60.85} $\pm$1.91 \,(-69.1) & \textbf{-60.41} $\pm$2.08 \,(-68.8) & \textbf{-60.88} $\pm$3.95 \,(-70.9) & -65.05 $\pm$3.60 \,(-78.1) \\
CD-1-0.05& \textbf{-60.37} $\pm$1.87 \,(-65.0) & \textbf{-60.25} $\pm$2.13 \,(-64.2) & \textbf{-60.74} $\pm$3.57 \,(-65.1) & -64.99 $\pm$3.63 \,(-71.2) \\
CD-1-0.01& \textbf{-61.00} $\pm$1.54 \,(-61.1) & -61.22 $\pm$1.49 \,(-61.3) & -63.28 $\pm$3.01 \,(-63.3) & -68.41 $\pm$2.91 \,(-68.6) \\
PCD-1-0.1& -55.65 $\pm$0.86 \,(-360.6) & \textbf{-54.75} $\pm$1.46 \,(-91.2) & -56.65 $\pm$3.88 \,(-97.3) & -57.27 $\pm$4.69 \,(-84.3) \\
PCD-1-0.05& -54.29 $\pm$1.25 \,(-167.4) & \textbf{-53.60} $\pm$1.48 \,(-67.2) & -56.50 $\pm$5.26 \,(-72.5) & -58.16 $\pm$5.50 \,(-70.6) \\
PCD-1-0.01& \textbf{-54.26} $\pm$0.79 \,(-55.3) & -56.68 $\pm$0.73 \,(-56.8) & -60.83 $\pm$3.76 \,(-61.0) & -64.52 $\pm$2.94 \,(-64.6) \\
PT$_{10}$-0.1& -52.55 $\pm$3.43 \,(-202.5) & \textbf{-51.13} $\pm$0.85 \,(-52.1) & -55.37 $\pm$5.44 \,(-56.7) & -53.99 $\pm$3.73 \,(-55.3) \\
PT$_{10}$-0.05& \textbf{-51.84} $\pm$0.98 \,(-70.7) & \textbf{-51.87} $\pm$1.05 \,(-52.3) & -56.11 $\pm$5.79 \,(-56.6) & -56.06 $\pm$4.50 \,(-56.8) \\
PT$_{10}$-0.01& \textbf{-53.36} $\pm$1.26 \,(-53.8) & -56.73 $\pm$0.77 \,(-56.8) & -61.24 $\pm$4.58 \,(-61.3) & -64.70 $\pm$3.53 \,(-64.7) \\
\hline
\abovespace
\emph{MNIST} &  &    &    &  \\
\abovespace
CD-1-0.1& -152.6 $\pm$0.89 \,(-158.5) & \textbf{-150.9} $\pm$1.53 \,(-154.6) & \textbf{-151.3} $\pm$1.77 \,(-154.8) & -165.9 $\pm$1.90 \,(-168.4) \\
CD-1-0.05& -152.5 $\pm$1.14 \,(-156.1) & \textbf{-151.2} $\pm$1.89 \,(-154.3) & \textbf{-151.6} $\pm$1.90 \,(-154.6) & -167.7 $\pm$1.66 \,(-169.0) \\
CD-1-0.01& \textbf{-153.0} $\pm$1.10 \,(-153.2) & \textbf{-152.4} $\pm$1.81 \,(-152.8) & \textbf{-153.5} $\pm$2.30 \,(-154.0) & -171.3 $\pm$1.49 \,(-172.4) \\
PCD-1-0.1& -147.5 $\pm$1.09 \,(-177.6) & \textbf{-140.9} $\pm$0.61 \,(-145.2) & -142.9 $\pm$0.74 \,(-147.2) & -160.7 $\pm$4.87 \,(-169.4) \\
PCD-1-0.05& -145.3 $\pm$0.61 \,(-162.4) & \textbf{-140.0} $\pm$0.45 \,(-142.8) & -141.1 $\pm$0.65 \,(-143.6) & -173.4 $\pm$4.42 \,(-178.1) \\
PCD-1-0.01& -143.0 $\pm$0.29 \,(-144.7) & \textbf{-140.7} $\pm$0.42 \,(-141.4) & -141.7 $\pm$0.49 \,(-142.5) & -198.0 $\pm$4.78 \,(-198.4) \\
PT$_{10}$-0.01& \textbf{-247.1} $\pm$12.52 \,(-643.4) & \textbf{-141.5} $\pm$0.54 \,(-143.6) & -144.0 $\pm$0.61 \,(-147.6) & -148.8 $\pm$1.15 \,(-153.6) \\
\hline
\end{tabular}
\end{sc}
\end{small}
\end{center}
\caption {Maximum average LL during training on (top) the \emph{Bars \& Stripes} dataset and (bottom) the \emph{MNIST} dataset using different sampling methods and learning rates $\eta$.}
\label{tab:originalMethods1}
\end{table*}
\end{landscape}

\begin{landscape}
\begin{table*}
\setlength{\tabcolsep}{3pt}
\begin{center}
\begin{small}
\begin{sc}
\hspace*{-2.1cm}
\begin{tabular}{l@{\hskip -0.07in} l l l l}
\hline
\abovespace\belowspace
Algorithm-$\eta$ & \multicolumn{1}{c}{$aa^b$} & \multicolumn{1}{c}{$dd_s^l$} & \multicolumn{1}{c}{\emph{d0}} & \multicolumn{1}{c}{\emph{00}} \\
\hline
\abovespace
Shifting Bar  &  &  &  &  \\
\abovespace
CD-1-0.2& \textbf{-20.52} $\pm$1.09 \,(-21.9) & \textbf{-20.32} $\pm$0.74 \,(-20.6) & -21.72 $\pm$1.21 \,(-22.5) & -21.89 $\pm$1.42 \,(-22.6) \\
CD-1-0.1& \textbf{-20.97} $\pm$1.14 \,(-21.5) & \textbf{-20.79} $\pm$0.86 \,(-20.9) & \textbf{-21.19} $\pm$0.82 \,(-21.4) & -21.40 $\pm$0.88 \,(-21.6) \\
CD-1-0.05& \textbf{-21.11} $\pm$0.78 \,(-21.2) & -22.72 $\pm$0.67 \,(-22.7) & -26.89 $\pm$0.29 \,(-26.9) & -26.11 $\pm$0.40 \,(-26.1) \\
PCD-1-0.2& -21.71 $\pm$0.81 \,(-237.2) & \textbf{-21.02} $\pm$0.52 \,(-32.4) & -21.62 $\pm$0.66 \,(-31.9) & -21.86 $\pm$0.75 \,(-31.7) \\
PCD-1-0.1& -21.10 $\pm$0.59 \,(-87.4) & \textbf{-20.92} $\pm$0.73 \,(-23.3) & -21.74 $\pm$0.76 \,(-23.7) & -21.52 $\pm$0.89 \,(-23.3) \\
PCD-1-0.05& \textbf{-20.96} $\pm$0.70 \,(-26.0) & -22.48 $\pm$0.60 \,(-22.6) & -26.83 $\pm$0.36 \,(-26.8) & -26.04 $\pm$0.48 \,(-26.1) \\
PT$_{10}$-0.2& -20.87 $\pm$0.86 \,(-31.9) & \textbf{-20.38} $\pm$0.77 \,(-20.9) & -21.14 $\pm$1.15 \,(-21.6) & -21.82 $\pm$1.22 \,(-22.3) \\
PT$_{10}$-0.1& \textbf{-20.57} $\pm$0.60 \,(-21.5) & \textbf{-20.51} $\pm$0.58 \,(-20.7) & -21.22 $\pm$0.91 \,(-21.4) & -21.06 $\pm$0.92 \,(-21.2) \\
PT$_{10}$-0.05& \textbf{-20.69} $\pm$0.89 \,(-20.8) & -22.39 $\pm$0.68 \,(-22.4) & -26.94 $\pm$0.30 \,(-27.0) & -26.17 $\pm$0.38 \,(-26.2) \\
\hline
\abovespace
Flipped Shifting Bar &  &    &    &  \\
\abovespace
CD-1-0.2& \textbf{-20.39} $\pm$0.86 \,(-21.3) & \textbf{-20.42} $\pm$0.80 \,(-20.8) & -21.55 $\pm$1.33 \,(-22.3) & -27.98 $\pm$0.26 \,(-28.2) \\
CD-1-0.1& \textbf{-20.57} $\pm$0.83 \,(-20.9) & -20.85 $\pm$0.82 \,(-21.0) & -21.04 $\pm$0.75 \,(-21.2) & -28.28 $\pm$0.00 \,(-28.4) \\
CD-1-0.05& \textbf{-21.11} $\pm$0.77 \,(-21.2) & -22.63 $\pm$0.66 \,(-22.6) & -26.85 $\pm$0.34 \,(-26.9) & -28.28 $\pm$0.00 \,(-28.3) \\
PCD-1-0.2& -21.56 $\pm$0.57 \,(-310.8) & \textbf{-20.97} $\pm$0.65 \,(-32.3) & -21.89 $\pm$0.86 \,(-32.6) & -28.01 $\pm$0.26 \,(-28.3) \\
PCD-1-0.1& -21.17 $\pm$0.60 \,(-88.3) & \textbf{-20.72} $\pm$0.50 \,(-23.1) & -21.28 $\pm$0.71 \,(-23.2) & -28.28 $\pm$0.00 \,(-28.4) \\
PCD-1-0.05& \textbf{-21.01} $\pm$0.77 \,(-25.6) & -22.30 $\pm$0.64 \,(-22.4) & -26.90 $\pm$0.34 \,(-26.9) & -28.28 $\pm$0.00 \,(-28.3) \\
PT$_{10}$-0.2& -20.60 $\pm$0.66 \,(-33.2) & \textbf{-20.25} $\pm$0.55 \,(-20.7) & -20.79 $\pm$0.87 \,(-21.4) & -28.01 $\pm$0.27 \,(-28.2) \\
PT$_{10}$-0.1& \textbf{-20.78} $\pm$0.82 \,(-21.6) & \textbf{-20.68} $\pm$0.69 \,(-20.8) & -21.11 $\pm$0.67 \,(-21.3) & -28.28 $\pm$0.00 \,(-28.4) \\
PT$_{10}$-0.05& \textbf{-20.90} $\pm$0.85 \,(-21.1) & -22.39 $\pm$0.65 \,(-22.4) & -26.87 $\pm$0.35 \,(-26.9) & -28.28 $\pm$0.00 \,(-28.3) \\
\hline
\end{tabular}
\end{sc}
\end{small}
\end{center}
\caption {Maximum average LL during training on (top) the \emph{Shifting Bar} dataset and (bottom) the \emph{Flipped Shifting Bar} dataset using different sampling methods and learning rates.}
\label{tab:originalMethods2}
\end{table*}
\end{landscape}

\subsection{Initialization}\label{sec:initExp}

\begin{table}
\begin{center}
\begin{small}
\begin{sc}
\begin{tabular}{l l l}
\hline
\abovespace\belowspace
Algorithm-$\eta$ & \multicolumn{1}{c}{$00~init ~zero$} &   \multicolumn{1}{c}{$00~init ~ \sigma^{-1}$} \\
\hline
\abovespace
CD-1-0.2& -27.98 $\pm$0.26 \,(-28.2) & \textbf{-21.49} $\pm$1.34 \,(-22.5) \\
CD-1-0.1& -28.28 $\pm$0.00 \,(-28.4) & \textbf{-21.09} $\pm$0.97 \,(-21.6) \\
CD-1-0.05& -28.28 $\pm$0.00 \,(-28.3) & \textbf{-24.87} $\pm$0.47 \,(-24.9) \\
PCD-1-0.2& -28.01 $\pm$0.26 \,(-28.3) & \textbf{-22.45} $\pm$1.00 \,(-42.3) \\
PCD-1-0.1& -28.28 $\pm$0.00 \,(-28.4) & \textbf{-21.76} $\pm$0.74 \,(-26.7) \\
PCD-1-0.05& -28.28 $\pm$0.00 \,(-28.3) & \textbf{-24.83} $\pm$0.55 \,(-25.0) \\
PT$_{10}$-0.2& -28.01 $\pm$0.27 \,(-28.2) & \textbf{-21.72} $\pm$1.24 \,(-23.5) \\
PT$_{10}$-0.1& -28.28 $\pm$0.00 \,(-28.4) & \textbf{-21.14} $\pm$0.85 \,(-21.8) \\
PT$_{10}$-0.05& -28.28 $\pm$0.00 \,(-28.3) & \textbf{-24.80} $\pm$0.52 \,(-24.9)\\
\hline
\end{tabular}
\end{sc}
\end{small}
\end{center}
\caption {Maximum average LL during training for \emph{00} on the \emph{Flipped Shifting Bar} dataset, where 
 the visible bias is initialized to zero or to the inverse sigmoid of the data mean.}
\label{tab:init00_SB_FLIP}
\end{table}

\begin{table}
\begin{center}
\begin{small}
\begin{sc}
\begin{tabular}{l l l }
\hline
\abovespace\belowspace
Algorithm-$\eta$ & \multicolumn{1}{c}{$00~init ~zero$} &   \multicolumn{1}{c}{$00~init ~ \sigma^{-1}$} \\
\hline
\abovespace
CD-1-0.1& \textbf{-165.91} $\pm$1.90 \,(-168.4) & -167.61 $\pm$1.44 \,(-168.9) \\
CD-1-0.05& \textbf{-167.68} $\pm$1.66 \,(-169.0) & -168.72 $\pm$1.36 \,(-170.8) \\
CD-1-0.01& -171.29 $\pm$1.49 \,(-172.4) & \textbf{-168.29} $\pm$1.54 \,(-171.1) \\
PCD-1-0.1& -160.74 $\pm$4.87 \,(-169.4) & \textbf{-147.56} $\pm$1.17 \,(-156.3) \\
PCD-1-0.05& -173.42 $\pm$4.42 \,(-178.1) & \textbf{-144.20} $\pm$0.97 \,(-149.7) \\
PCD-1-0.01& -198.00 $\pm$4.78 \,(-198.4) & \textbf{-144.06} $\pm$0.47 \,(-145.0) \\
PT$_{10}$-0.01& -148.76 $\pm$1.15 \,(-153.6) & \textbf{-145.63} $\pm$0.66 \,(-149.4) \\
\hline
\end{tabular}
\end{sc}
\end{small}
\end{center}
\caption {Maximum average LL during training for \emph{00} on the \emph{MNIST} dataset, where the visible bias is initialized to zero or to the inverse sigmoid of the data mean.}
\label{tab:InitMNIST00}
\end{table}

\begin{table}
\begin{center}
\begin{small}
\begin{sc}
\begin{tabular}{l l l }
\hline
\abovespace\belowspace
Algorithm-$\eta$ & \multicolumn{1}{c}{$dd_s^l~init ~zero$} &   \multicolumn{1}{c}{$dd_s^l~init ~ \sigma^{-1}$} \\
\hline
\abovespace
CD-1-0.2& \textbf{-20.34} $\pm$0.74 \,(-20.6) & \textbf{-20.42} $\pm$0.80 \,(-20.8) \\
CD-1-0.1& \textbf{-20.75} $\pm$0.79 \,(-20.9) & \textbf{-20.85} $\pm$0.82 \,(-21.0) \\
CD-1-0.05& -23.00 $\pm$0.72 \,(-23.0) & \textbf{-22.63} $\pm$0.66 \,(-22.6) \\
PCD-1-0.2& \textbf{-21.03} $\pm$0.51 \,(-30.6) & \textbf{-20.97} $\pm$0.65 \,(-32.3) \\
PCD-1-0.1& \textbf{-20.86} $\pm$0.75 \,(-23.0) & \textbf{-20.72} $\pm$0.50 \,(-23.1) \\
PCD-1-0.05& -22.75 $\pm$0.66 \,(-22.8) & \textbf{-22.30} $\pm$0.64 \,(-22.4) \\
PT$_{10}$-0.2& \textbf{-20.08} $\pm$0.38 \,(-20.5) & \textbf{-20.25} $\pm$0.55 \,(-20.7) \\
PT$_{10}$-0.1& \textbf{-20.56} $\pm$0.69 \,(-20.7) & \textbf{-20.68} $\pm$0.69 \,(-20.8) \\
PT$_{10}$-0.05& -22.93 $\pm$0.72 \,(-22.9) & \textbf{-22.39} $\pm$0.65 \,(-22.4) \\
\hline
\end{tabular}
\end{sc}
\end{small}
\end{center}
\caption {Maximum average LL during training for $dd_s^l$ on the \emph{Flipped Shifting Bar} dataset, where 
 the visible bias is initialized to zero or to the inverse sigmoid of the data mean.}
\label{tab:initdd_SB_FLIP}
\end{table}

The following set of experiments was done to analyze the effects of different initializations of the parameters as discussed in Section~\ref{sec:initialization}. First, we trained normal binary RBMs (that is \emph{00}) where the visible bias was initialized to zero or to the inverse sigmoid of the data mean. 
In both cases the hidden bias was initialized to zero.
Table~\ref{tab:init00_SB_FLIP} shows the results for normal binary RBMs trained on the \emph{Flipped Shifting Bar} dataset, where RBMs with zero initialization failed to learn the distribution accurately. 
The RBMs using the inverse sigmoid initialization achieved good performance and therefore seem to be less sensitive to the `difficult' representation of the data. 
However, the results are not as good as the results of the centered RBMs shown in Table~\ref{tab:originalMethods2}.
The same observations can be made when training RBMs on the \emph{MNIST} dataset (see Table~\ref{tab:InitMNIST00}).
The RBMs with inverse sigmoid initialization achieved significantly better results than RBMs initialized to zero in the case of PCD and PT, but still worse compared to the centered RBMs.
Furthermore, using the inverse sigmoid initialization allows us to achieve similar performance on the flipped and normal version of the \emph{MNIST} dataset, while the RBM with zero initialization failed to learn \emph{1-MNIST} at all. 

\begin{table}[t]
\begin{center}
\begin{small}
\begin{sc}
\begin{tabular}{l l l}
\hline
\abovespace\belowspace
Algorithm-$\eta$ &   \multicolumn{1}{c}{$dd_s^b$} & \multicolumn{1}{c}{$dd_s^l$} \\
\hline
\abovespace
Bars \& Stripes &  &      \\
CD-1-0.1 & \textbf{-60.34} $\pm$2.18 & \textbf{-60.41} $\pm$2.08  \\
CD-1-0.05 & \textbf{-60.19} $\pm$1.98 & \textbf{-60.25} $\pm$2.13  \\
CD-1-0.01 & \textbf{-61.23} $\pm$1.49 & \textbf{-61.22} $\pm$1.49  \\
PCD-1-0.1 & \textbf{-54.86} $\pm$1.52 & \textbf{-54.75} $\pm$1.46  \\
PCD-1-0.05 & \textbf{-53.71} $\pm$1.45 & \textbf{-53.60} $\pm$1.48  \\
PCD-1-0.01 & \textbf{-56.68} $\pm$0.74 & \textbf{-56.68} $\pm$0.73  \\
PT$_{10}$-0.1 & \textbf{-51.25} $\pm$1.09 & \textbf{-51.13} $\pm$0.85 \\
PT$_{10}$-0.05 & \textbf{-52.06} $\pm$1.38 & \textbf{-51.87} $\pm$1.05  \\
PT$_{10}$-0.01 & \textbf{-56.72} $\pm$0.77 & \textbf{-56.73} $\pm$0.77  \\
\hline
\abovespace
\emph{MNIST} &  &    \\
\abovespace
CD-1-0.1 & \textbf{-150.60} $\pm$1.55 & -150.87 $\pm$1.53\\
CD-1-0.05 & \textbf{-150.98} $\pm$1.90 & -151.21 $\pm$1.89  \\
CD-1-0.01 & \textbf{-152.23} $\pm$1.75 & -152.39 $\pm$1.81  \\
PCD-1-0.1 & \textbf{-141.11} $\pm$0.53 & \textbf{-140.89} $\pm$0.61  \\
PCD-1-0.05 & \textbf{-139.95} $\pm$0.47 & \textbf{-140.02} $\pm$0.45  \\
PCD-1-0.01 & \textbf{-140.67} $\pm$0.46 & \textbf{-140.68} $\pm$0.42  \\
PT$_{10}$-0.01 & \textbf{-141.56} $\pm$0.52 & \textbf{-141.46} $\pm$0.54  \\
\hline
\end{tabular}
\end{sc}
\end{small}
\end{center}
\caption {Maximum average LL during training on (top) the \emph{Bars \& Stripes} dataset and (bottom) the \emph{MNIST} dataset, using the reparameterization before ($dd_s^b$) and after ($dd_s^l$) the gradient update.}
\label{tab:algo_BAS}
\end{table}

Second, we trained models using the centering versions $dd$, $aa$, and \emph{d0}  comparing the initialization suggested in Section~\ref{sec:initialization} against the initialization to zero, where we observed that the different ways to initialize had little effect on the performance. 
In most cases the results show no significant difference in terms of the maximum LL reached during trials with different initializations or slightly better results were found when using the inverse sigmoid, which can be explained by the better starting point 
yielded by this initialization.
See Table~\ref{tab:initdd_SB_FLIP} for the results for $dd_s^l$ on the \emph{Bars \& Stripes} dataset as an example. We used the inverse sigmoid initialization in the following experiments since it was beneficial for normal RBMs.

\subsection{Reparameterization}

To investigate the effects of performing the reparameterization 
before or after the gradient update in the training of centered RBMs
(that is, the difference of the algorithm suggested here and 
the algorithm suggested by \citet{MontavonMueller-2012}),
we analyzed the learning behavior of $dd_s^b$ and $dd_s^l$ 
on all datasets.
The results for RBMs trained on the \emph{Bar \& Stripes} dataset are given in  Table~\ref{tab:algo_BAS}~(top). 
No significant difference between the two versions can be observed. 
The same result was obtained for the \emph{Shifting Bar} and \emph{Flipped Shifting Bar} dataset. 
The results for the \emph{MNIST} dataset are shown in Table~\ref{tab:algo_BAS}~(bottom). 
Here $dd_s^b$ performs slightly better than $dd_s^l$ in the case of CD and no difference could be observed for PCD and PT. 
Therefore, we reparameterize the RBMs before the gradient update in the remainder of this work.

\subsection{Analyzing the Model Mean Related Divergence Effect}\label{sec:divergence}

\begin{figure}[t]
\centering
\subfigure[Exact gradient with approximated offsets]{\label{subfig:TrueGrad_A}
\includegraphics[trim=0.6cm 0cm 1.55cm 0cm, clip=true,scale=0.412]{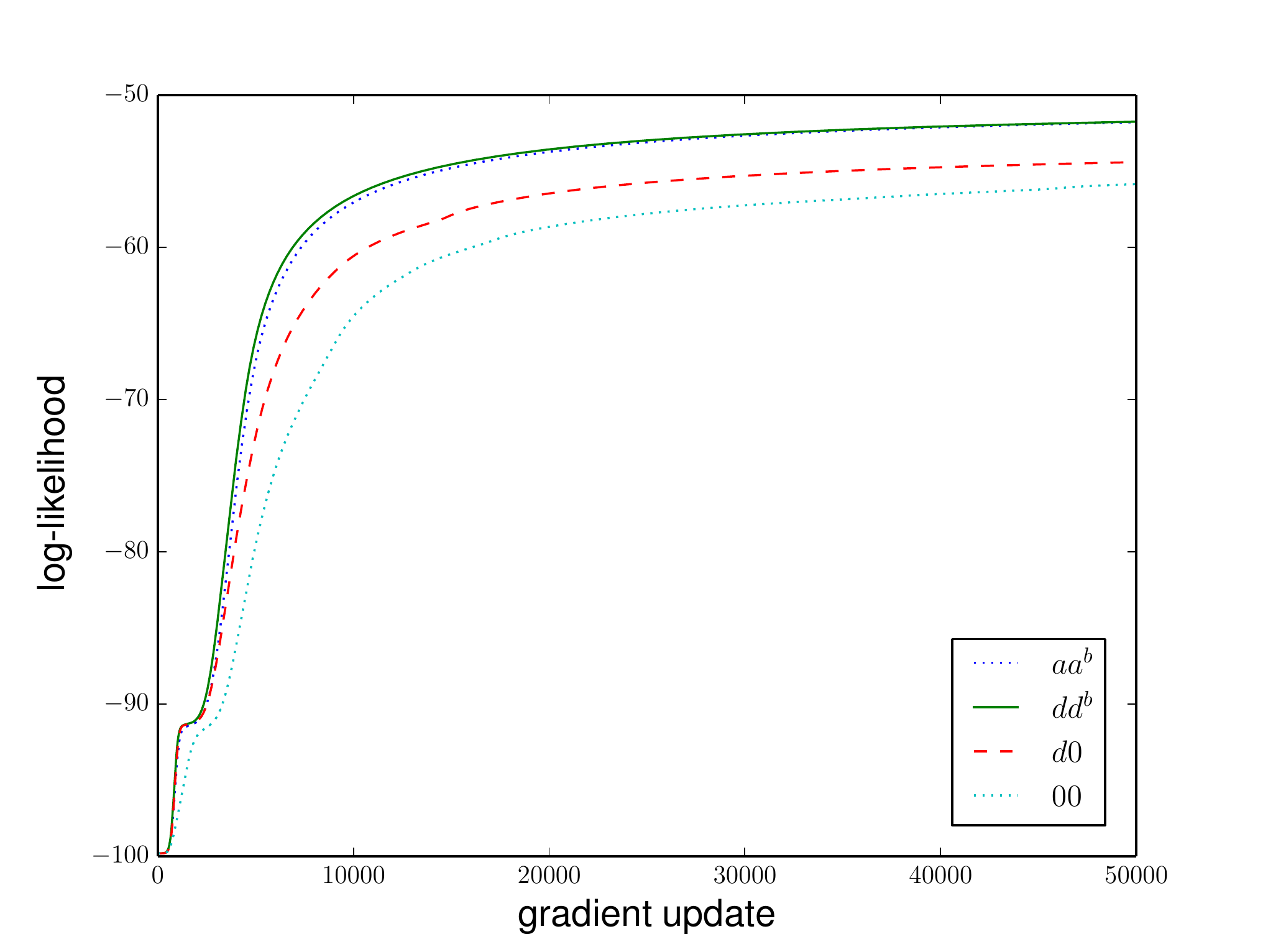}}
\subfigure[PT$_{10}$ with exact means]{\label{subfig:TrueGrad_B}
\includegraphics[trim=1.56cm 0cm 1.55cm 0cm, clip=true,scale=0.418]{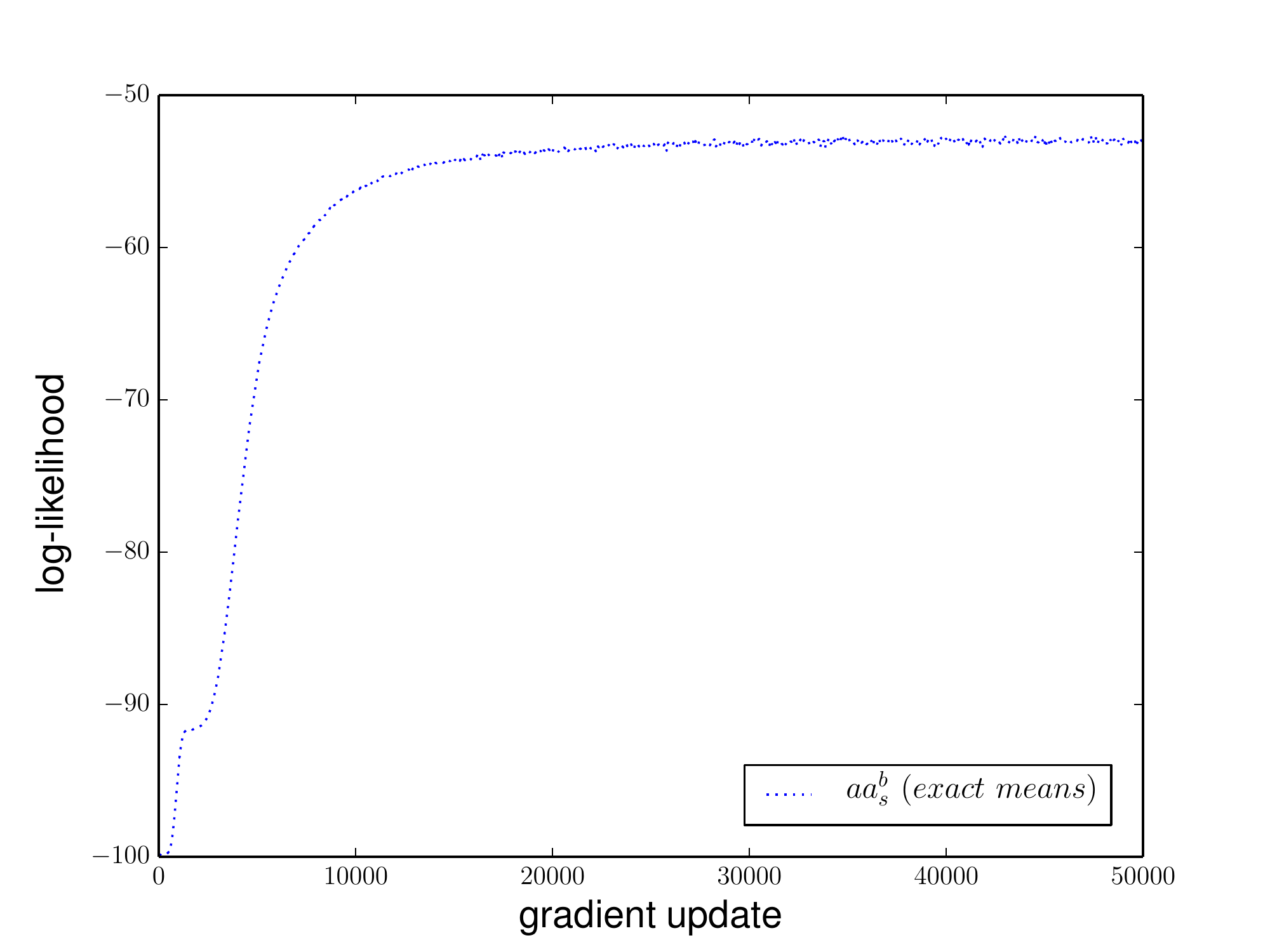}}
\caption{Evolution of the average LL during training on the \emph{Bars \& Stripes} dataset for the standard centering methods. 
(a) When the exact gradient is used, with approximated offsets and (b) when PT$_{10}$ is used for estimating the gradient while the mean values for the offsets are calculated exactly. In both cases a learning rate of $\eta=0.05$ was used.}
\label{fig:TrueGrad}
\end{figure} 
\begin{figure}[t]
\centering
\subfigure[Close up of the offset evolution exemplary for $\lambda_1$]{\label{subfig:para_div_A}
\includegraphics[trim=3.6cm 0cm 1.55cm 1cm, clip=true,scale=0.4]{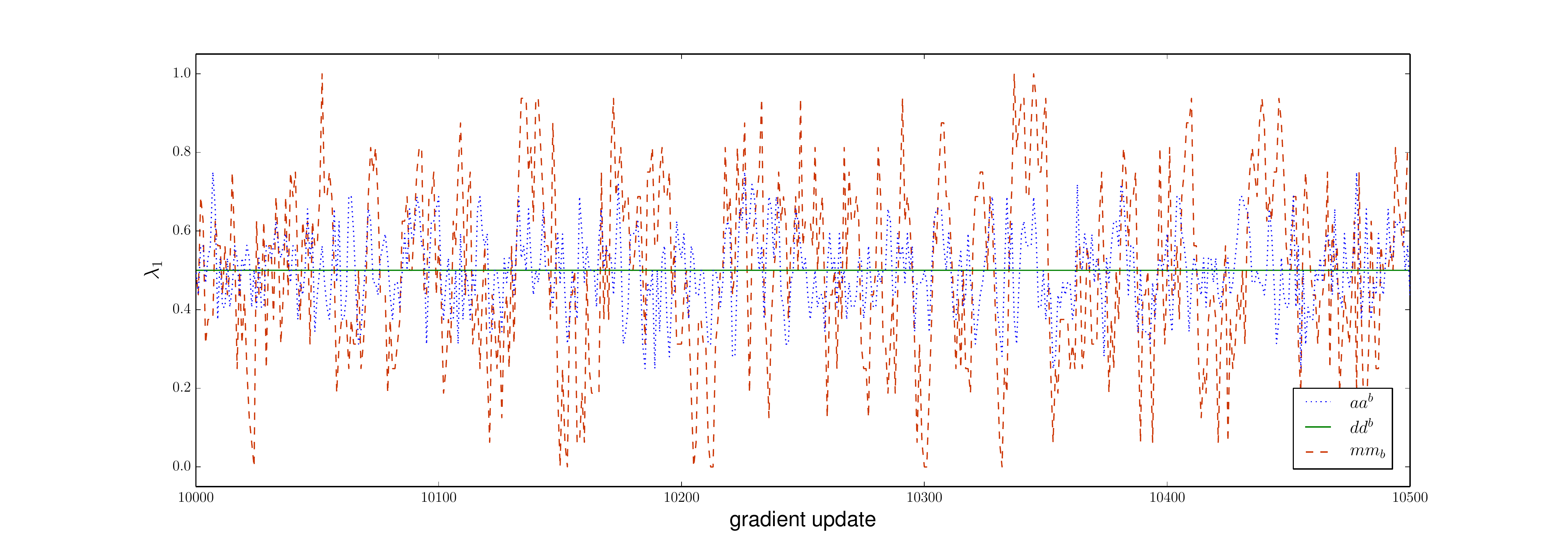}}
\subfigure[Evolution of the offset approximation error]{\label{subfig:para_div_B}
\includegraphics[trim=0.8cm 0cm 1.5cm 1cm, clip=true,scale=0.4]{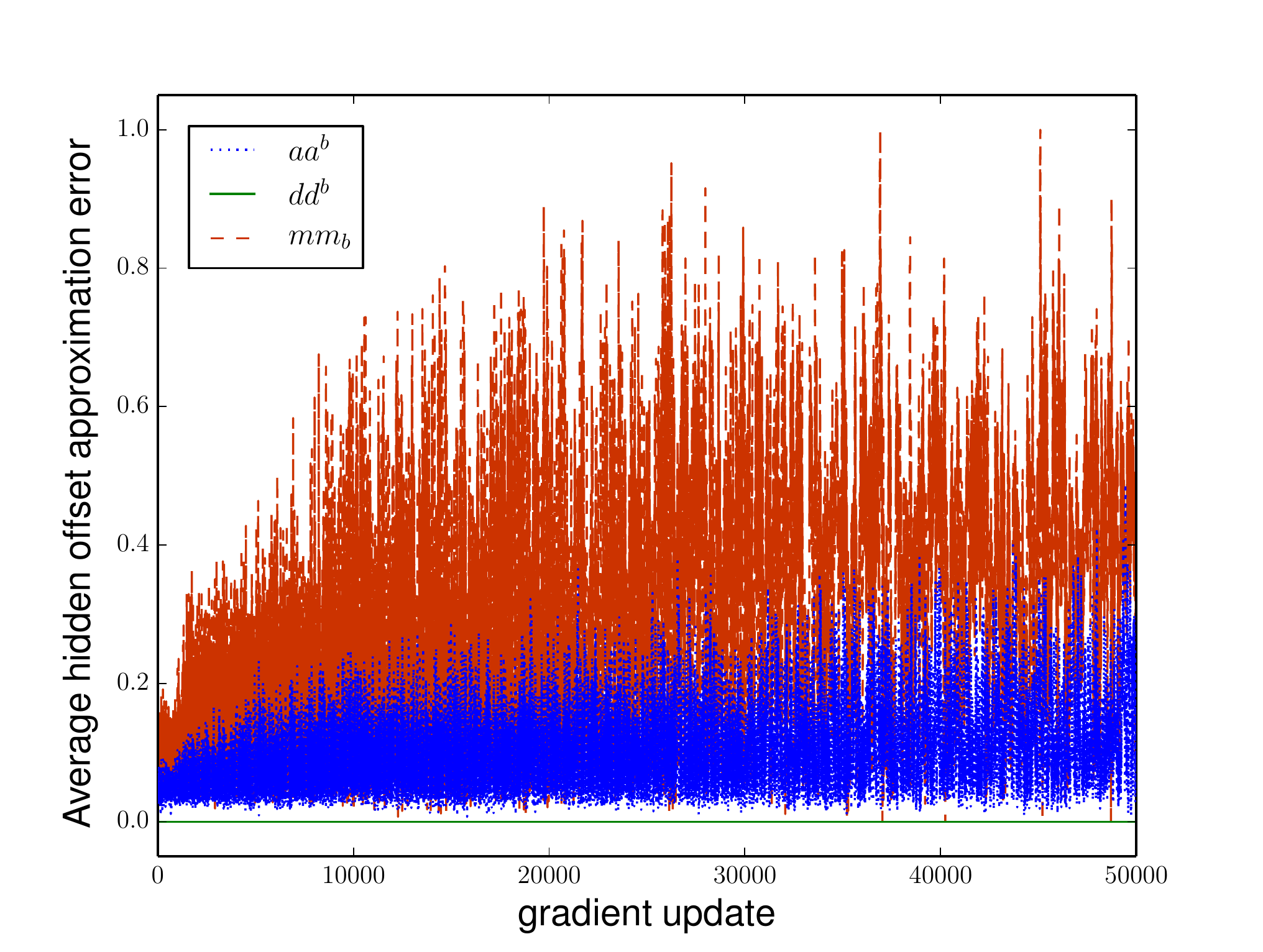}}
\subfigure[Evolution of the weight norm]{\label{subfig:para_div_C}
\includegraphics[trim=0.8cm 0cm 1.5cm 1cm, clip=true,scale=0.4]{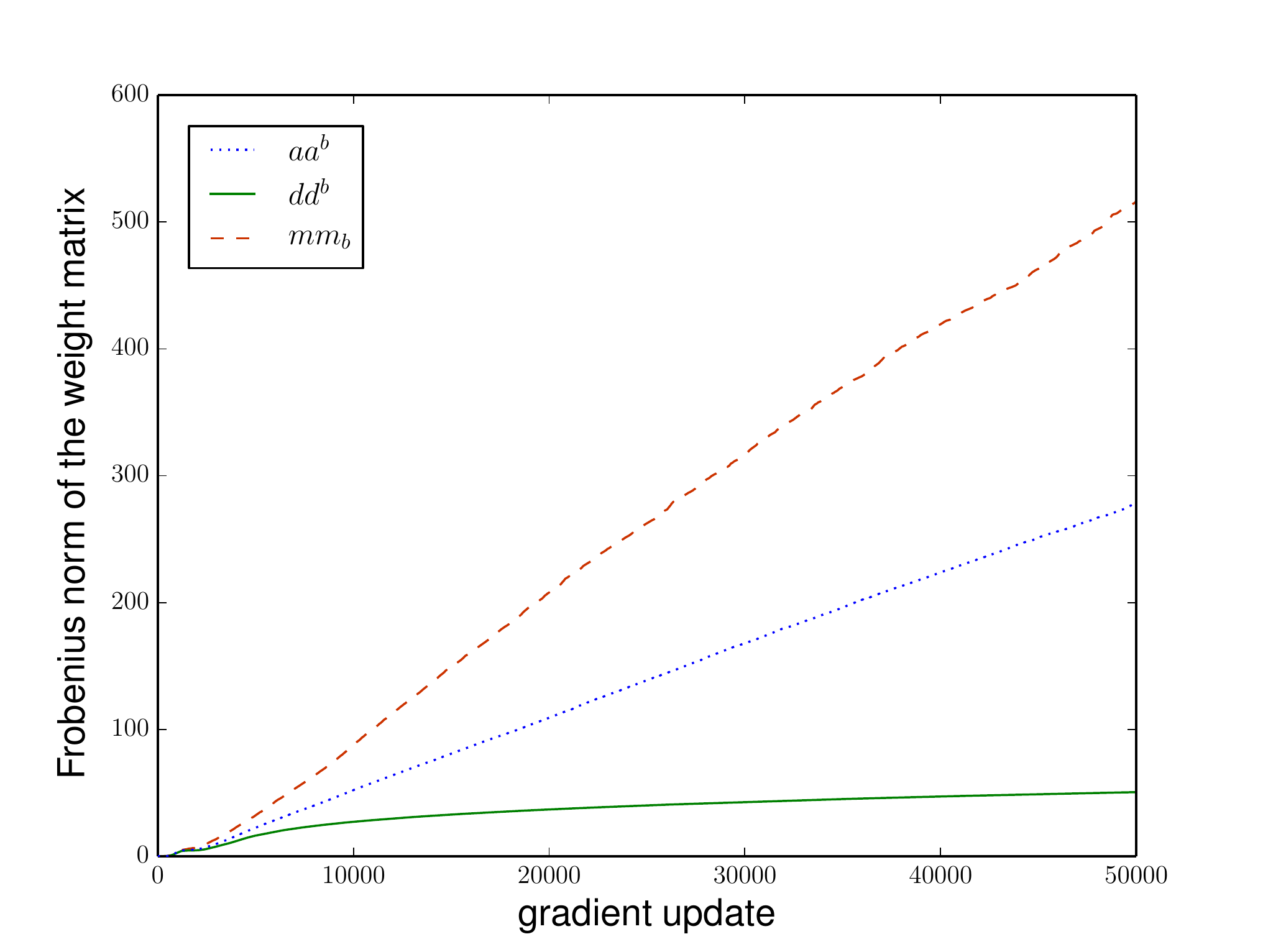}}
\caption{
Evolution of the offsets and weights of different centered variants for an RBM with 9 visible and 4 hidden units trained on \emph{Bars \& Stripes 3x3} using PT$_{10}$. For clearness a single trial is shown, but the experiments where repeated 25 trials where all results showed quantitatively the same results. (a) Close up of the evolution of offset $\lambda_1$ over 500 gradient updates. 
(b) The evolution of the absolute difference between exact and approximated hidden offset averaged of all hidden units. (c) The evolution of the Frobenius norm of the weight matrices. 
}
\label{fig:para_div}
\end{figure} 

The severe divergence  problem observed when using the enhanced gradient in combination with PCD or PT raises the question whether the problem is induced by setting the offsets to $ 0.5 (\langle \vect x \rangle_d + \langle \vect x \rangle_m)$ and $ 0.5 (\langle \vect h \rangle_d + \langle \vect h \rangle_m)$ or by bad sampling based estimates of gradient and offsets.
We therefore trained centered RBMs with 9 visible and 4 hidden units on the 2x2 \emph{Bars \& Stripes} dataset using either the exact gradient where only $\langle \vect x \rangle_m$ and $\langle \vect h \rangle_m$ were approximated by samples or using PT$_{10}$ estimates of the gradient while $\langle \vect x \rangle_m$ and $\langle \vect h \rangle_m$ were calculated exactly. The results are shown in Figure~\ref{fig:TrueGrad}.
If the true model expectations are used as offsets instead of the sample approximations no divergence for $aa$ in combination with PT is observed and the performance of $aa$ and $dd$ become almost equivalent.
Interestingly, the divergence is also prevented if one calculates the exact gradient while still approximating the offsets by samples. 
Thus, the divergence behavior must be induced by approximating both, gradient and offsets. 
The mean under the data distribution can always be calculated exactly, which might explain why we do not observe divergence for $dd$ in combination with PT. 
On the contrary, the divergence becomes even worse when using just the model mean (a centering variant which we denote by $mm$ in the following) instead of the average of data and model mean (as in $aa$) as offsets, as can be seen in Figure~\ref{fig:NoiseMean}.
Furthermore, the divergence also occurs if either visible or hidden offsets are set to the PT-approximated model mean, which can be seen for $dm$ in Figure~\ref{fig:move_average}(b).

To further deepen the understanding of the divergence effect we investigated the parameter evolution during training of RBMs with different offsets. We observed that the change of the offset values between two gradient updates gets extremely large during training when using the model mean. Figure~\ref{subfig:para_div_A} shows exemplary the evolution of the first hidden offset $\lambda_1$ for a single trial, where the offset approximation for $dd^b$ is almost constant while it is rather large for $aa^b$ and even bigger for $mm^b$. In each iteration we calculated the exact offsets to estimate the approximation error shown in Figure~\ref{subfig:para_div_B}. Obviously, there is no approximation error for $dd$ while the error for $aa$ quickly gets large and $mm$ gets even twice as big. In combination with the gradient approximation error this causes the weight matrices for $aa$ and $mm$ to grow extremely big as shown in Figure~\ref{subfig:para_div_C}.

\begin{figure}[t]
\centering
\subfigure[Random offsets]{\label{subfig:Divergence_noise}
\includegraphics[trim=0.6cm 0cm 1.55cm 0cm, clip=true,scale=0.412]{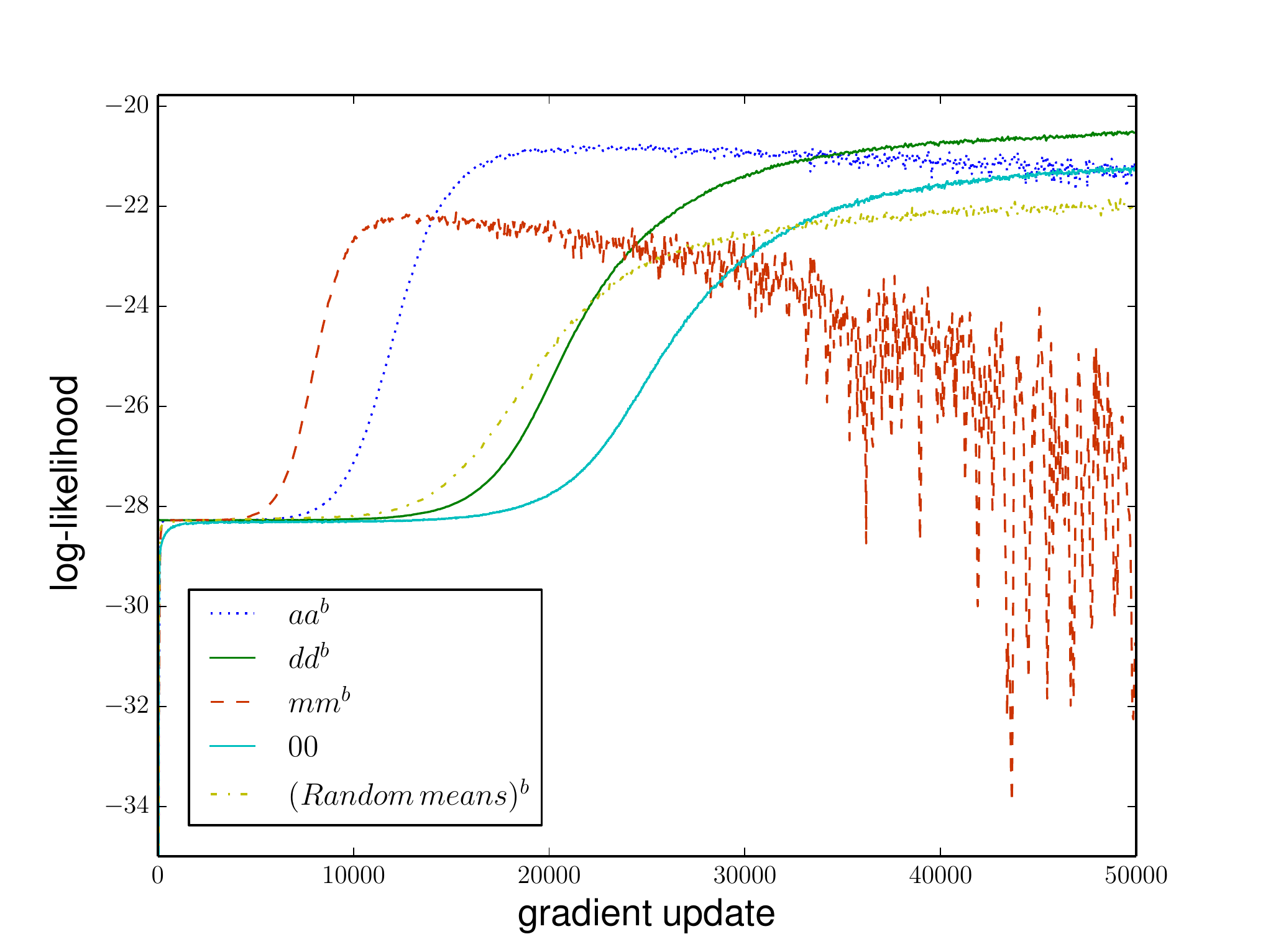}}
\subfigure[Independant approximations]{\label{subfig:IndiPT}
\includegraphics[trim=1.56cm 0cm 1.48cm 0cm, clip=true,scale=0.412]{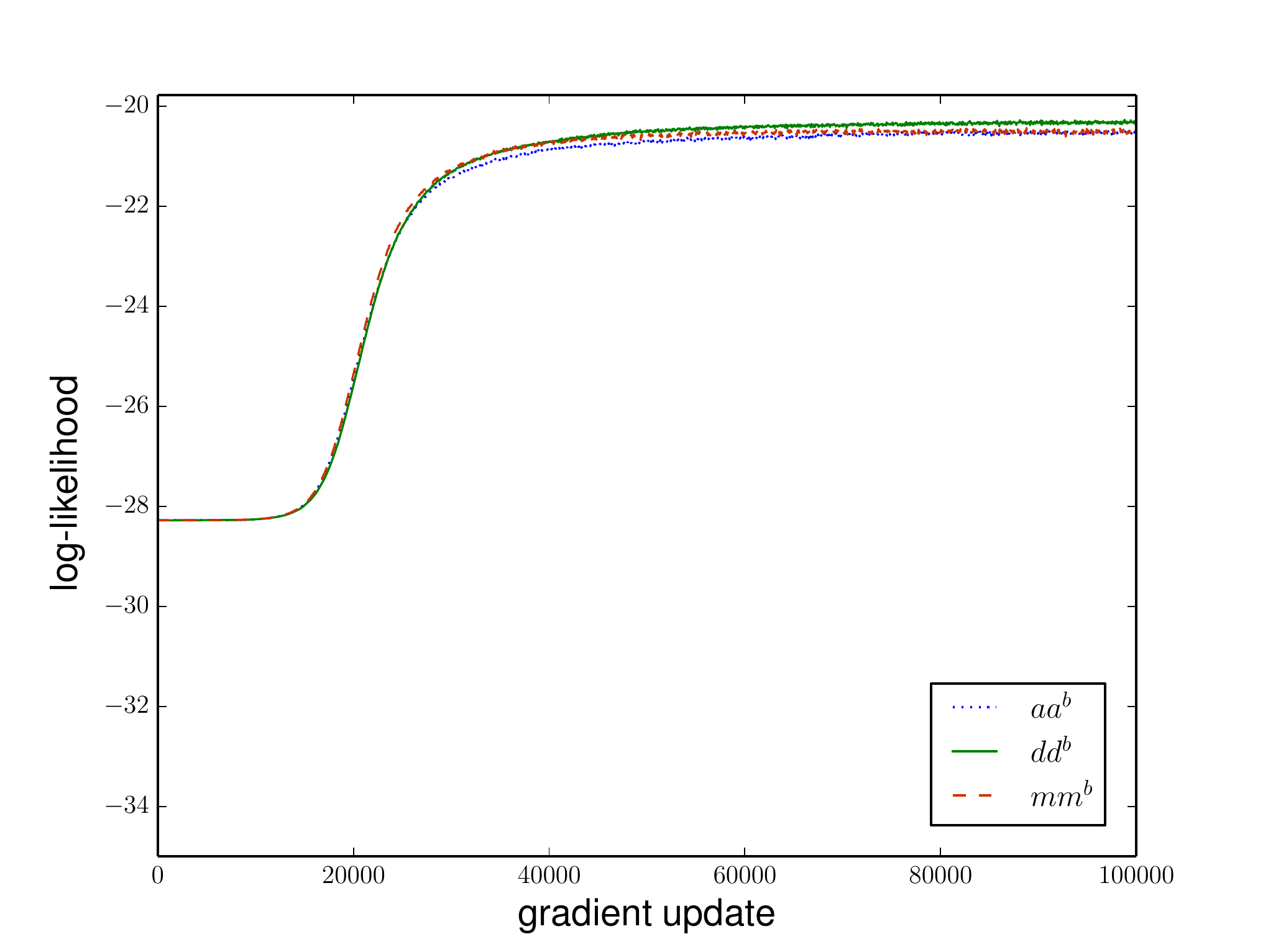}}
\caption{Evolution of the average LL during training on the \emph{Shifting Bars} dataset using PT$_{10}$ with learning rate of $\eta=0.1$. (a) normal RBMs, centered RBMs and centered RBMs using random offset values. (b) centered RBMs when the samples for the offset approximations come from a different Markov chain than the samples used for the gradient estimation.}
\label{fig:NoiseMean}
\end{figure} 
To verify that the divergence is not just caused by the additional sampling noise introduced by approximating the offsets, we trained a centered RBM using PT for the gradient approximation while we set the offsets to uniform random values between zero and one in each iteration. The results are shown in Figure~\ref{fig:NoiseMean}(a) and demonstrate that even random offset values do  not lead to the divergence problems. Thus, the divergence seems not be caused by additional sampling noise but rather by the correlated errors of gradient and offset approximations. 
To test this hypothesis we investigated $mm^b$ where the samples for offset and gradient approximations were taken from different PT sampling processes. The results are shown in Figure~\ref{fig:NoiseMean}(b) where no divergence can be observed anymore. 
While creating two sets of samples for gradient and offset approximations prevents the LL from diverging it almost doubles the computational cost and can therefore not be considered as a relevant solution in practice.
Moreover, using the model mean as offset still leads to slightly worse final LL values than using the mean under the data distribution. This might be explained by the fact that the additional approximation of the model mean introduces noise while the data mean can be estimated exactly.

Interestingly, the observed initially faster learning speed of $mm$ and $aa$,  which can be seen in Figure~\ref{fig:NoiseMean}(a), does not occur anymore when offset and gradient approximation are based on different sample sets. This observation can also be made when the exact gradient is used (see Figure~\ref{fig:TrueGrad}(a) and Figure~\ref{fig:TrueGradTrueAll}). Thus, the initially faster learning speed seems also be caused by the correlated approximations of gradient and offsets.

\subsection{Usage of an Exponentially Moving Average}

An exponentially moving average can be used to smooth the approximation of the offsets between parameter updates. This seems to be reasonable for stabilizing the approximations when small batch sizes are used as well as when the model mean is used for the offsets.
We therefore analyzed the impact of using an exponentially moving average with a sliding factor of $0.01$ for the estimation of the offset parameters.
Figure~\ref{subfig:BASPT005SHIFT_A} illustrates on the~\emph{Bars\& Stripes} dataset that the learning curves of the different models become almost equivalent when using an exponentially moving average. 
The maximum LL values reached are the same whether an exponentially moving average is used or not, 
 which can be seen by comparing Figure~\ref{subfig:BASPT005SHIFT_A} and Figure~\ref{subfig:BASPT005SHIFT_B} and also by comparing the results in Table~\ref{tab:originalMethods1} and Table~\ref{tab:originalMethods2} with those in Table~\ref{tab:final}. Notably, the divergence problem does not occur anymore when an exponentially moving average is used. 
As discussed in the previous section, this problem is caused by the correlation between the approximation error of gradient and offsets. When using an exponential moving average the current offsets contain only a small fraction of the current mean such that the correlation is reduced.
\\
In the previous experiments $dd$ was used with an exponentially moving average as
suggested for this centering variant by \citet{MontavonMueller-2012}.
Note however, that in batch learning when $\langle \vect x \rangle_d$ is used for the visible offsets, these values stay constant such that an exponentially moving average has no effect. More generally if the training data and thus $\langle \vect x \rangle_d$ is known in advance the visible offsets should be fixed to this value independent of whether batch, mini-batch or online learning is used. 
However, the use of an exponentially moving average for approximating $\langle \vect x \rangle_d$ is reasonable 
if the training data is not known in advance, as well as for the approximation of
the mean of the hidden representation $\langle \vect h \rangle_d$.
\\
In our experiments, $dd$ does not suffer from the divergence 
problem when PT is used for sampling, even without exponentially moving average, 
as can be seen in Figure~\ref{subfig:BASPT005SHIFT_B} for example. 
We did not even observe the divergence without a moving average in the
case of mini-batch learning. Thus, $dd$ seems to be generally more
stable than the other centering variants.

\subsection{Other Choices for the Offsets}

\begin{figure}[t]
\centering
\subfigure[Visible and hidden units centered]{\label{subfig:TrueGradTrueOffset_a}
\includegraphics[trim=0.6cm 0cm 1.55cm 0cm, clip=true,scale=0.412]{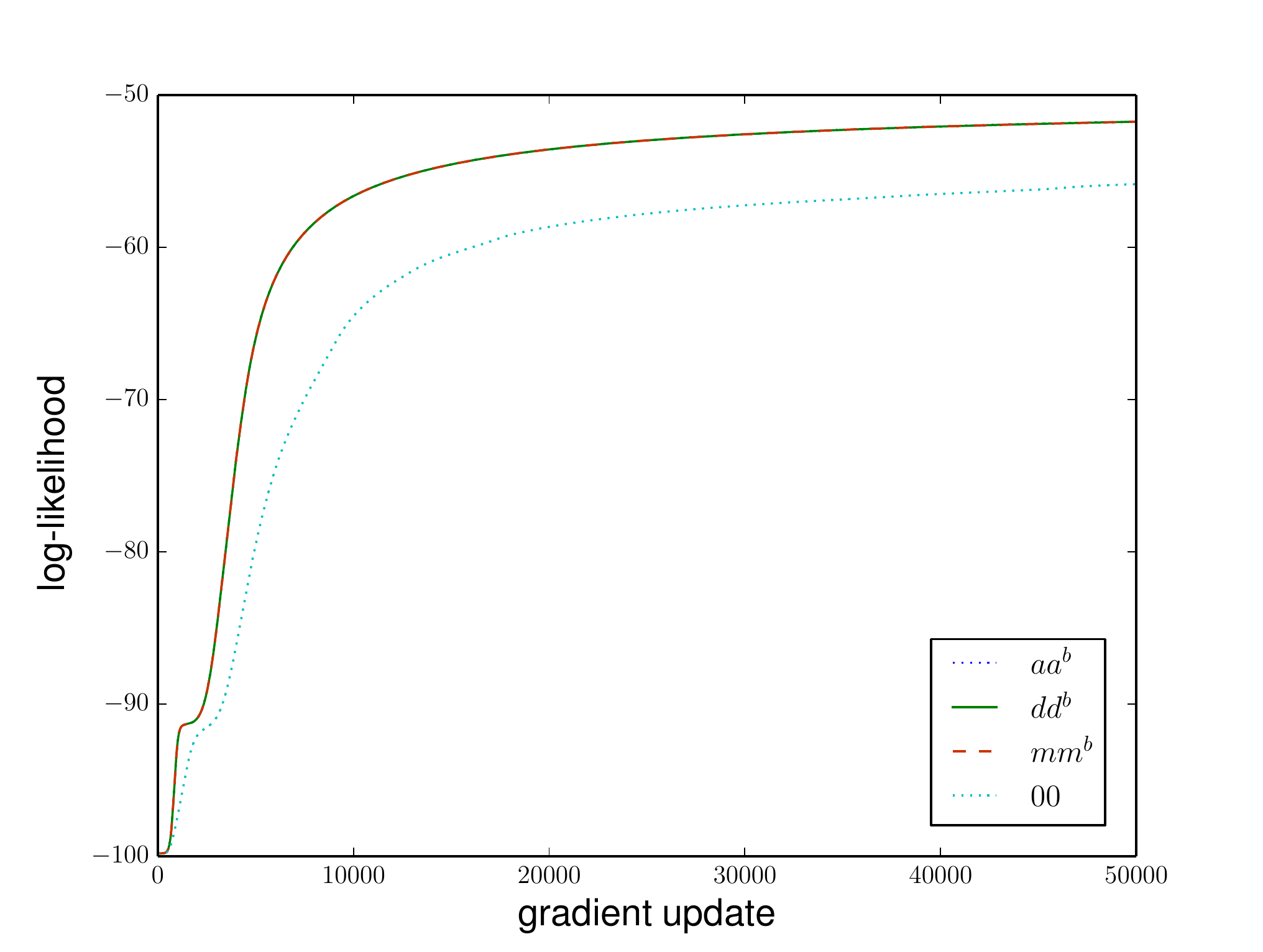}}
\subfigure[Visible or hidden units centered]{\label{subfig:TrueGradTrueOffset_b}
\includegraphics[trim=1.56cm 0cm 1.55cm 0cm, clip=true,scale=0.412]{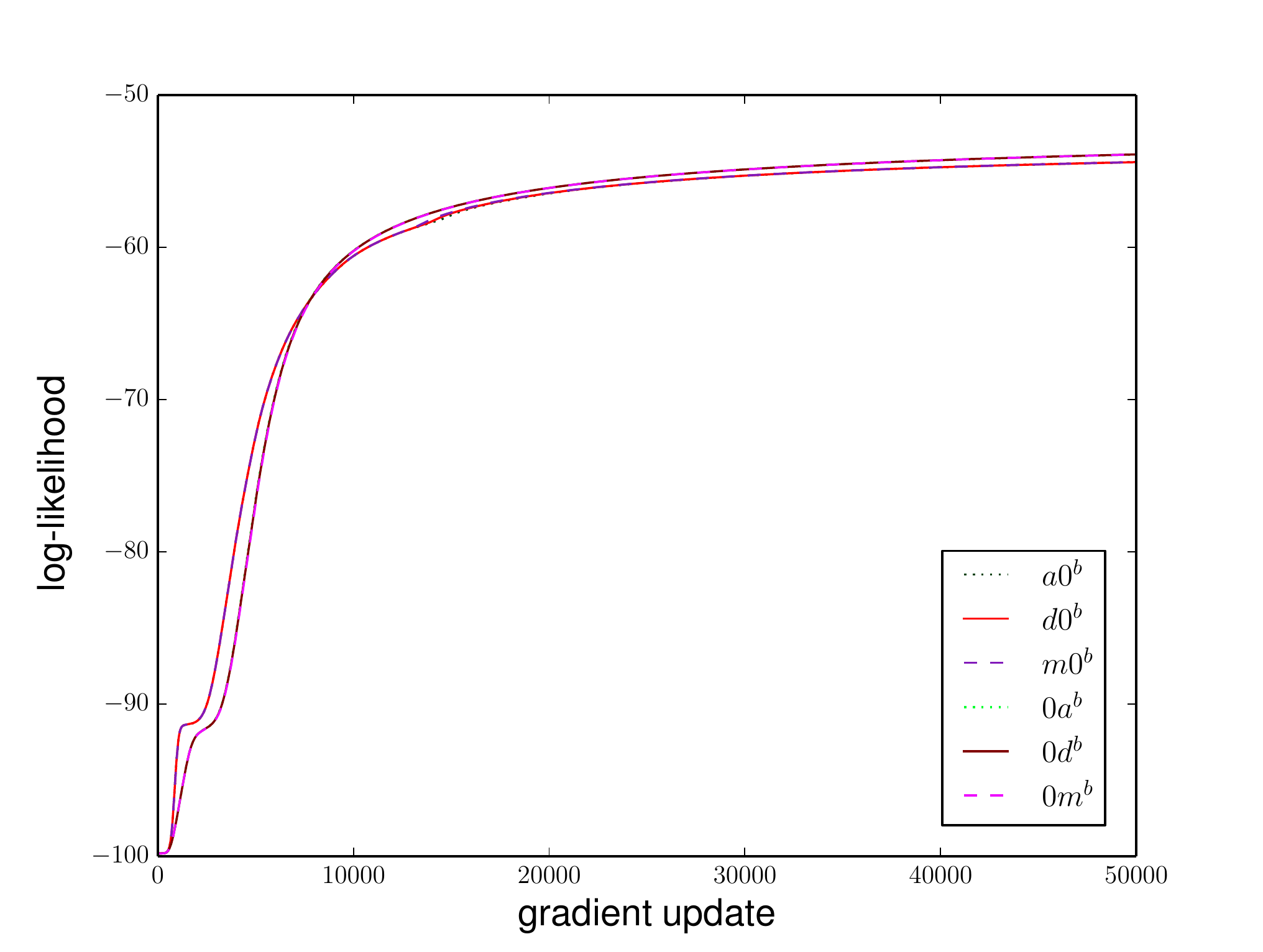}}
\caption{Evolution of the average LL during training on the \emph{Bars \& Stripes} dataset for the various centering methods when the exact gradient is used with the exact offsets and a learning rate of $\eta=0.05$.}
\label{fig:TrueGradTrueAll}
\end{figure} 

As discussed in Section~\ref{sec:CenteredRBMs}, any offset value between $0$ and $1$ guarantees the flip invariance property as long as it flips simultaneously with the data. 
An intuitive and constant choice is to set the offsets to $0.5$, 
which has also been proposed by \citet{OllivierArnoldEtAl-2013} and results in a symmetric variant of the energy of RBMs. 
This leads to comparable LL values on flipped and unflipped datasets. 
However, if the dataset is unbalanced in the amount of zeros and ones like \emph{MNIST}, the performances is always worse compared to that of a normal RBM on the version of the dataset which has less ones than zeros. 
Therefore, fixing the offset values to 0.5 cannot be considered as an alternative for centering using expectation values over  data or model distribution.

In Section~\ref{sec:CenteredRBMs}, we mentioned the existence of alternative offset parameters which lead to the same updates for the weights as the enhanced gradient. 
Setting $\vect{\mu} = \langle  \vect{x} \rangle_{d}$ and $\vect{\lambda} = \langle  \vect{h} \rangle_{m}$ seems reasonable since the data mean is usually known in advance. 
As mentioned above we refer to centering with this choice of offsets as $dm$.
We trained RBMs with $dm_s^b$ using a sliding factor of $0.01$. 
The results are shown in Table~\ref{tab:final} and suggest that there is no significant difference between $dm_s^b$, $aa_s^b$, and $dd_s^b$. 
However, without an exponentially moving average $dm^b$ has the same divergence problems as $aa^b$, as shown in Figure~\ref{subfig:BASPT005SHIFT_B}. 

We further tried variants like $mm$, $m0$, $0d$, $m0$ etc. but did not find better performance than that of $dd$ for any of these choices. 
The variants that subtract an offset from both,  visible and  hidden variables  outperformed or achieved the same performance as the variants that subtract an offset only from one type of variables. When the model expectation was used without a exponentially moving average either for  $\vect \mu$ or $\vect\lambda$, or for both offsets we always observed the divergence problem.

Interestingly, if the exact gradient and offsets are used for training no significant difference can be observed in terms of the LL evolution whether data mean, model mean or the average of both is used for the offsets as shown in Figure~\ref{fig:TrueGradTrueAll}. But centering both visible and hidden units still leads to better results than centering only one. Furthermore, the results illustrate that centered RBMs outperform normal binary RBMs also if the exact gradient is used for training the both models. This emphasizes that the worse performance of normal binary RBMs is caused by the properties of its gradient rather than by the gradient approximation.

\begin{table}[h!]
\setlength{\tabcolsep}{2pt}
\begin{center}
\begin{small}
\begin{sc}
\hspace*{-0.3cm}
\begin{tabular}{l@{\hskip -0.13in} l l l }
\hline
\abovespace\belowspace
Algorithm-$\eta$ & \multicolumn{1}{c}{$aa_s^b$} & \multicolumn{1}{c}{$dd_s^b$} & \multicolumn{1}{c}{$dm_s^b$} \\
\hline
\abovespace
Bars \& Stripes &  &   &  \\
\abovespace
CD-1-0.1& \textbf{-60.09} $\pm$2.02 \,(-69.6) & \textbf{-60.34} $\pm$2.18 \,(-69.9) & \textbf{-60.35} $\pm$1.99 \,(-68.8) \\
CD-1-0.05& \textbf{-60.31} $\pm$2.10 \,(-64.2) & \textbf{-60.19} $\pm$1.98 \,(-63.6) & \textbf{-60.25} $\pm$2.13 \,(-64.2) \\
CD-1-0.01& \textbf{-61.22} $\pm$1.50 \,(-61.3) & \textbf{-61.23} $\pm$1.49 \,(-61.3) & \textbf{-61.23} $\pm$1.49 \,(-61.3) \\
PCD-1-0.1& \textbf{-54.78} $\pm$1.63 \,(-211.7) & \textbf{-54.86} $\pm$1.52 \,(-101.0) & \textbf{-54.92} $\pm$1.49 \,(-177.3) \\
PCD-1-0.05& \textbf{-53.81} $\pm$1.58 \,(-89.9) & \textbf{-53.71} $\pm$1.45 \,(-67.7) & \textbf{-53.88} $\pm$1.54 \,(-83.3) \\
PCD-1-0.01& \textbf{-56.48} $\pm$0.74 \,(-56.7) & -56.68 $\pm$0.74 \,(-56.9) & \textbf{-56.47} $\pm$0.74 \,(-56.6) \\
PT$_{10}$-0.1& \textbf{-51.20} $\pm$1.11 \,(-52.4) & -51.25 $\pm$1.09 \,(-52.3) & \textbf{-51.10} $\pm$1.02 \,(-52.5) \\
PT$_{10}$-0.05& \textbf{-51.99} $\pm$1.39 \,(-52.6) & -52.06 $\pm$1.38 \,(-52.6) & \textbf{-51.82} $\pm$1.05 \,(-52.4) \\
PT$_{10}$-0.01& \textbf{-56.65} $\pm$0.77 \,(-56.7) & -56.72 $\pm$0.77 \,(-56.7) & \textbf{-56.67} $\pm$0.77 \,(-56.7) \\
\hline
\abovespace
Flipped Shifting Bar \hspace{-2cm} &  &  &  \\
\abovespace
CD-1-0.2& \textbf{-20.36} $\pm$0.74 \,(-20.7) & \textbf{-20.32} $\pm$0.69 \,(-20.6) & \textbf{-20.32} $\pm$0.70 \,(-20.6) \\
CD-1-0.1& \textbf{-20.80} $\pm$0.76 \,(-20.9) & \textbf{-20.86} $\pm$0.81 \,(-21.0) & \textbf{-20.69} $\pm$0.76 \,(-20.8) \\
CD-1-0.05& \textbf{-22.58} $\pm$0.64 \,(-22.6) & -22.64 $\pm$0.69 \,(-22.7) & -22.94 $\pm$0.73 \,(-23.0) \\
PCD-1-0.2& \textbf{-21.00} $\pm$0.65 \,(-41.5) & \textbf{-20.96} $\pm$0.49 \,(-31.0) & \textbf{-21.00} $\pm$0.68 \,(-38.3) \\
PCD-1-0.1& \textbf{-20.75} $\pm$0.53 \,(-23.4) & \textbf{-20.76} $\pm$0.53 \,(-22.8) & \textbf{-20.88} $\pm$0.70 \,(-23.2) \\
PCD-1-0.05& \textbf{-22.28} $\pm$0.68 \,(-22.3) & \textbf{-22.29} $\pm$0.64 \,(-22.3) & -22.68 $\pm$0.65 \,(-22.7) \\
PT$_{10}$-0.2& \textbf{-20.14} $\pm$0.45 \,(-20.7) & \textbf{-20.31} $\pm$0.61 \,(-20.7) & \textbf{-20.07} $\pm$0.38 \,(-20.5) \\
PT$_{10}$-0.1& \textbf{-20.42} $\pm$0.51 \,(-20.7) & \textbf{-20.46} $\pm$0.56 \,(-20.6) & \textbf{-20.60} $\pm$0.72 \,(-20.8) \\
PT$_{10}$-0.05& \textbf{-22.36} $\pm$0.64 \,(-22.4) & \textbf{-22.39} $\pm$0.69 \,(-22.4) & -22.86 $\pm$0.70 \,(-22.9) \\
\hline
\abovespace
\emph{MNIST} &  & &  \\
\abovespace
CD-1-0.1& \textbf{-150.61} $\pm$1.52 \,(-153.8) & \textbf{-150.60} $\pm$1.55 \,(-153.9) & \textbf{-150.50} $\pm$1.48 \,(-153.6) \\
CD-1-0.05& \textbf{-151.11} $\pm$1.55 \,(-153.2) & -150.98 $\pm$1.90 \,(-153.8) & \textbf{-150.80} $\pm$1.92 \,(-153.5) \\
CD-1-0.01& \textbf{-152.83} $\pm$2.42 \,(-153.3) & \textbf{-152.23} $\pm$1.75 \,(-152.6) & \textbf{-152.17} $\pm$1.72 \,(-152.5) \\
PCD-1-0.1& \textbf{-141.10} $\pm$0.64 \,(-145.4) & \textbf{-141.11} $\pm$0.53 \,(-145.7) & \textbf{-140.99} $\pm$0.56 \,(-144.8) \\
PCD-1-0.05& \textbf{-140.01} $\pm$0.58 \,(-142.9) & \textbf{-139.95} $\pm$0.47 \,(-142.6) & \textbf{-139.94} $\pm$0.46 \,(-142.7) \\
PCD-1-0.01& \textbf{-140.85} $\pm$0.47 \,(-141.6) & \textbf{-140.67} $\pm$0.46 \,(-141.4) & \textbf{-140.72} $\pm$0.39 \,(-141.5) \\
PT$_{10}$-0.01& -142.32 $\pm$0.47 \,(-145.7) & \textbf{-141.56} $\pm$0.52 \,(-143.3) & -142.18 $\pm$0.45 \,(-146.0) \\
\hline
\end{tabular}
\end{sc}
\end{small}
\end{center}
\caption{Maximum average LL during training on (top) \emph{Bars \& Stripes}, (middel) \emph{Flipped Shifting Bar}, and (bottom) \emph{MNIST} when using an exponentially moving average with an sliding factor of 0.01.} \label{tab:final}
\end{table}

\begin{figure}[t]
\centering
\subfigure[With exponentially moving average]{\label{subfig:BASPT005SHIFT_A}
\includegraphics[trim=0.6cm 0cm 1.57cm 0cm, clip=true,scale=0.413]{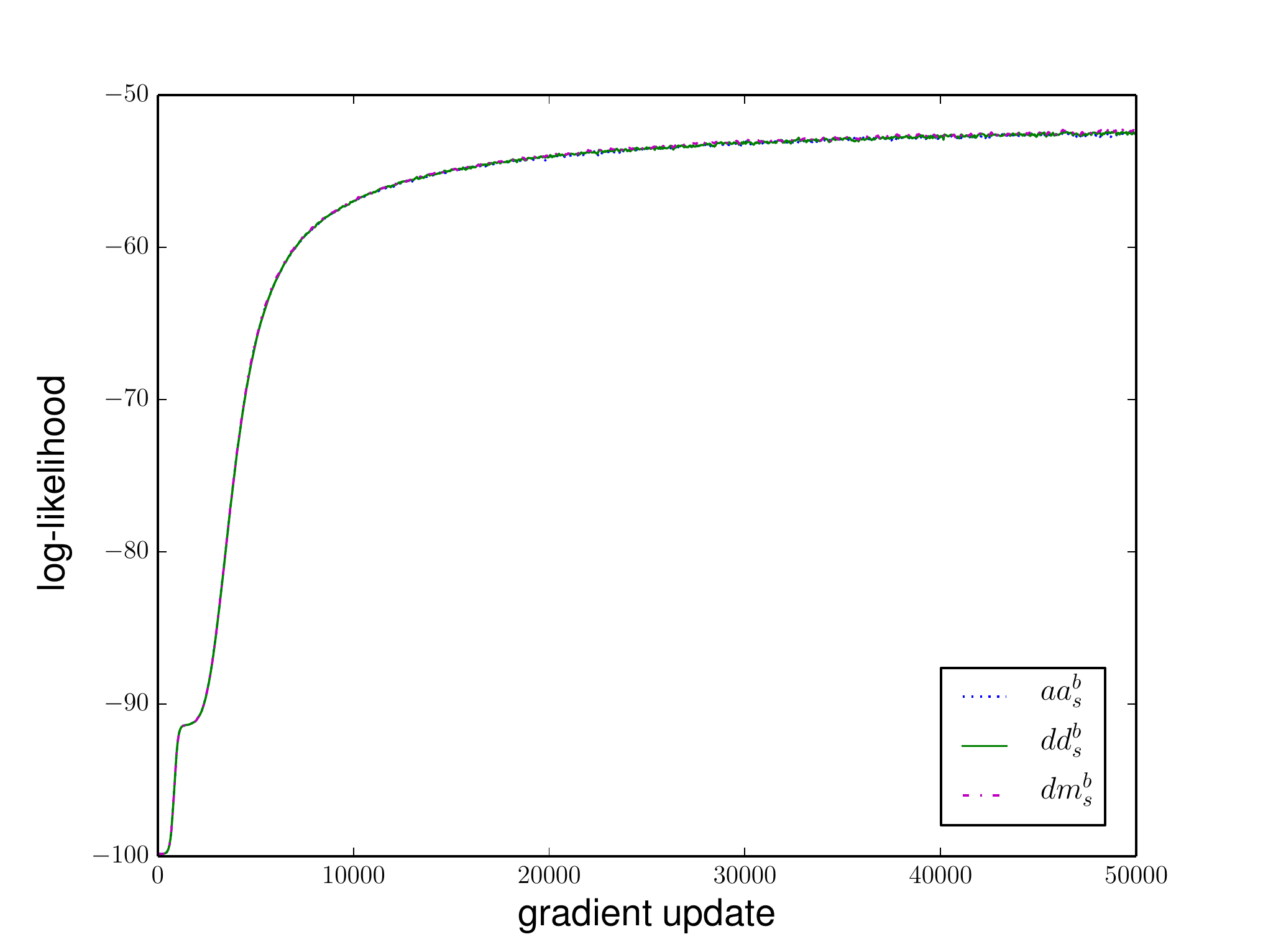}}
\subfigure[Without exponentially moving average]{\label{subfig:BASPT005SHIFT_B}
\includegraphics[trim=1.58cm 0cm 1.55cm 0cm, clip=true,scale=0.413]{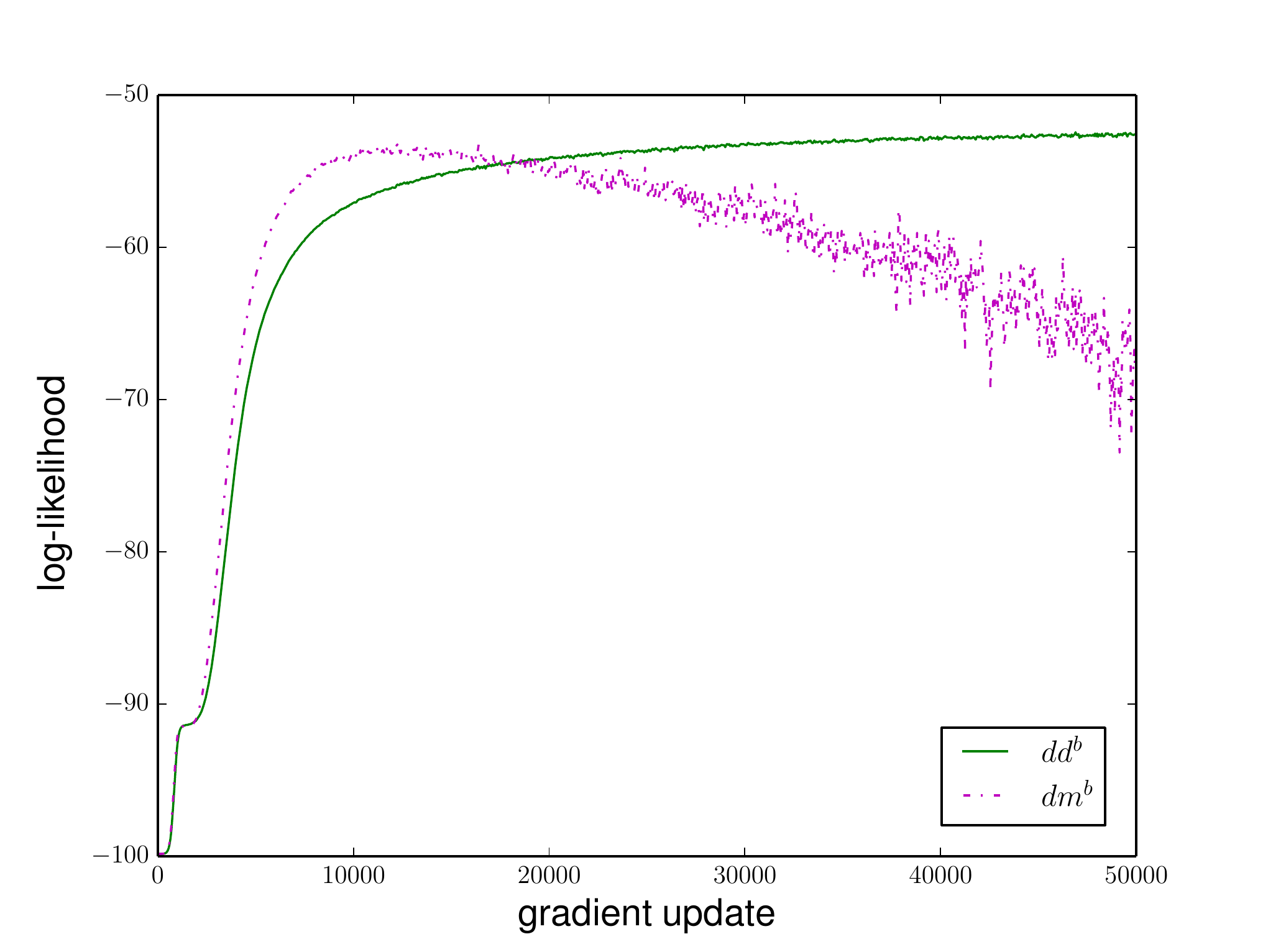}}
\caption{Evolution of the average LL during training on \emph{Bars \& Stripes} with the different centering variants, using PT$_{10}$, and a learning rate of $\eta=0.05$. (a) When an exponentially moving average with sliding factor of 0.01 was used (where the curves are almost equivalent) and (b) when no exponentially moving average was used.}
\label{fig:move_average}
\end{figure}

\subsection{Experiments with Big RBMs}

In the previous experiments we trained small models in order to be able to run  many experiments and to evaluate the LL exactly. 
We now want to show that the results we have observed on the toy problems and \emph{MNIST} with RBMs with $16$ hidden units carry over to more realistic settings. Furthermore, we want to investigate  the generalization performance of the different models.
In a first set of experiments we therefore trained the models \emph{00}, \emph{d0}, $dd_s^b$, and $aa_s^b$ with $500$ hidden units on \emph{MNIST} and \emph{Caltech}. The weight matrices were initialized with random values sampled from a
Gaussian with zero mean and a standard deviation of $0.01$ and visible and hidden biases, and offsets were initialized as described in Section~\ref{sec:initialization}. The LL was estimated using Annealed Importance Sampling (AIS), where we used the same setup as described in the analysis of~\citet{SalakhutdinovMurray-2008}. 

\begin{figure}[t]
\centering
\subfigure[PCD-$1$]{
\includegraphics[trim=0.6cm 0cm 1.43cm 0cm, clip=true,scale=0.41]{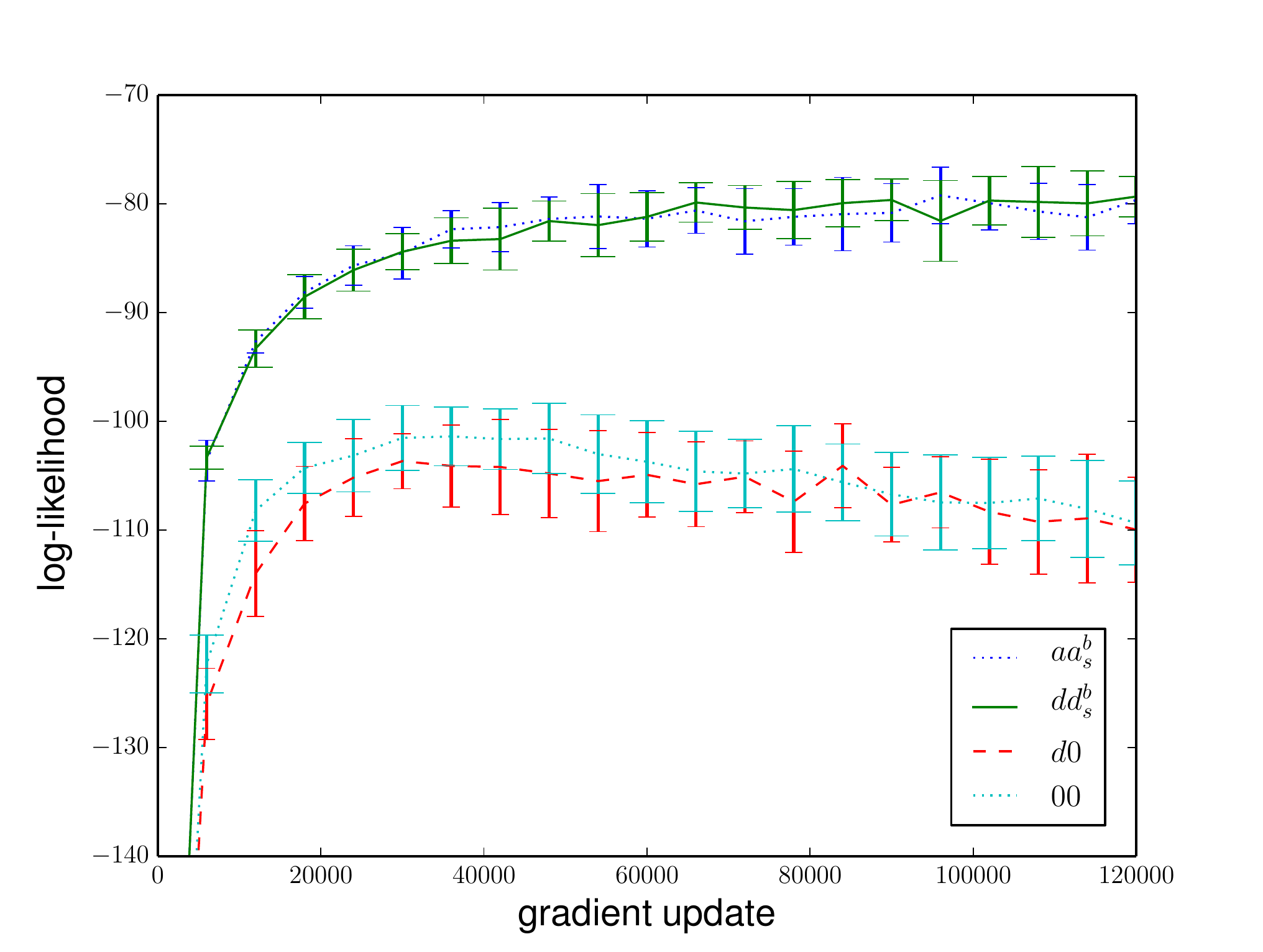}}
\subfigure[PT$_{20}$]{
\includegraphics[trim=1.6cm 0cm 1.43cm 0cm, clip=true,scale=0.41]{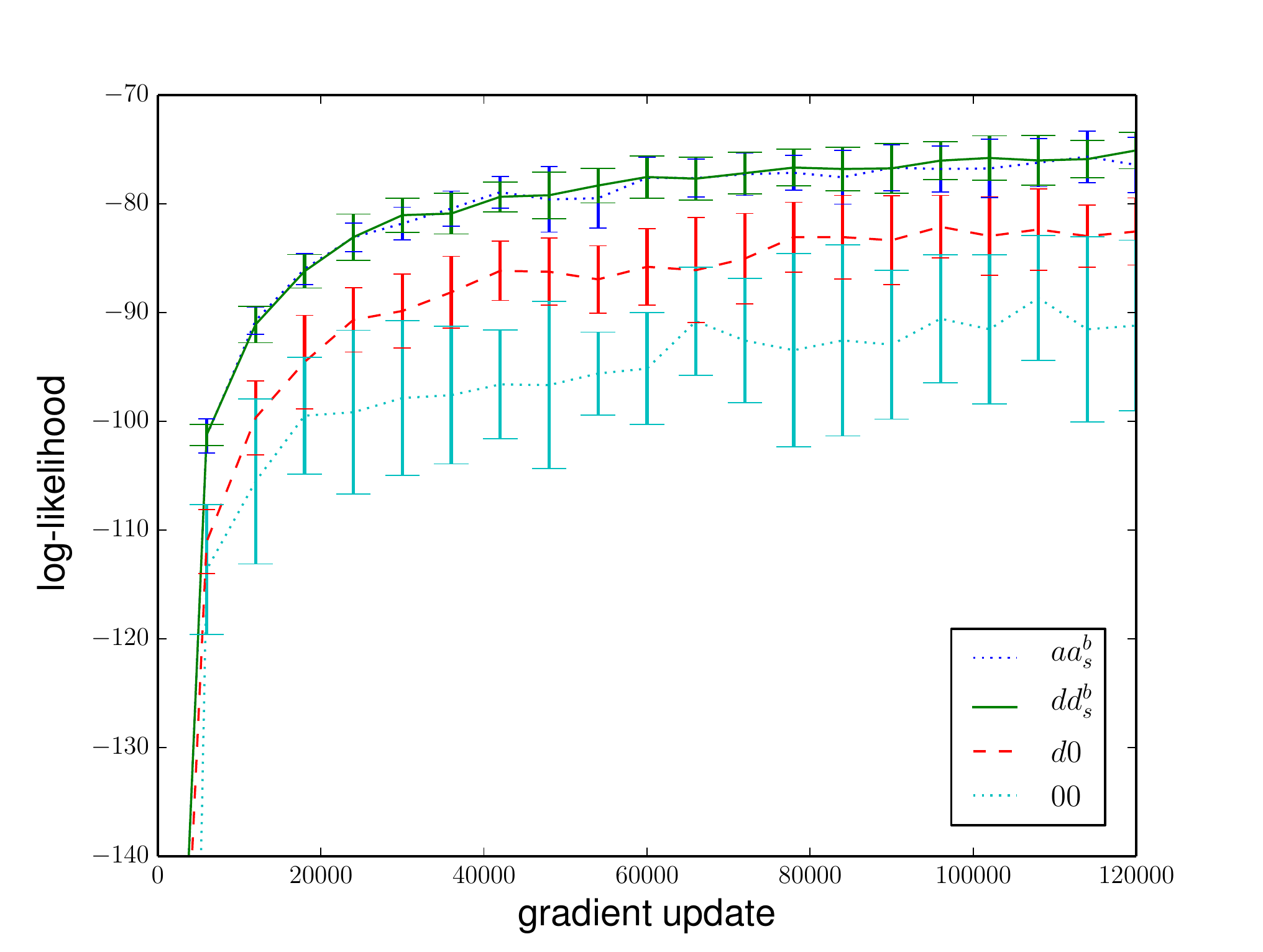}}
\caption{Evolution of the average LL on the test data of \emph{MNIST} during training for different centering variants with 500 hidden units, using a learning rate of $\eta=0.01$, and a sliding factor of 0.01. (a) When using PCD-$1$ and (b) when using $PT_{20}$ for sampling. The error bars indicate the standard deviation of the LL over the 25 trials. }
\label{fig:MNIST500}
\end{figure} 

Figure~\ref{fig:MNIST500} shows the evolution of the average LL on the test data of \emph{MNIST} over 25 trials for PCD-$1$ and PT$_{20}$ for the different centering versions. The models were trained for 200 epochs, each consisting of 600 gradient updates with a batch size of 100 and  the LL was estimated every 10th epoch using AIS. 
Both variants $dd_s^b$ and $aa_s^b$ reach significantly higher LL values than \emph{00} and \emph{d0}. 
The standard deviation over the 25 trials indicated by the error bars is smaller for $dd_s^b$ and $aa_s^b$ than for \emph{00} and \emph{d0}, especially when PT$_{20}$ is used for sampling. Furthermore, \emph{00} and \emph{d0} show divergence already after 30.000 gradient updates when PCD-$1$ is used, while no divergence can be observed for $dd_s^b$ and $aa_s^b$ after 120.000 gradient updates. The evolution of the LL on the training data is not shown, since it is almost equivalent to the evolution on the test data.
To our knowledge the best reported performance of an RBM with 500 hidden units carefully trained on \emph{MNIST} was -84~\citep{Salakhutdinov-2008, SalakhutdinovMurray-2008, TangSutskever-2011, ChoRaikoEtAl-2013}.\footnote{Note, that the preprocessing of \emph{MNIST} is usually done by treating the gray values (normalized to values in $[0,1]$) as probabilities. In different studies the probabilities are then either used directly as input, or the data set is binarized using a threshold of $0.5$ or by sampled according to the probabilities. 
This can make the LL values reported for \emph{MNIST} experiments difficult to compare across studies.}
In our experience choosing the correct training setup and using additional modifications of the update rule like a momentum term, weight decay, and  an annealing learning rate is essential to reach a value of -84 with normal binary RBMs.
However, in order to get an unbiased comparison of the different models, we did not use any of these modifications in our experiments. This explains why our performance of \emph{00} does not reach -84.  
\emph{d0} however, reaches a value of -84 when PT is used for sampling, and $dd_s^b$ and $aa_s^b$ reach even higher values around -80 with PCD-$1$ and -75 with PT$_{20}$. 

\begin{figure}[t]
\centering
\subfigure[Training data- learning rate 0.001]{
\includegraphics[trim=0.6cm 0cm 1.43cm 0cm, clip=true,scale=0.41]{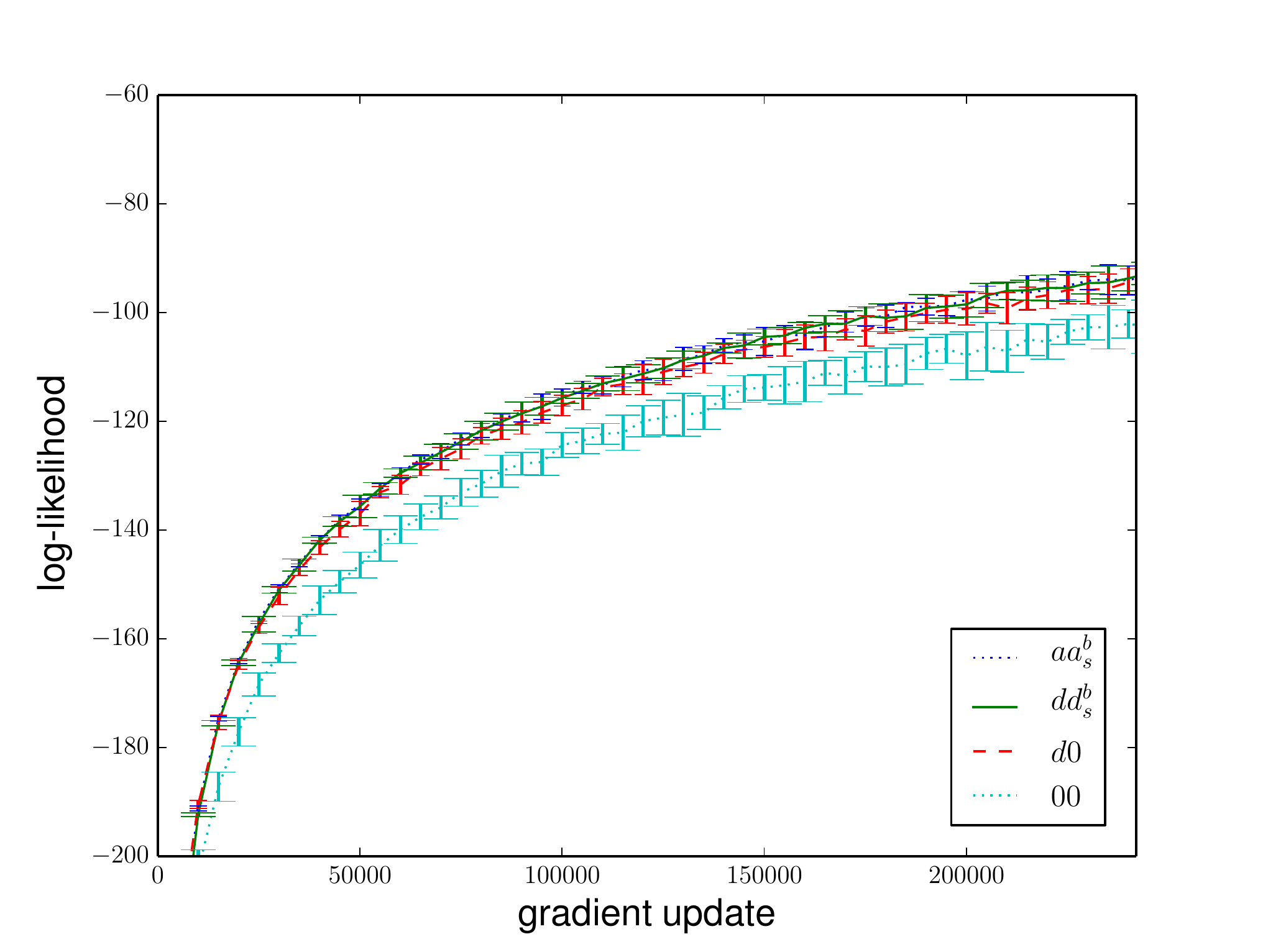}}
\subfigure[Test data - learning rate 0.001]{
\includegraphics[trim=1.6cm 0cm 1.43cm 0cm, clip=true,scale=0.41]{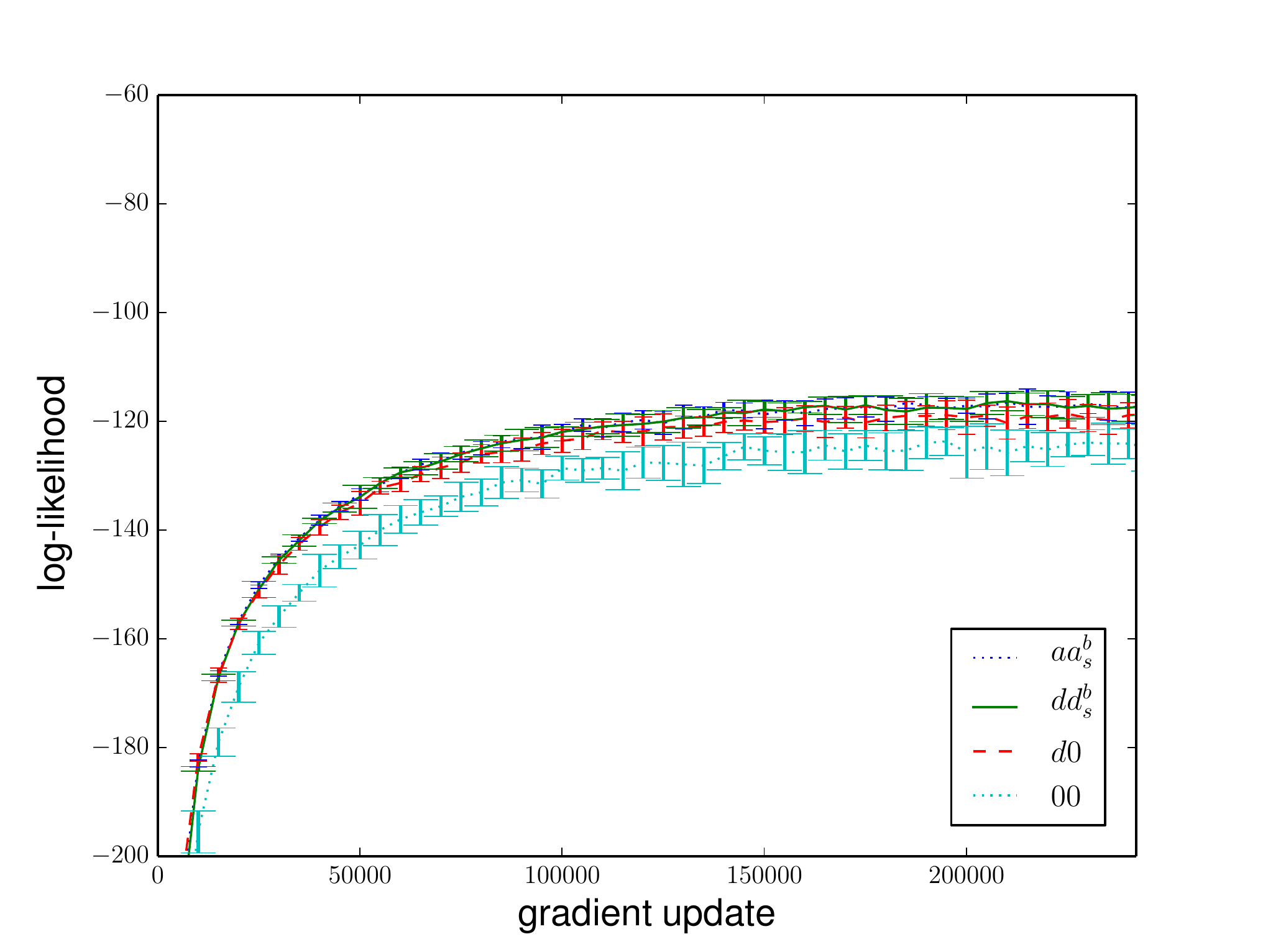}}
\subfigure[Training data - learning rate 0.01]{
\includegraphics[trim=0.6cm 0cm 1.43cm 0cm, clip=true,scale=0.41]{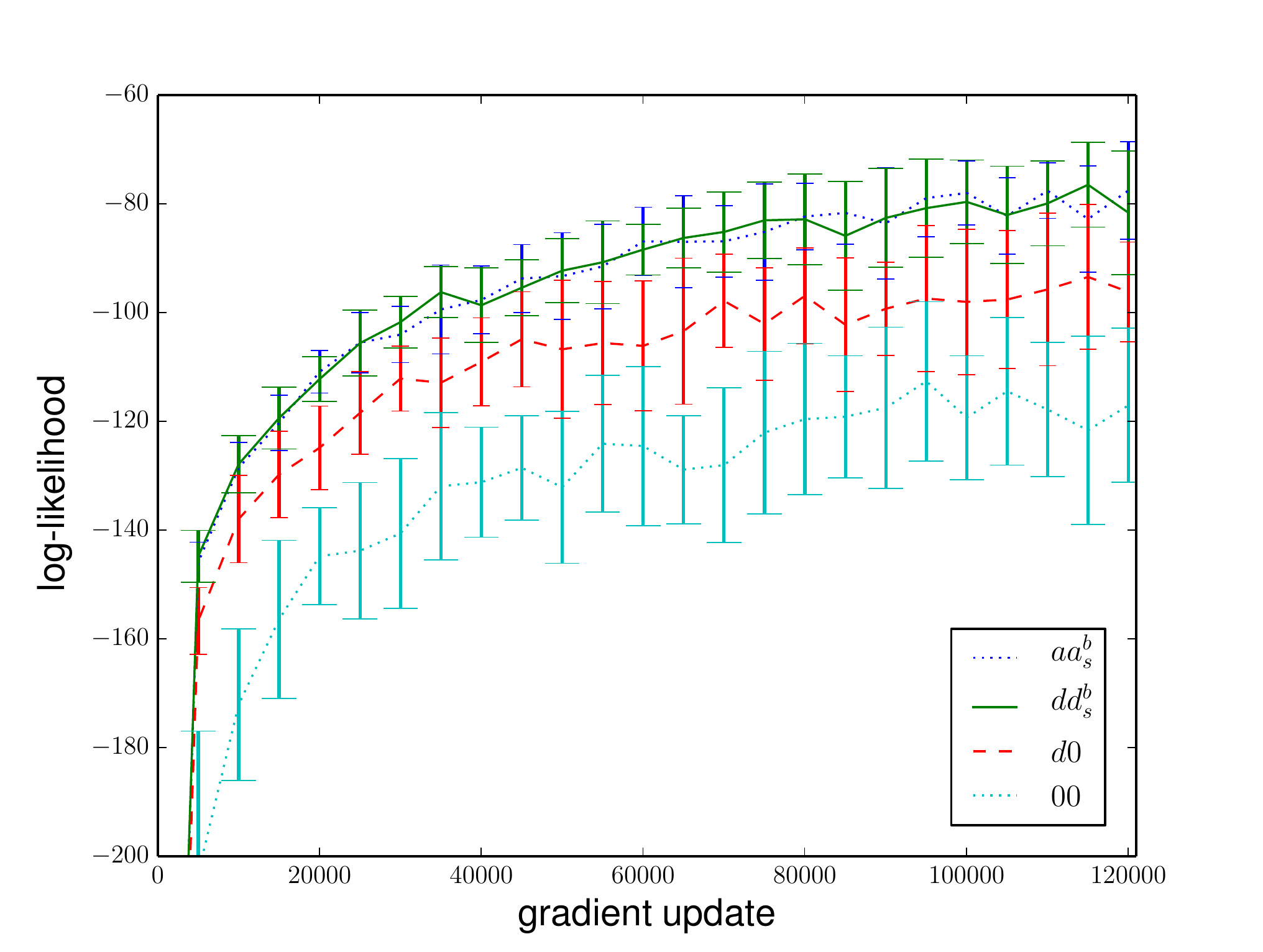}}
\subfigure[Test data - learning rate 0.01]{
\includegraphics[trim=1.6cm 0cm 1.43cm 0cm, clip=true,scale=0.41]{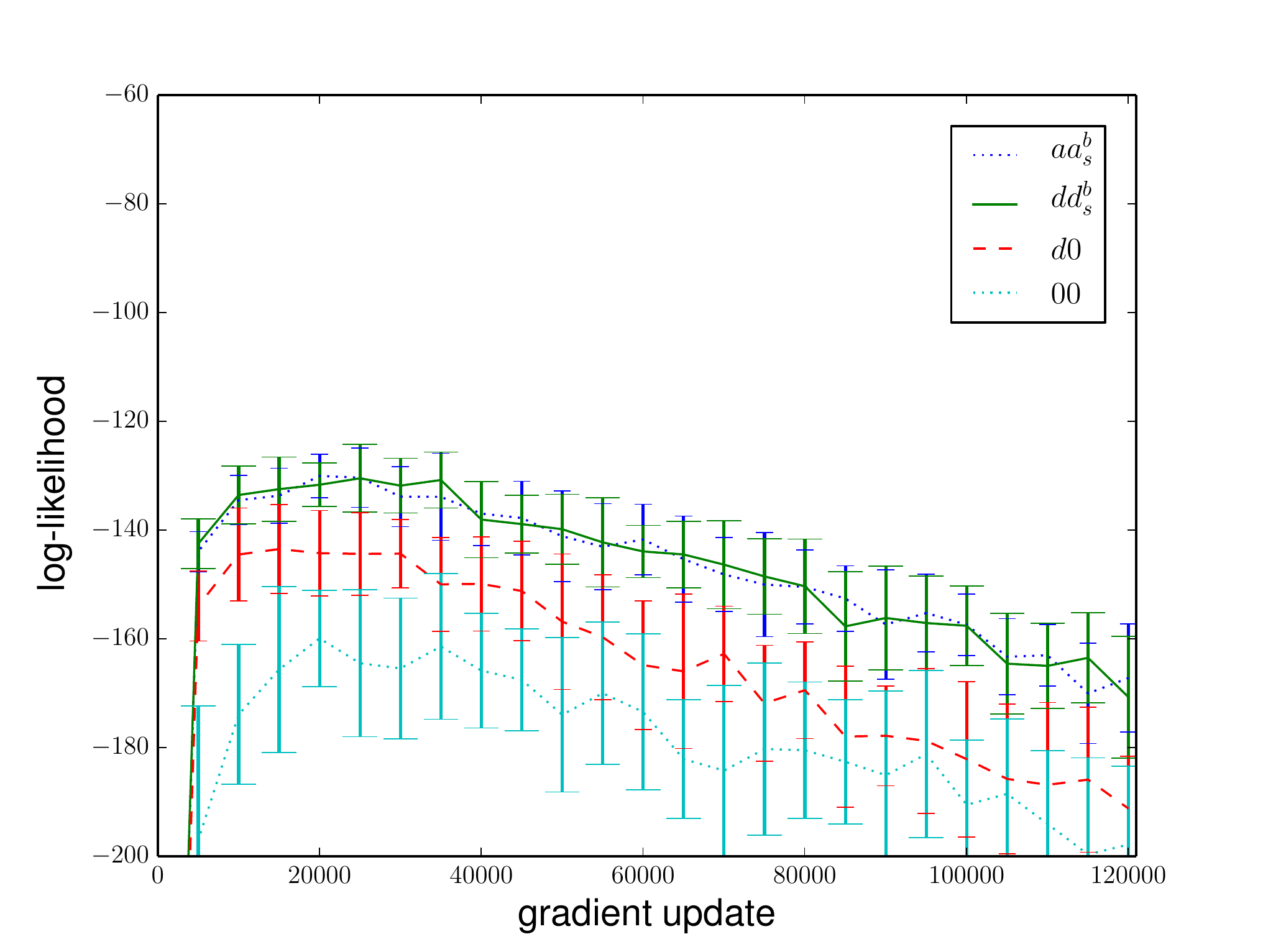}}
\caption{Evolution of the average LL on \emph{Caltech} dataset with the different centering variants with 500 hidden units. The results on training and test data for a learning rate of $\eta=0.001$ are shown in sub-figures (a) and (b), respectively and for a learning rate of $\eta=0.01$  in sub-figures (c) and (d), respectively. In both cases a sliding factor of 0.01 and PCD-$1$ was used. The error bars indicate the standard deviation of the LL over the 25 trials.}
\label{fig:CALTECH500}
\end{figure} 

For the \emph{Caltech} dataset, Figure~\ref{fig:CALTECH500} shows the evolution of the average LL on training and test data over 25 trials for the different centering versions using PCD-$1$ with  a batch size of 100 and either a learning rate of $0.001$ or $0.01$. The LL was estimated every 5000th gradient update using AIS.
The results show that $dd_s^b$, $aa_s^b$ and \emph{d0} reach higher LL values than \emph{00} for both learning rates and on training and test data. While $dd_s^b$ and $aa_s^b$ perform only slightly better than \emph{d0} when a small learning rate is used, the difference becomes more prominent for a big learning rate. Figure~\ref{fig:CALTECH500}(c) and (d) show that all models over fit to the training data. Nevertheless, $dd_s^b$ and $aa_s^b$ reach higher LL values on the test data and thus lead to a better generalization.

To emphasize that the divergence problems induced by using the model means as offsets also appear for big models when no shifting average is used, we trained RBMs with 500 hidden units on \emph{MNIST} and \emph{Caltech} using PT$_{20}$ with a learning rate of $0.001$ and a batch size of 100. 
In addition, we trained $aa^b$ and $dd^b$ on \emph{Caltech} using full batch learning and a learning rate of $0.01$.
Figure~\ref{fig:500Divergence} shows that $aa^b$
diverges, while $dd^b$ and the corresponding centering versions using a moving average, $aa_s^b$ and $dd_s^b$, show no divergence. The divergence for  $aa^b$ even occurs in full batch training as shown in Figure~\ref{fig:500Divergence}(b).

\begin{figure}[t]
\centering
\subfigure[MNIST LL on test data]{
\includegraphics[trim=0.4cm 0cm 1.43cm 0cm, clip=true,scale=0.405]{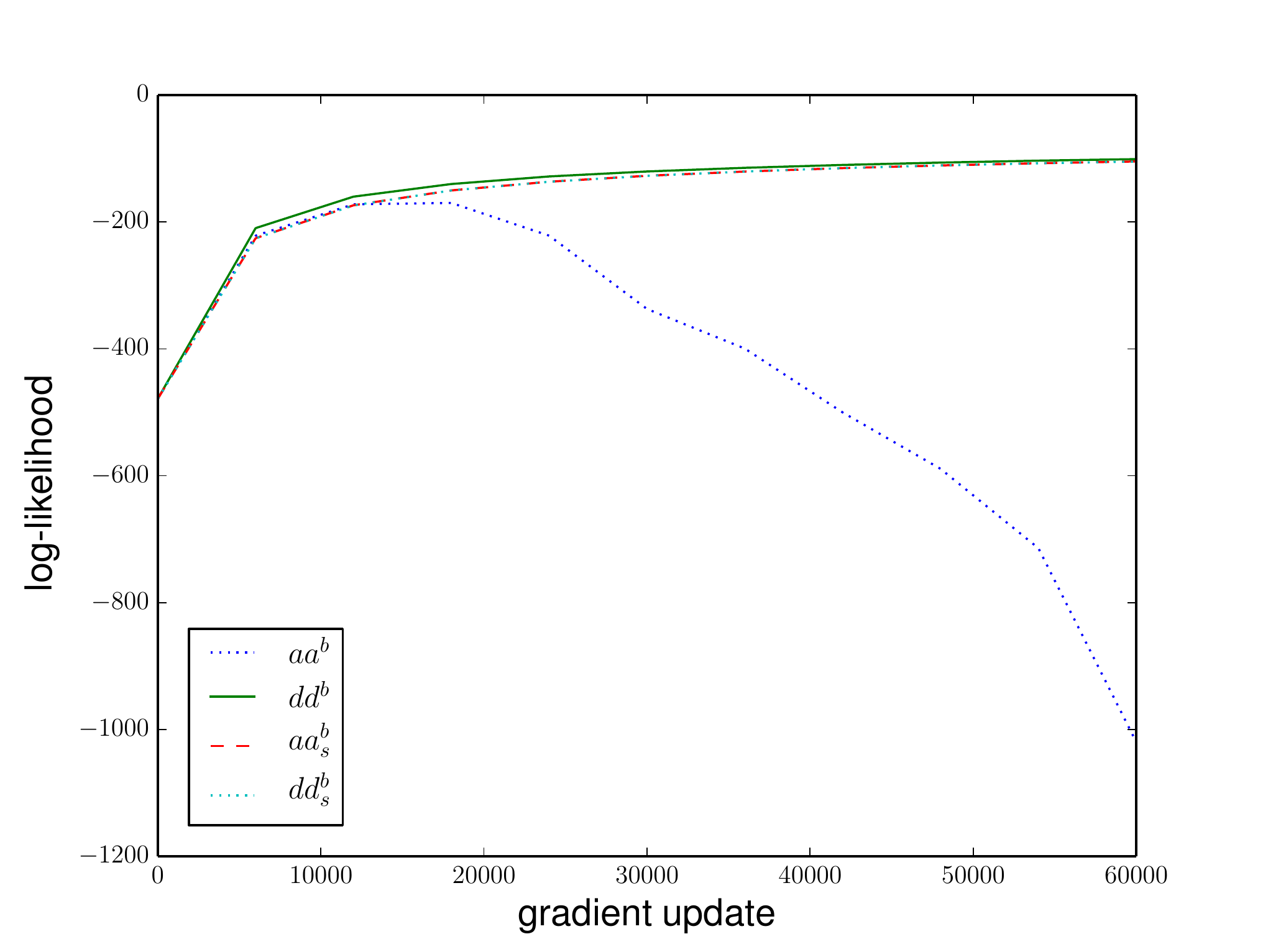}}
\subfigure[Caltech LL on test data]{
\includegraphics[trim=1.6cm 0cm 1.43cm 0cm, clip=true,scale=0.405]{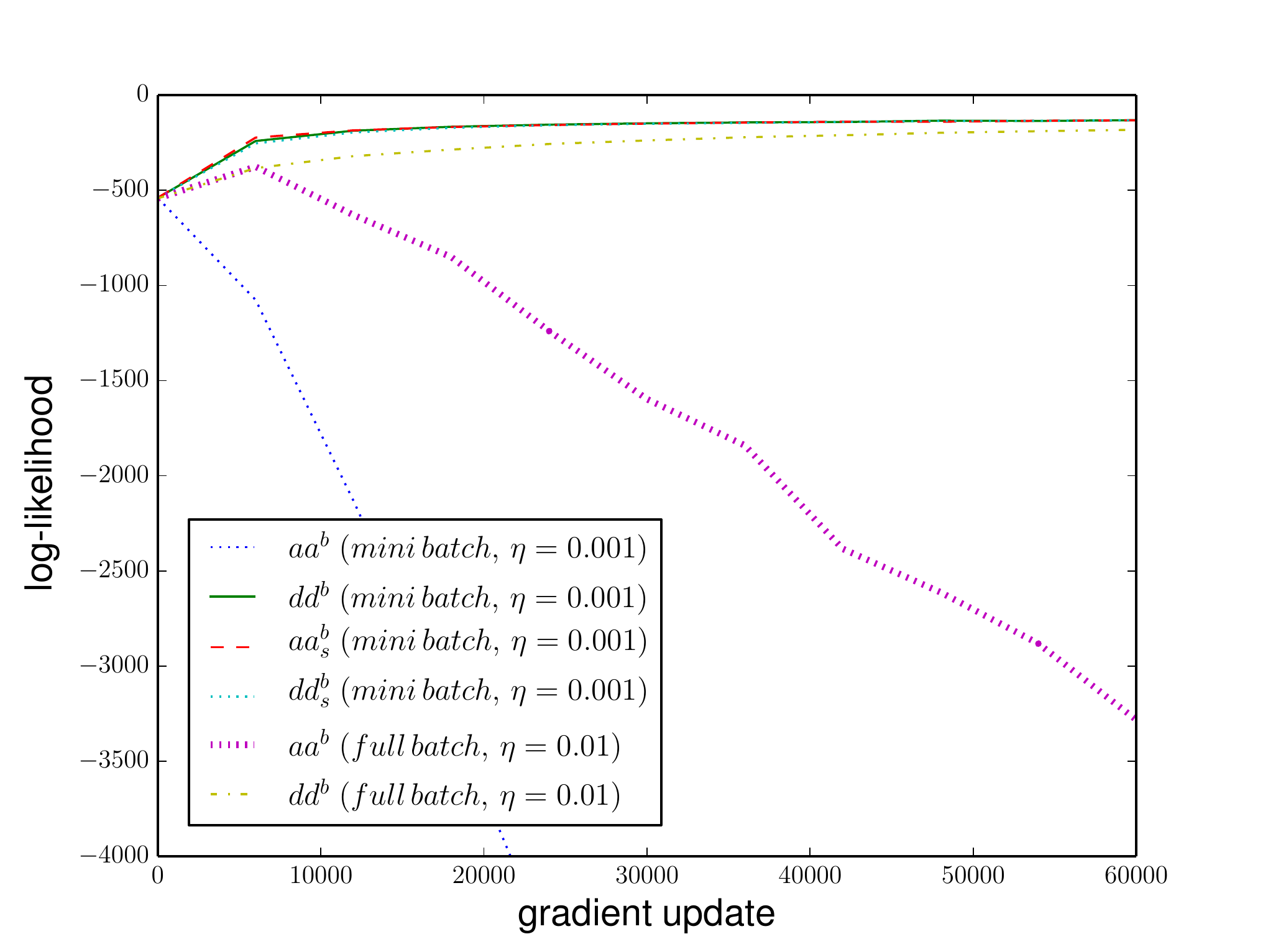}}
\caption{Evolution of the LL of exemplary trials for the different centering variants $aa^b$, $dd^b$, $aa_s^b$ and $dd_s^b$ on (a) \emph{MNIST} and  (b) \emph{Caltech} during training using PT$_{20}$ with a batch size of 100, 500 hidden units and a learning rate of $0.001$. For \emph{Caltech} $aa^b$ and $dd^b$ were also trained in full batch mode with PT$_{20}$ and a learning rate of $0.01$.}
\label{fig:500Divergence}
\end{figure}

In a second set of experiments we extended our analysis to the eight datasets used by \citet{LarochelleBengioEtAl-2010}. The four different models \emph{00}, \emph{d0}, $dd_s^b$, and $aa_s^b$ were trained with the same setup as before using PCD-$1$, a learning rate of $0.01$, a batch size of 100 and a total number of $5000$ epochs. All experiments were repeated 25 times and we trained either RBMs with 16 hidden units and calculated the LL exactly or RBMs with 200 hidden units using AIS for estimating the LL. Additionally we trained RBMs with 200 hidden  with
a smaller learning rate of $0.001$ for $30000$ epochs.
Due to the long training time these experiments were repeated only 10 times.
The maximum average LL for the test data is shown in Table~\ref{tab:testLL_Various}. On seven out of eight datasets $dd_s^b$ or $aa_s^b$ reached the best result independent of whether 16 or 200 hidden units or a learning rate of $0.01$ or $0.001$ were used. Whenever $aa_s^b$ reached the highest value it was not significantly different to $dd_s^b$.
Note, that for training  RBMs with 16 hidden units $d0$ reached comparable  results to $dd_s^b$ on some datasets. Only on the \emph{RCV1} dataset, $00$ lead to better LL values than the centered RBMs for both 16 and 200 hidden units. It seems that the convergence rate on the \emph{RCV1}, \emph{OCR} and \emph{Web} dataset is rather slow for all models since the difference between the highest and the final LL values are rather small. This can also be observed on the training data shown in Table~\ref{tab:trainLL_Various}.
The \emph{DNA} dataset and the \emph{NIPS} dataset over fit to the training data as indicated by the fact that the divergence is only observed for the test data. In contrast, on the remaining three datasets \emph{Adult}, \emph{CONNECT-4} and \emph{Mushroom} the divergence can be observed on training and test data.
Finally note, that all eight datasets contain more zeros than ones in the current representation  as mentioned in Section~\ref{sec:Benchmark_problems}. Thus, the performance of the normal RBM would be even worse on the flipped datasets while for the centering variants it would stay the same.

Consistent with the experiments on small models, the results from nine of the ten real world datasets clearly support the superiority of centered over normal RBMs and show that centering visible and hidden units in RBMs is important for yielding good models. 

One explanation why centering works has been provided by \citet{MontavonMueller-2012}, who found that centering leads to an initially better conditioned optimization problem. 

Furthermore, \citet{ChoRaikoEtAl-2011a} have shown that when the enhanced gradient is used for 
training the update directions for the weights are less correlated than when the standard gradient is used, which allows to learn more meaningful features. 

From our analysis in Section~\ref{sec:CenteredRBMs} we know that centered RBMs and normal RBMs belong to the same model class and therefore the reason why centered RBMs outperform normal RBMs can indeed only be due to the optimization procedure. 
Furthermore, one has to keep in mind that in centered RBMs the variables mean values are explicitly stored in the corresponding offset parameters, or if the centered gradient is used for training normal RBMs the mean values are transferred to the corresponding bias parameters. 
This allows the weights to model second and higher order statistics right from the start, which is in contrast to normal binary RBMs where weights usually capture parts of the mean values. 
To support this statement empirically, we calculated the average weight and bias norms during training of the RBMs with 500 hidden units on \emph{MNIST} using the standard and the centered gradient. 
The results are shown in Figure~\ref{fig:MNISTNORMS}, where it can be seen that the row and column norms (see Figure~\ref{subfig:MNISTNORMS_A} and \ref{subfig:MNISTNORMS_B}) of the weight matrix for $dd_s^b$, $aa_s^b$, and \emph{d0} are consistently smaller than for \emph{00}. 
At the same time the bias values (see Figure~\ref{subfig:MNISTNORMS_C} and \ref{subfig:MNISTNORMS_D}) for $dd_s^b$, $aa_s^b$, and \emph{d0} are much bigger than for \emph{00}, indicating that the weight vectors of \emph{00} model information that could potentially be modeled by the bias values.
Interestingly, the curves for all parameters of $dd_s^b$ and $aa_s^b$ show the same logarithmic shape, while for \emph{d0} and \emph{00} the visible bias norm does not change significantly. 
It seems that the bias values did not adapt properly during training. 
Comparing, \emph{d0} with $dd_s^b$ and $aa_s^b$, the weight norms are slightly bigger and the visible bias is much smaller for \emph{d0}, indicating that it is not sufficient to center only the visible variables and that visible and hidden bias influence each other. 
This dependence of the hidden mean and visible bias can also be seen from Equation~\eqref{eqn:transform_b} where the transformation of the visible bias depends on the offset of the hidden variables.

\begin{landscape}
\begin{table*}
\setlength{\tabcolsep}{3pt}
\begin{center}
\begin{small}
\begin{sc}
\hspace*{-2.1cm}
\begin{tabular}{l l l l l}
\hline
\abovespace\belowspace
Dataset & \multicolumn{1}{c}{$aa_s^b$} & \multicolumn{1}{c}{$dd_s^b$} & \multicolumn{1}{c}{\emph{d0}} & \multicolumn{1}{c}{\emph{00}} \\
\hline
\multicolumn{5}{l}{16 hidden units - learning rate 0.01}
\abovespace\belowspace \\
Adult 	       & -18.09 $\pm$0.62 \,(-18.18) &  \textbf{-17.70} $\pm$0.25 \,(-17.74) &  -17.90 $\pm$0.21 \,(-17.97) &  -17.94 $\pm$0.22 \,(-17.98) \\ 
CONNECT-4  	 & -20.07 $\pm$0.19 \,(-20.52) &  \textbf{-19.89} $\pm$0.21 \,(-19.90) &  -20.14 $\pm$0.24 \,(-20.16) &  -20.59 $\pm$0.29 \,(-20.74) \\ 
DNA        	 & \textbf{-96.97} $\pm$0.05 \,(-97.26) &  \textbf{-96.97} $\pm$0.04 \,(-97.25) &  -97.01 $\pm$0.04 \,(-97.21) &  -97.03 $\pm$0.05 \,(-97.19) \\ 
Mushroom  	 & -16.83 $\pm$0.57 \,(-17.11) &  \textbf{-16.53} $\pm$0.64 \,(-16.53) &  \textbf{-16.76} $\pm$0.56 \,(-17.11) &  -17.05 $\pm$0.59 \,(-17.44) \\ 
NIPS      	 & \textbf{-276.37} $\pm$0.16 \,(-279.13) &  \textbf{-276.38} $\pm$0.16 \,(-279.11) &  -276.53 $\pm$0.21 \,(-278.97) &  -278.04 $\pm$0.27 \,(-279.95) \\ 
OCR       	 & \textbf{-45.81} $\pm$0.13 \,(-45.81) &  \textbf{-45.82} $\pm$0.12 \,(-45.83) &  \textbf{-45.84} $\pm$0.12 \,(-45.86) &  -45.97 $\pm$0.17 \,(-45.99) \\ 
RCV1       	 & -49.57 $\pm$0.04 \,(-49.57) &  -49.58 $\pm$0.05 \,(-49.58) &  -49.59 $\pm$0.05 \,(-49.59) &  \textbf{-49.53} $\pm$0.04 \,(-49.53) \\ 
Web       	 & \textbf{-29.99} $\pm$0.05 \,(-29.99) &  \textbf{-29.98} $\pm$0.05 \,(-29.98) &  \textbf{-30.20} $\pm$0.70 \,(-30.86) &  -30.75 $\pm$0.11 \,(-38.82) \\ 
\hline
\multicolumn{5}{l}{200 hidden units - learning rate 0.01}
\abovespace\belowspace \\
Adult 	       & -15.98 $\pm$0.37 \,(-17.80) &  \textbf{-15.55} $\pm$0.25 \,(-16.22) &  -16.65 $\pm$0.61 \,(-18.96) &  -16.91 $\pm$0.78 \,(-19.24) \\ 
CONNECT-4  	 & -14.85 $\pm$0.24 \,(-23.31) &  \textbf{-14.70} $\pm$0.20 \,(-17.71) &  -16.14 $\pm$0.50 \,(-23.14) &  -17.88 $\pm$0.78 \,(-30.41) \\ 
DNA        	 & -90.15 $\pm$0.09 \,(-95.17) &  \textbf{-90.12} $\pm$0.10 \,(-94.13) &  -90.78 $\pm$0.10 \,(-95.93) &  -91.13 $\pm$0.12 \,(-97.62) \\ 
Mushroom  	 & \textbf{-15.45} $\pm$1.35 \,(-20.11) &  \textbf{-15.61} $\pm$1.23 \,(-19.93) &  -16.28 $\pm$0.88 \,(-21.59) &  -16.42 $\pm$1.61 \,(-21.90) \\ 
NIPS      	 & \textbf{-270.81} $\pm$0.05 \,(-291.82) &  \textbf{-270.81} $\pm$0.07 \,(-291.77) &  -271.38 $\pm$0.28 \,(-294.28) &  -272.88 $\pm$0.44 \,(-290.02) \\ 
OCR       	 & \textbf{-29.75} $\pm$0.50 \,(-30.18) &  \textbf{-29.53} $\pm$0.52 \,(-29.97) &  -30.66 $\pm$0.65 \,(-30.66) &  -30.08 $\pm$0.52 \,(-30.25) \\ 
RCV1       	 & -47.13 $\pm$0.12 \,(-47.13) &  -47.14 $\pm$0.12 \,(-47.16) &  -47.14 $\pm$0.10 \,(-47.19) &  \textbf{-46.70} $\pm$0.11 \,(-46.72) \\ 
Web       	 & \textbf{-28.27} $\pm$0.19 \,(-28.47) &  \textbf{-28.27} $\pm$0.19 \,(-28.58) &  -28.50 $\pm$0.51 \,(-29.16) &  -28.35 $\pm$0.14 \,(-28.71) \\ 
\hline
\multicolumn{5}{l}{200 hidden units - learning rate 0.001}
\abovespace\belowspace \\
Adult 	       & -15.51 $\pm$0.17 \,(-17.14) &  \underline{-14.90} $\pm$0.07 \,(-14.90) &  -15.27 $\pm$0.22 \,(-15.48) &  -15.16 $\pm$0.21 \,(-15.39) \\ 
CONNECT-4  	 & -13.73 $\pm$0.17 \,(-15.95) &  \underline{-13.13} $\pm$0.10 \,(-13.25) &  -13.86 $\pm$0.29 \,(-14.79) &  -14.47 $\pm$0.38 \,(-16.04) \\ 
DNA        	 & -90.17 $\pm$0.07 \,(-90.17) &  \underline{-90.15} $\pm$0.07 \,(-90.15) &  -90.45 $\pm$0.09 \,(-90.45) &  -90.73 $\pm$0.13 \,(-90.73) \\ 
Mushroom  	 & -13.44 $\pm$0.75 \,(-13.96) &  \underline{-13.17} $\pm$1.15 \,(-13.94) &  -13.47 $\pm$0.73 \,(-14.51) &  -13.68 $\pm$0.77 \,(-14.77) \\ 
NIPS      	 & -270.71 $\pm$0.03 \,(-313.30) &  \underline{-270.70} $\pm$0.03 \,(-313.16) &  -270.84 $\pm$0.04 \,(-315.96) &  -271.91 $\pm$0.12 \,(-309.03) \\ 
OCR       	 & -27.92 $\pm$0.34 \,(-27.92) &  \underline{-27.70} $\pm$0.15 \,(-27.75) &  -28.24 $\pm$0.21 \,(-28.34) &  -28.04 $\pm$0.39 \,(-28.29) \\ 
RCV1       	 & -47.09 $\pm$0.05 \,(-47.09) &  -47.10 $\pm$0.03 \,(-47.11) &  -46.93 $\pm$0.07 \,(-46.93) &  \underline{-46.56} $\pm$0.07 \,(-46.56) \\ 
Web       	 & -28.13 $\pm$0.02 \,(-28.17) &  \underline{-28.11} $\pm$0.06 \,(-28.11) &  -28.16 $\pm$0.05 \,(-28.20) &  -28.13 $\pm$0.05 \,(-28.15) \\ 
\hline
\end{tabular}
\end{sc}
\end{small}
\end{center}
\caption {Maximum average LL on test data on various datasets using PCD-$1$, with (top) 16 
hidden units and a learning rate of $0.01$, (middle) 200 
hidden units and a learning rate of $0.01$ and (bottom) 200 
hidden units and learning rate of $0.001$ (since 10 trials are are not enough to perform a statistical significance test we simply underlined the best result).
}
\label{tab:testLL_Various}
\end{table*}
\end{landscape}

\begin{landscape}
\begin{table*}
\setlength{\tabcolsep}{3pt}
\begin{center}
\begin{small}
\begin{sc}
\hspace*{-2.1cm}
\begin{tabular}{l l l l l}
\hline
\abovespace
Dataset & \multicolumn{1}{c}{$aa_s^b$} & \multicolumn{1}{c}{$dd_s^b$} & \multicolumn{1}{c}{\emph{d0}} & \multicolumn{1}{c}{\emph{00}} \\
\hline
\multicolumn{5}{l}{16 hidden units - learning rate 0.01}
\abovespace\belowspace \\
Adult 	       & -17.93 $\pm$0.61 \,(-18.01) &  \textbf{-17.54} $\pm$0.25 \,(-17.58) &  -17.73 $\pm$0.20 \,(-17.80) &  -17.78 $\pm$0.23 \,(-17.81) \\ 
CONNECT-4  	 & -19.94 $\pm$0.19 \,(-20.38) &  \textbf{-19.75} $\pm$0.25 \,(-19.77) &  -20.01 $\pm$0.24 \,(-20.03) &  -20.46 $\pm$0.28 \,(-20.61) \\ 
DNA        	 & \textbf{-94.33} $\pm$0.05 \,(-94.33) &  \textbf{-94.33} $\pm$0.06 \,(-94.33) &  -94.44 $\pm$0.04 \,(-94.44) &  -94.42 $\pm$0.05 \,(-94.42) \\ 
Mushroom  	 & -16.55 $\pm$0.56 \,(-16.82) &  \textbf{-16.25} $\pm$0.64 \,(-16.25) &  \textbf{-16.46} $\pm$0.56 \,(-16.80) &  -16.75 $\pm$0.60 \,(-17.15) \\ 
NIPS      	 & \textbf{-255.03} $\pm$0.25 \,(-255.05) &  \textbf{-255.02} $\pm$0.23 \,(-255.04) &  -255.88 $\pm$0.28 \,(-255.88) &  -258.57 $\pm$0.29 \,(-258.57) \\ 
OCR       	 & \textbf{-45.88} $\pm$0.12 \,(-45.88) &  \textbf{-45.90} $\pm$0.11 \,(-45.90) &  \textbf{-45.90} $\pm$0.12 \,(-45.93) &  -46.03 $\pm$0.16 \,(-46.05) \\ 
RCV1       	 & -49.42 $\pm$0.04 \,(-49.42) &  -49.43 $\pm$0.05 \,(-49.43) &  -49.44 $\pm$0.05 \,(-49.44) &  \textbf{-49.38} $\pm$0.04 \,(-49.38) \\ 
Web       	 & \textbf{-29.79} $\pm$0.05 \,(-29.79) &  \textbf{-29.78} $\pm$0.05 \,(-29.78) &  \textbf{-29.97} $\pm$0.70 \,(-30.62) &  -30.64 $\pm$0.11 \,(-38.61) \\ 
\hline
\multicolumn{5}{l}{200 hidden units - learning rate 0.01}
\abovespace\belowspace \\
Adult 	       & -15.19 $\pm$0.62 \,(-16.29) &  \textbf{-14.51} $\pm$0.43 \,(-14.80) &  -15.83 $\pm$0.61 \,(-17.17) &  -16.00 $\pm$0.79 \,(-17.51) \\ 
CONNECT-4  	 & -14.39 $\pm$0.24 \,(-22.50) &  \textbf{-14.22} $\pm$0.20 \,(-17.02) &  -15.74 $\pm$0.49 \,(-22.32) &  -17.58 $\pm$0.76 \,(-29.55) \\ 
DNA        	 & \textbf{-60.19} $\pm$1.20 \,(-60.69) &  \textbf{-59.75} $\pm$1.45 \,(-59.75) &  \textbf{-60.31} $\pm$1.64 \,(-60.31) &  -61.11 $\pm$1.54 \,(-61.71) \\ 
Mushroom  	 & \textbf{-14.90} $\pm$1.34 \,(-19.14) &  \textbf{-15.06} $\pm$2.09 \,(-19.00) &  \textbf{-15.70} $\pm$2.88 \,(-20.52) &  -15.87 $\pm$1.63 \,(-20.87) \\ 
NIPS      	 & -180.48 $\pm$0.24 \,(-180.48) &  -180.51 $\pm$0.31 \,(-180.51) &  \textbf{-180.29} $\pm$0.54 \,(-180.29) &  -184.84 $\pm$0.62 \,(-184.84) \\ 
OCR       	 & \textbf{-28.97} $\pm$0.49 \,(-29.39) &  \textbf{-28.76} $\pm$0.52 \,(-29.20) &  -29.84 $\pm$0.63 \,(-29.84) &  -29.34 $\pm$0.51 \,(-29.49) \\ 
RCV1       	 & -45.83 $\pm$0.12 \,(-45.83) &  -45.85 $\pm$0.11 \,(-45.87) &  -45.92 $\pm$0.18 \,(-45.97) &  \textbf{-45.60} $\pm$0.11 \,(-45.61) \\ 
Web       	 & \textbf{-26.22} $\pm$0.67 \,(-26.22) &  -26.28 $\pm$0.20 \,(-26.32) &  -26.57 $\pm$0.66 \,(-26.75) &  -26.34 $\pm$0.38 \,(-26.39) \\ 
\hline
\multicolumn{5}{l}{200 hidden units - learning rate 0.001}
\abovespace\belowspace \\
Adult 	       & -14.61 $\pm$0.52 \,(-15.96) &  \underline{-13.85} $\pm$0.06 \,(-13.85) &  -14.07 $\pm$0.22 \,(-14.17) &  -14.03 $\pm$0.21 \,(-14.11) \\ 
CONNECT-4  	 & -13.15 $\pm$0.17 \,(-15.20) &  \underline{-12.54} $\pm$0.11 \,(-12.62) &  -13.28 $\pm$0.29 \,(-14.10) &  -13.98 $\pm$0.38 \,(-15.39) \\ 
DNA        	 & -70.22 $\pm$0.08 \,(-70.22) &  -70.24 $\pm$0.08 \,(-70.24) &  -69.66 $\pm$0.06 \,(-69.66) &  \underline{-69.40} $\pm$0.08 \,(-69.40) \\ 
Mushroom  	 & -12.76 $\pm$0.75 \,(-13.21) &  \underline{-12.43} $\pm$1.18 \,(-13.21) &  -12.69 $\pm$0.72 \,(-13.68) &  -13.05 $\pm$0.76 \,(-13.92) \\ 
NIPS      	 & -148.14 $\pm$0.30 \,(-148.14) &  -148.30 $\pm$0.31 \,(-148.30) &  \underline{-147.33} $\pm$0.30 \,(-147.33) &  -150.91 $\pm$0.30 \,(-150.91) \\ 
OCR       	 & -27.06 $\pm$0.33 \,(-27.06) &  \underline{-26.90} $\pm$0.14 \,(-26.95) &  -27.41 $\pm$0.20 \,(-27.52) &  -27.28 $\pm$0.39 \,(-27.53) \\ 
RCV1       	 & -45.83 $\pm$0.04 \,(-45.83) &  -45.85 $\pm$0.05 \,(-45.85) &  -45.73 $\pm$0.07 \,(-45.73) &  \underline{-45.47} $\pm$0.06 \,(-45.47) \\ 
Web       	 & -26.44 $\pm$0.06 \,(-26.45) &  -26.36 $\pm$0.07 \,(-26.36) &  \underline{-26.32} $\pm$0.05 \,(-26.33) &  -26.35 $\pm$0.10 \,(-26.35) \\ 
\hline
\end{tabular}
\end{sc}
\end{small}
\end{center}
\caption {Maximum average LL on training data on various datasets using PCD-$1$, with (top) 16 
hidden units and a learning rate of $0.01$, (middle) 200 
hidden units and a learning rate of $0.01$ and (bottom) 200 
hidden units and a learning rate of $0.001$.(since 10 trials are are not enough to perform a statistical significance test we simply underlined the best result).}
\label{tab:trainLL_Various}
\end{table*}
\end{landscape}

\begin{figure}[t]
\centering
\subfigure[Norm of the weight matrix colums]{\label{subfig:MNISTNORMS_A}
\includegraphics[trim=0.6cm 0cm 1.75cm 0cm, clip=true,scale=0.417]{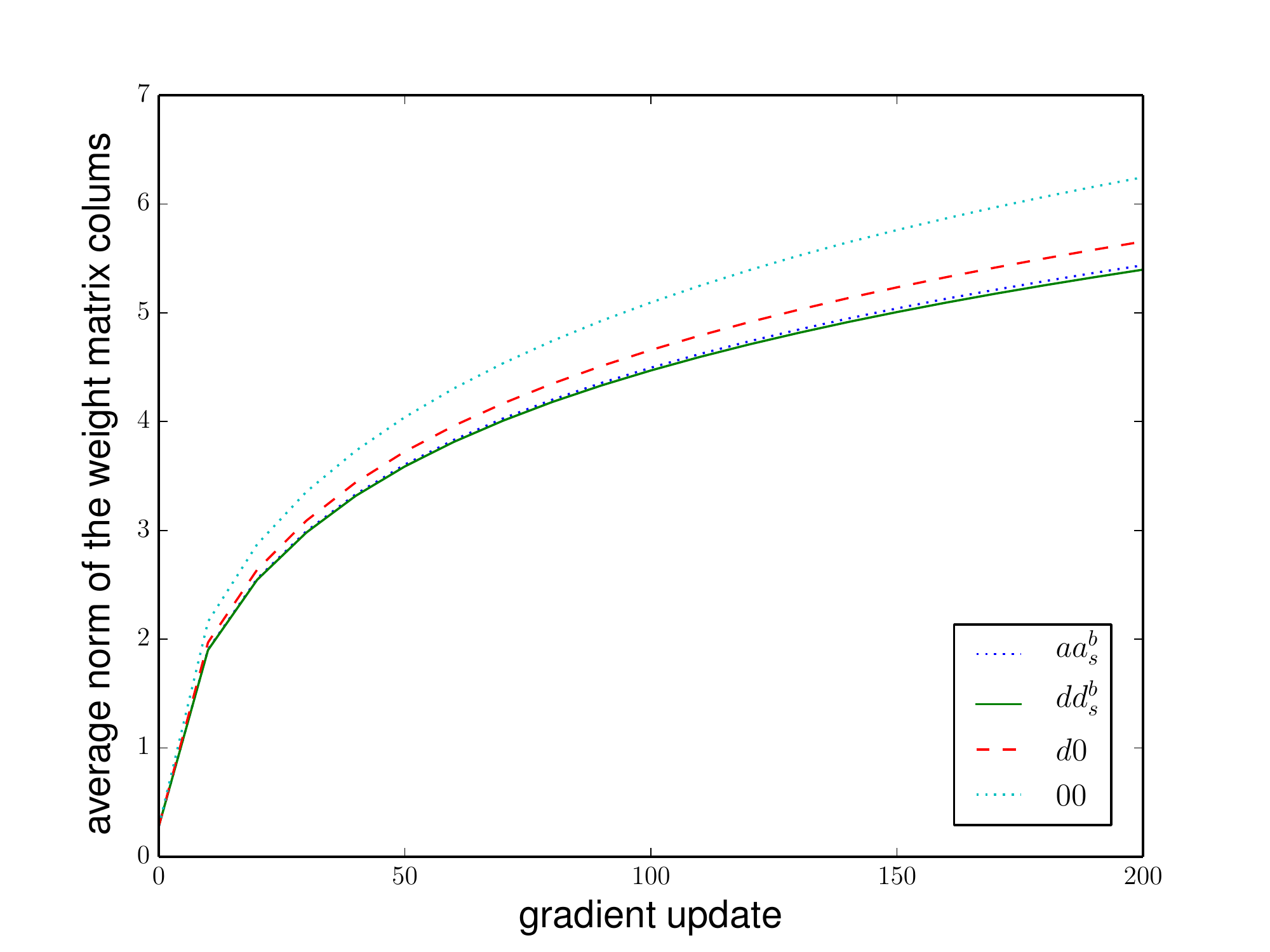}}
\subfigure[Norm of the weight matrix rows]{\label{subfig:MNISTNORMS_B}
\includegraphics[trim=0.9cm 0cm 1.75cm 0cm, clip=true,scale=0.417]{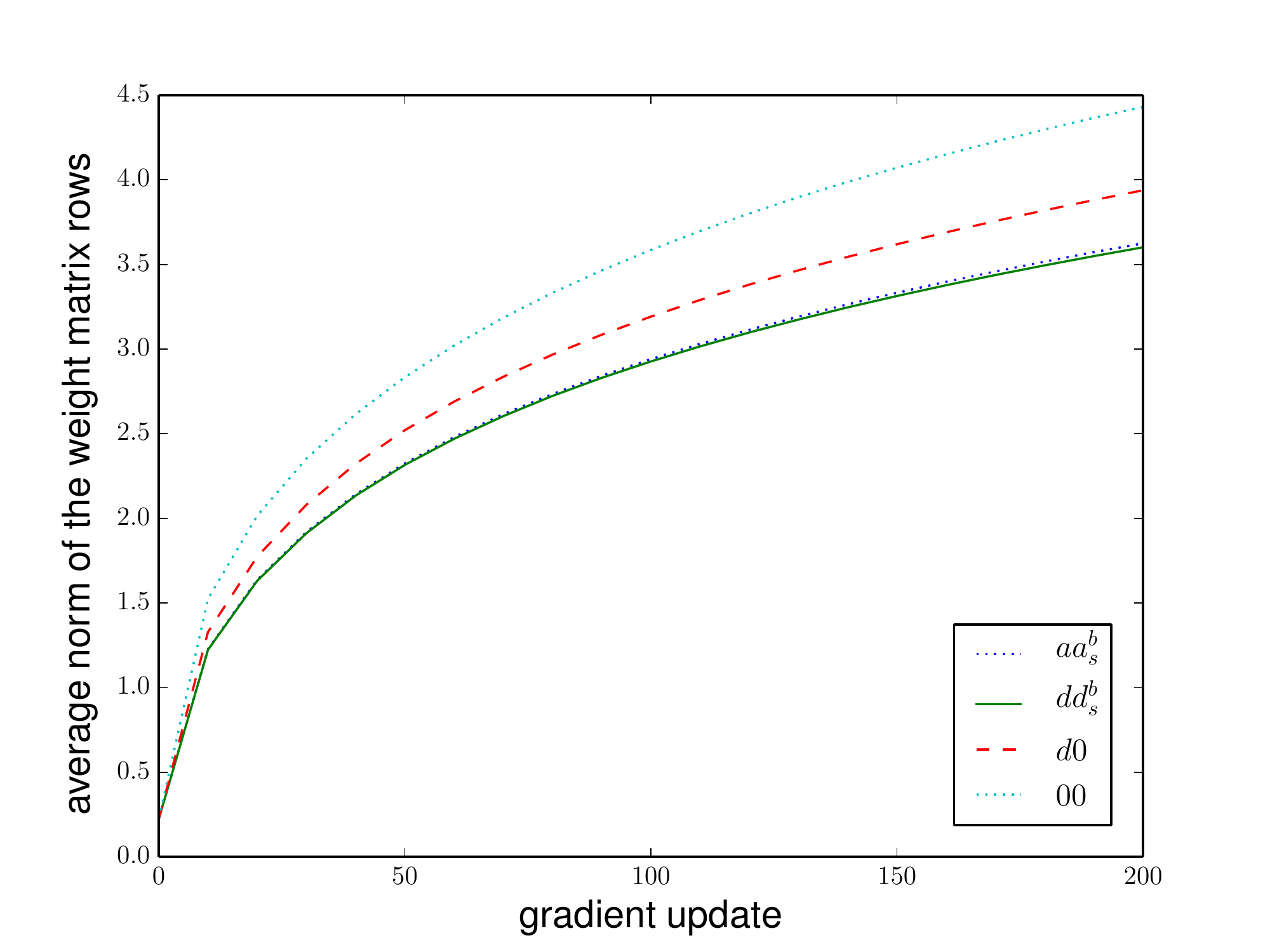}}
\subfigure[Norm of the visible bias]{\label{subfig:MNISTNORMS_C}
\includegraphics[trim=0.6cm 0cm 1.75cm 0cm, clip=true,scale=0.417]{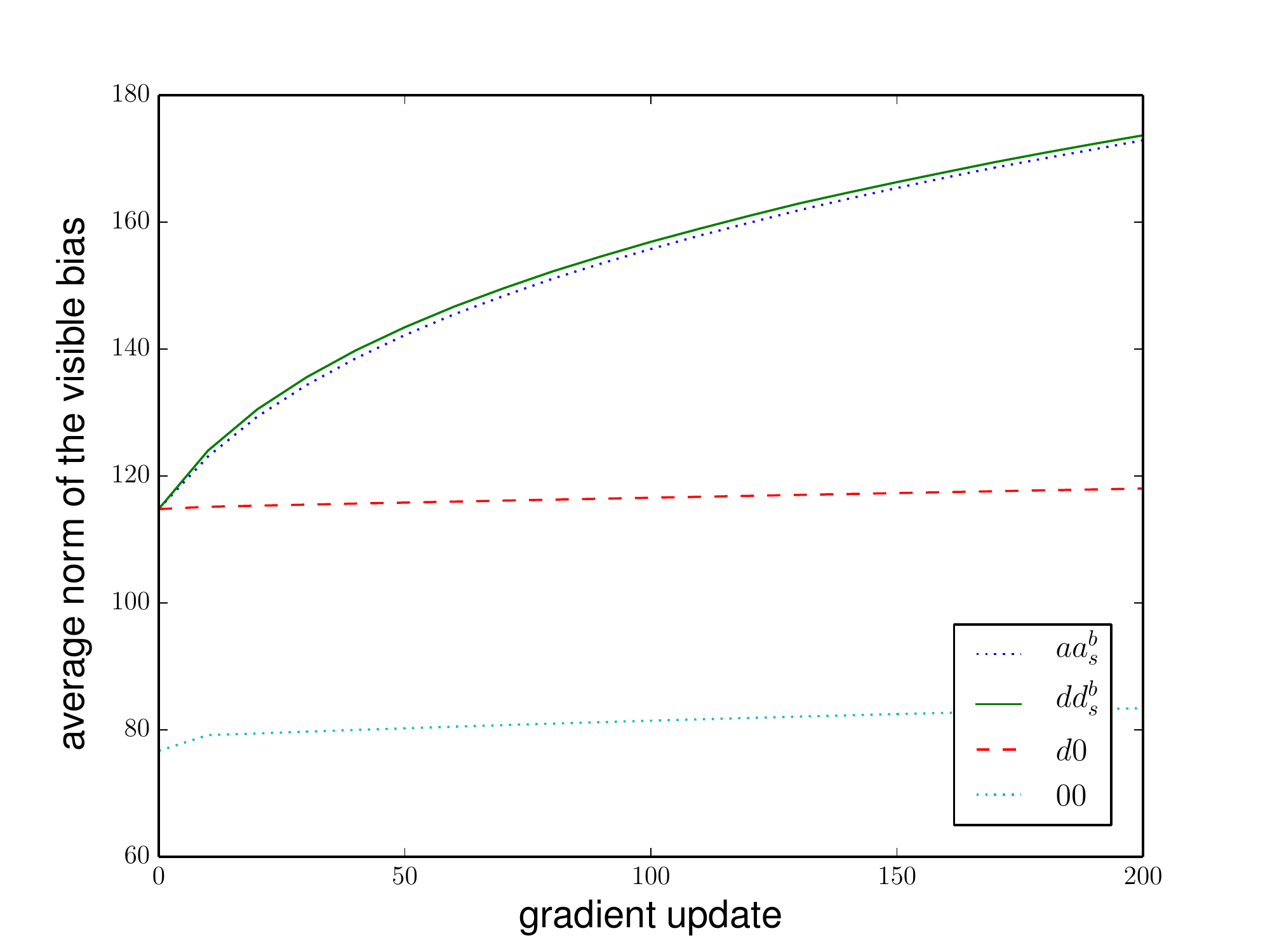}}
\subfigure[Norm of the hidden bias]{\label{subfig:MNISTNORMS_D}
\includegraphics[trim=0.9cm 0cm 1.75cm 0cm, clip=true,scale=0.417]{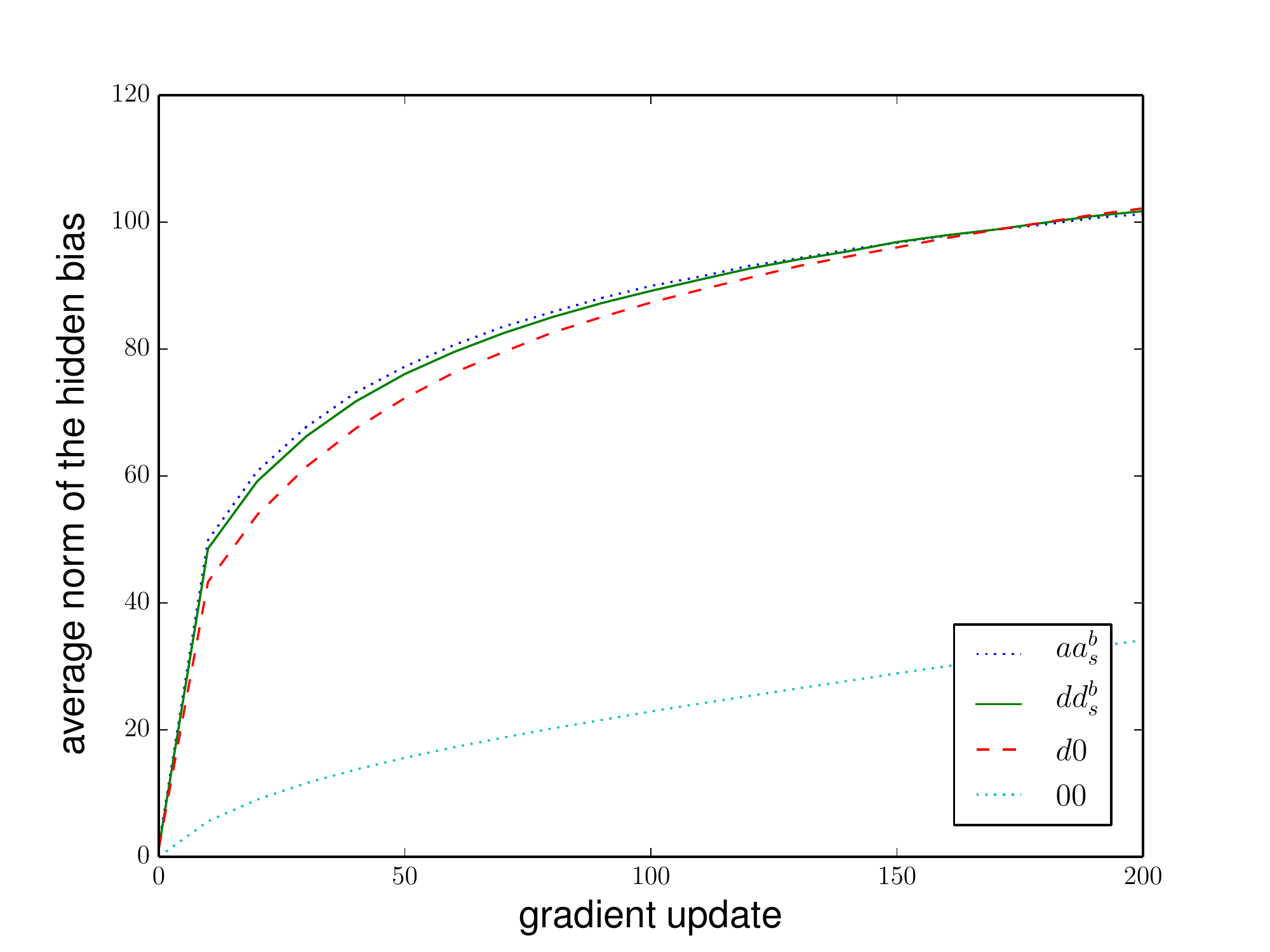}}
\caption{Evolution of the average Euclidean norm of the parameters of the RBMs with 500 hidden units trained on \emph{MNIST}.}
\label{fig:MNISTNORMS}
\end{figure} 

\subsection{Comparision to the Natural Gradient}\label{sec:natural_grad_exp}

The results of the previous section indicate that one explanation for the better performance of the centered gradient compared to the standard gradient is the decoupling of the bias and weight parameters. As described in Section~\ref{sec:theoNatGrad} the natural gradient is independent of the parameterization of the model distribution. Thus it is also independent of how the mean information is stored in the parameters and should not suffer from the described bias-weight coupling problem. 
For the same reason it is also invariant to changes of the representation of the data distribution (e.g. variable flipping). 
That is why we expect the direction of the natural gradient to be closer to the direction of the centered gradient than the direction of the standard gradient.

\begin{figure}[t]
\centering
\subfigure[LL evolution of the natural gradient]{\label{subfig:natgradangle_A}
\includegraphics[trim=0.8cm 0cm 1.75cm 0cm, clip=true,scale=0.42]{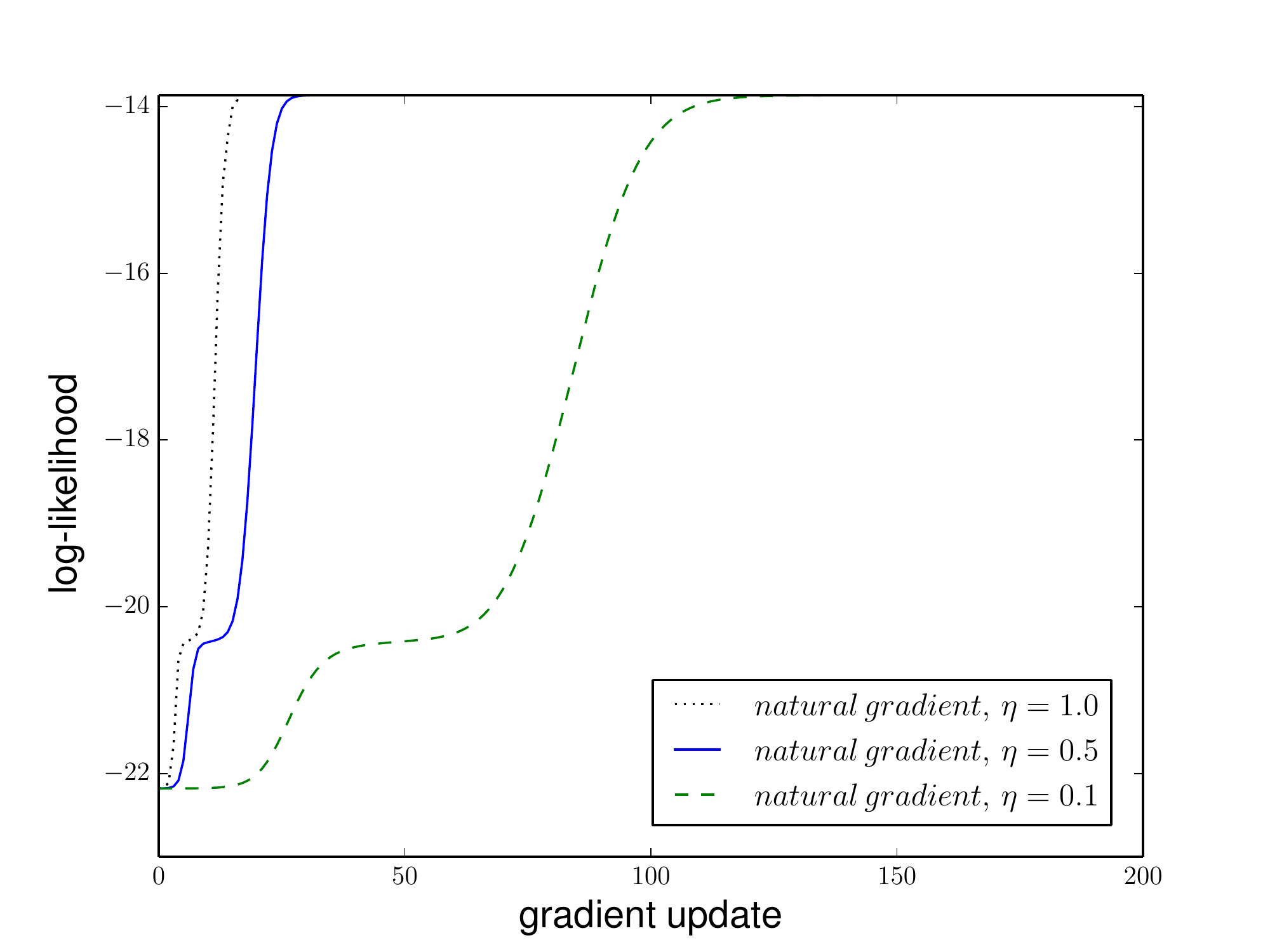}}
\subfigure[Angle between standard natural gradient]{\label{subfig:natgradangle_B}
\includegraphics[trim=1.13cm 0cm 1.7cm 0cm, clip=true,scale=0.42]{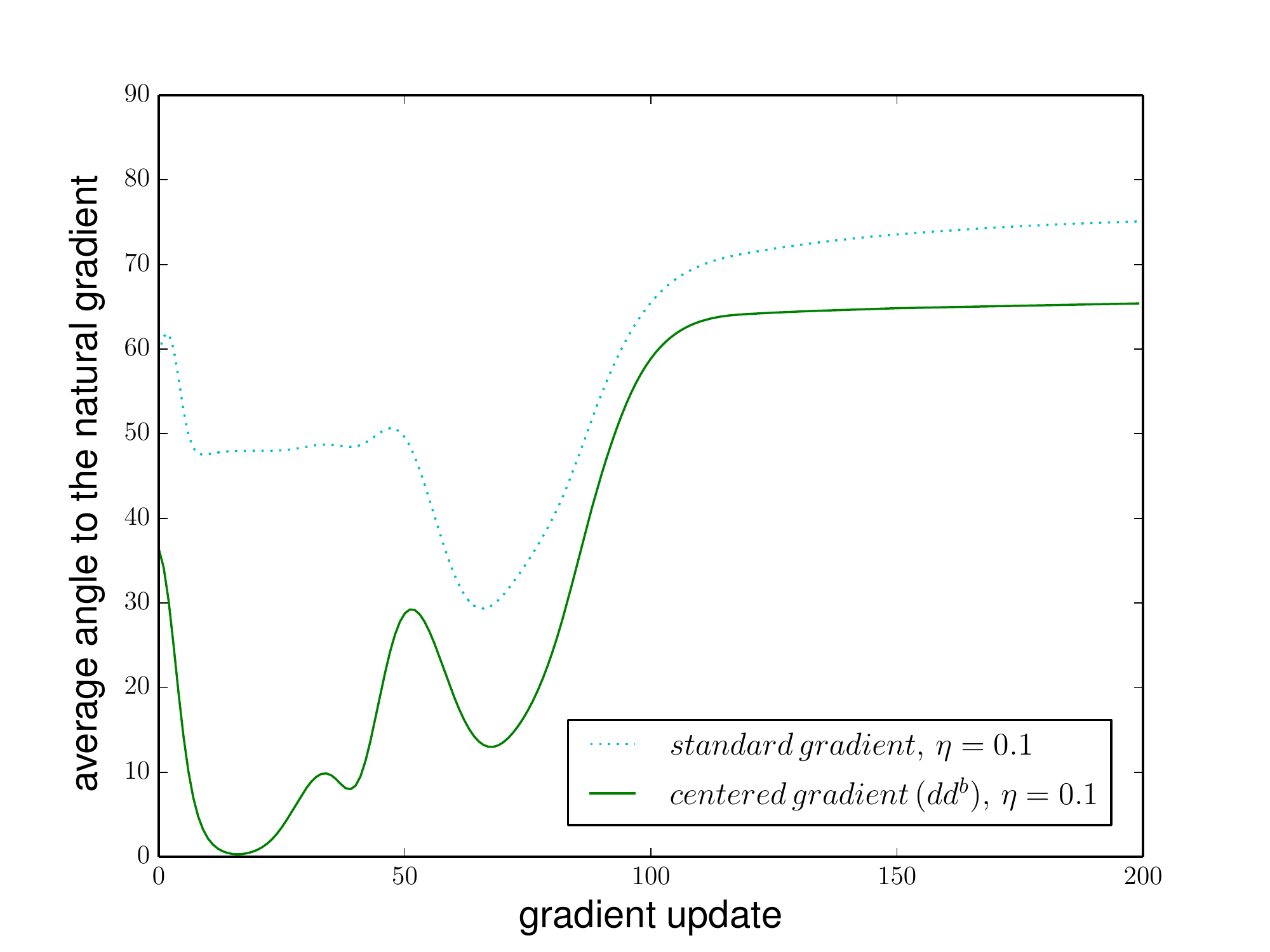}}
\subfigure[Comparision of the LL evolution of standard, natural and centered gradient]{\label{subfig:natgradangle_C}
\includegraphics[trim=4.25cm 0cm 1.55cm 0cm, clip=true,scale=0.362]{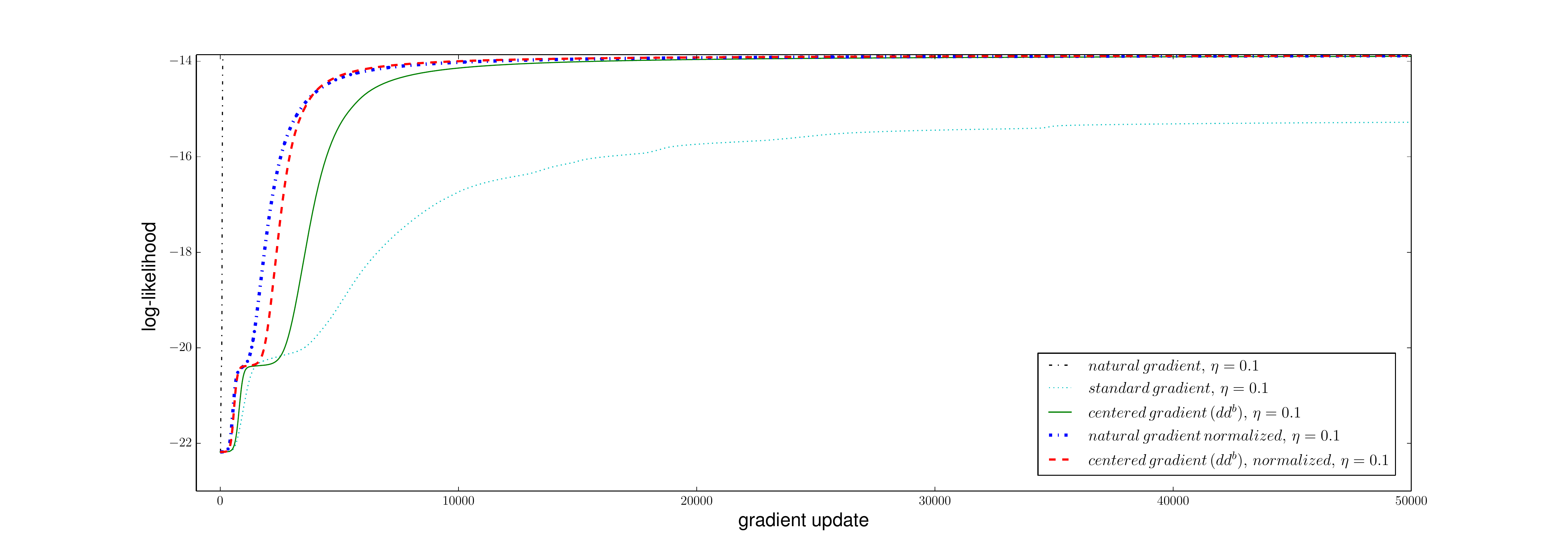}}
\caption{Comparison of the centered gradient, standard gradient, and natural gradient for RBMs with 4 visible and 4 hidden units trained on \emph{Bars \& Stripes 2x2}. (a) The average LL evolution over 25 trials when the natural gradient is used for training with different learning rates, (b) the average angle over 25 trials between the natural and standard gradient as well as natural and centered gradient when a learning rate of 0.1 is used, and (c) average LL evolution over 25 trials when either the natural gradient, standard gradient or centered gradient is used for training. }
\label{fig:natgradangle}
\end{figure} 

To verify this hypothesis empirically, we trained small RBMs with 4 visible and 4 hidden units using the exact natural gradient on the 2x2 \emph{Bars \& Stripes} dataset. After each gradient update the different exact gradients were calculated and the angle between the centered and the natural gradient as well as the angle between the standard and the natural gradient were evaluated.
The results are shown in Figure~\ref{fig:natgradangle} where Figure~\ref{subfig:natgradangle_A} shows the evolution of the average LL when the exact natural gradient is used for training with different learning rates. Figure~\ref{subfig:natgradangle_B} shows the average angles between the different gradients during training when the natural gradient is used for training with a learning rate of $0.1$. 
The angle between centered and natural gradient is consistently much smaller than the angle between standard and natural gradient. 
Comparable results can also be observed for the \emph{Shifting \& Bars} dataset and when the standard, or centered gradient is used for training. 

Notice, how fast the natural gradient reaches a value very close to the theoretical LL upper bound of $-13.86$ even for a learning rate of $0.1$. 
This verifies empirically the theoretical statement that the natural gradient is clearly the update direction of choice, which should be used if it is tractable.
To further emphasize how quick the natural gradient converges, we compared the average LL evolution of the standard, centered and natural gradient, as shown in Figure~\ref{subfig:natgradangle_C}.  Although much slower than the natural gradient, the centered gradient reaches the theoretical upper bound of the LL. The standard gradient seems to saturate on a much smaller value, showing again 
the inferiority of the standard gradient even if it is calculated exactly and not only approximated. 

To verify that the better performance of natural and centered gradient is not only due to larger gradients resulting in bigger step sizes, we also analyzed the LL evolution for the natural and centered gradient when they are scaled to the norm of the normal gradient before updating the parameters. The results are shown in Figure~\ref{subfig:natgradangle_C}. 
The natural gradient still outperforms the other methods but it becomes significantly slower than when used with its original norm. The reason why the norm of the natural gradient is somehow optimal can be explained by the fact that for
distributions of the exponential family a natural gradient update on the LL becomes equivalent to performing a Newton step on the LL. In this sense, the Fisher metric  results in an automatic step size adaption such that even a learning rate of $1.0$ can be used as illustrated in  Figure~\ref{subfig:natgradangle_A}.
Interestingly, if the length of the natural gradient is normalized to the length of the centered gradient and therefore the optimal step size is ignored, the centered gradient becomes almost as fast as the natural gradient. 
The fact that the normalization of the centered gradient increases the resulting learning speed shows that the norm of the centered gradient is smaller than the norm of the normal gradient. Therefore, the worse performance of the normal gradient does not result from the length but the direction of the gradient.
To conclude, these results support our assumption that the centered gradient is closer to the natural gradient and that it is therefore preferable over the standard gradient.

\subsection{Experiments with DBMs}
 
When centering was firstly applied to DBMs by \citet{MontavonMueller-2012} the authors only saw an improvement of centering for locally connected DBMs.
Due to our observations for RBMs and the structural similarity between RBMs and DBMs (a DBM is an RBM with restricted connections and partially unknown input data as discussed in Section~\ref{sec:dbm}) we expect that the benefit of centering carries over to DBMs.
To verify this assumption and empirically investigate the different centering variants in DBMs we performed extensive experiments on the big datasets listed in Section~\ref{sec:Benchmark_problems}. 

Training the models and evaluating the lower bound of the LL was performed as originally proposed for normal DBMs in~\citet{SalakhutdinovHinton-2009a}.
The authors also proposed to pre-train DBMs in a layer-wise fashion based on RBMs~\citep{HintonSalakhutdinov-2012}. In our experiments we trained all models with and without pre-training to investigate the effect of pre-training in both normal and centered DBMs. For pre-training we used the same learning rate and the same offset type as in the final DBM models. 
Notice, that we keep using the term ``average LL'' although it is precisely speaking only the lower bound of the average LL, which has been shown to be rather tied~\citep{SalakhutdinovHinton-2009a}. For the estimation of the partition function we again used AIS where we doubled the number of  intermediate temperatures compared to the  RBM setting to $29000$.
We continue using the short hand notation introduced for RBMs also for DBMs with the only difference that we add a third letter to indicate the offset used in the second hidden layer, such that $000$ corresponds to the normal binary DBM, and $ddd^b_s$ and $aaa^b_s$ corresponds to the centered DBMs using the data mean and the average of data and model mean as offsets, respectively. 
Due to the large number of experiments and the high computational cost -- especially for estimating the LL -- the experiments where repeated only $10$ times and we focused our analysis only on normal DBMs ($000$) and fully centered DBMs ($ddd^b_s$, $aaa^b_s$).

\begin{figure}[t]
\centering
\subfigure[Without pre-training]{
\includegraphics[trim=0.6cm 0cm 1.43cm 0cm, clip=true,scale=0.41]{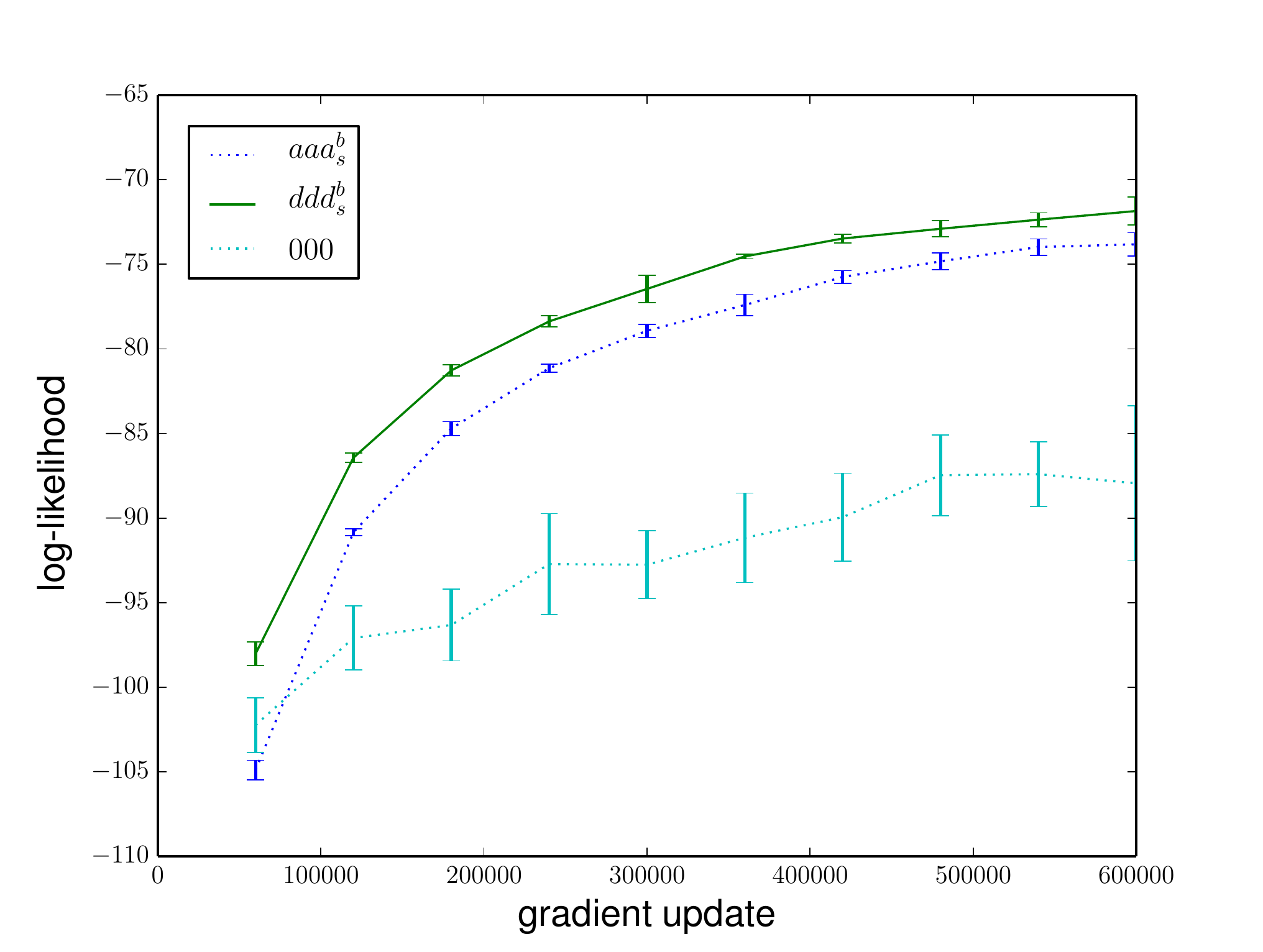}}
\subfigure[With pre-training]{
\includegraphics[trim=1.6cm 0cm 1.43cm 0cm, clip=true,scale=0.41]{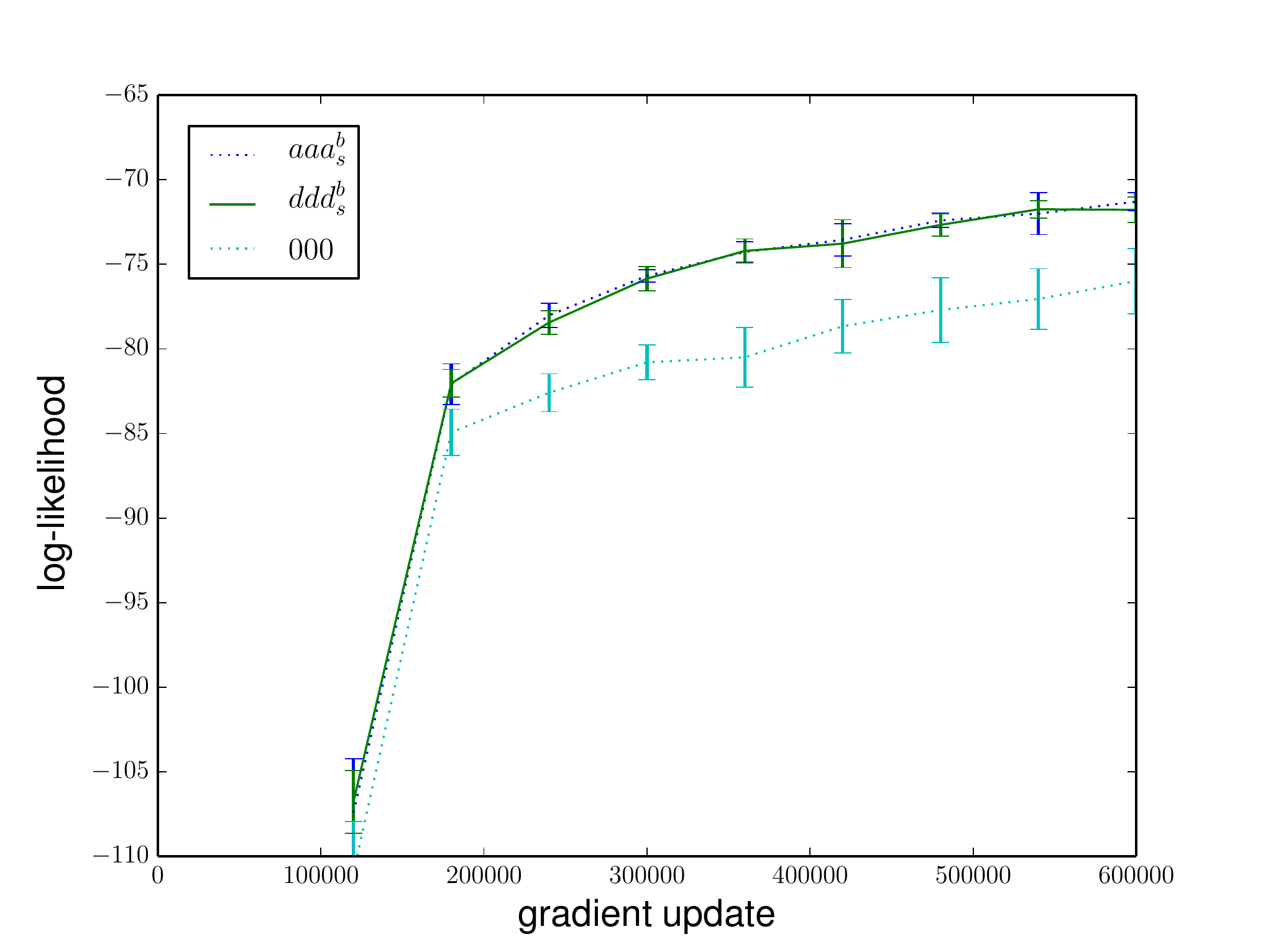}}
\caption{Evolution of the average LL on the test data of the \emph{MNIST} dataset for DBMs with 500x500 hidden units. The different variants $aaa^b_s$, $ddd^b_s$ and $000$ were either trained (a) without pre-training or (b) when each DBM layer was pre-trained for 120.000 gradient updates (200 epochs). In both cases PCD-$1$ with a learning rate of $\eta=0.001$ and for centering a sliding factor of 0.01 was used. The error bars indicate the standard deviation of the average LL over the 10 trials. We skipped evaluating the initial model and (b) starts after the 200 epochs of pre-training to roughly account for the computation overhead of pre-training}
\label{fig:MNIST_DBM}
\end{figure} 

\begin{figure}[t]
\centering
	\subfigure[$000$ without pre-training]{\label{subfig:000_no_pre}
	\includegraphics[trim=0.0cm 0cm 0cm 0cm, clip=true,scale=0.7]{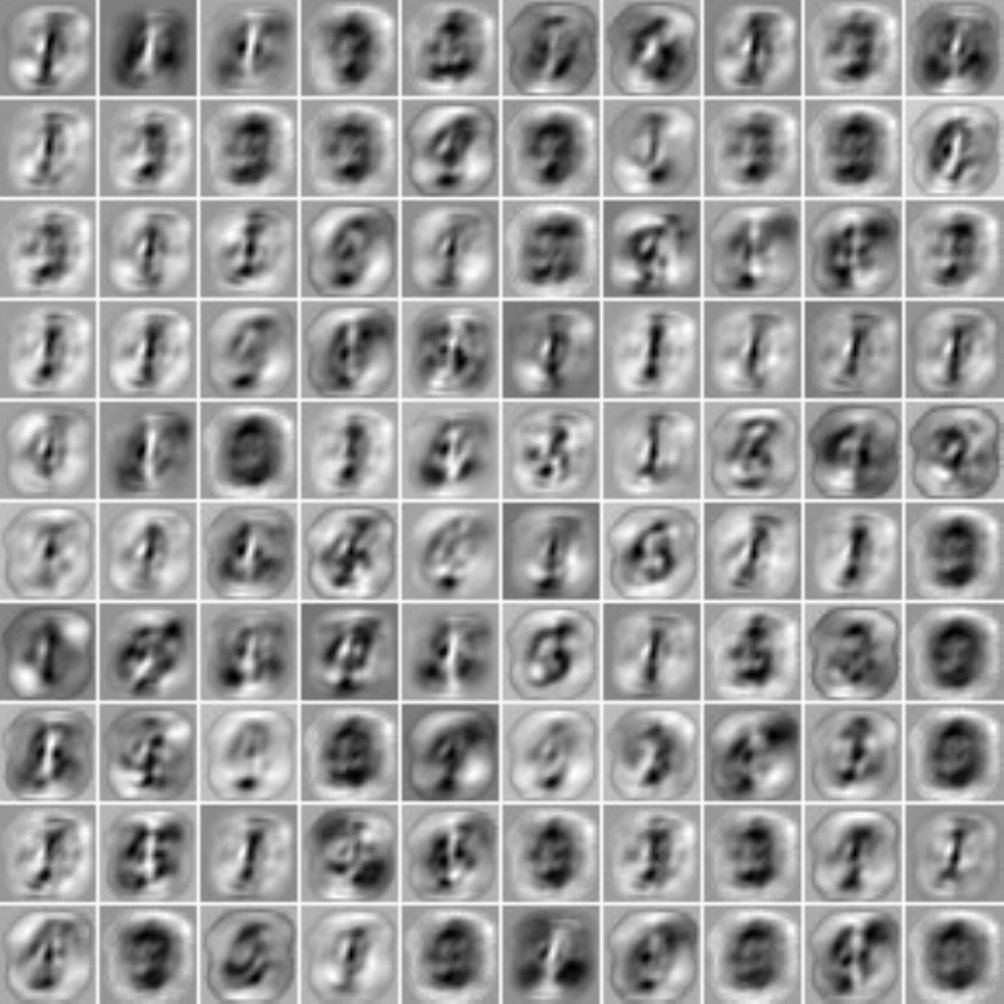}}
	\hspace{0.4cm}
	\subfigure[$ddd^b$ without pre-training]{\label{subfig:DDD_no_pre}
	\includegraphics[trim=0cm 0cm 0cm 0cm, clip=true,scale=0.7]{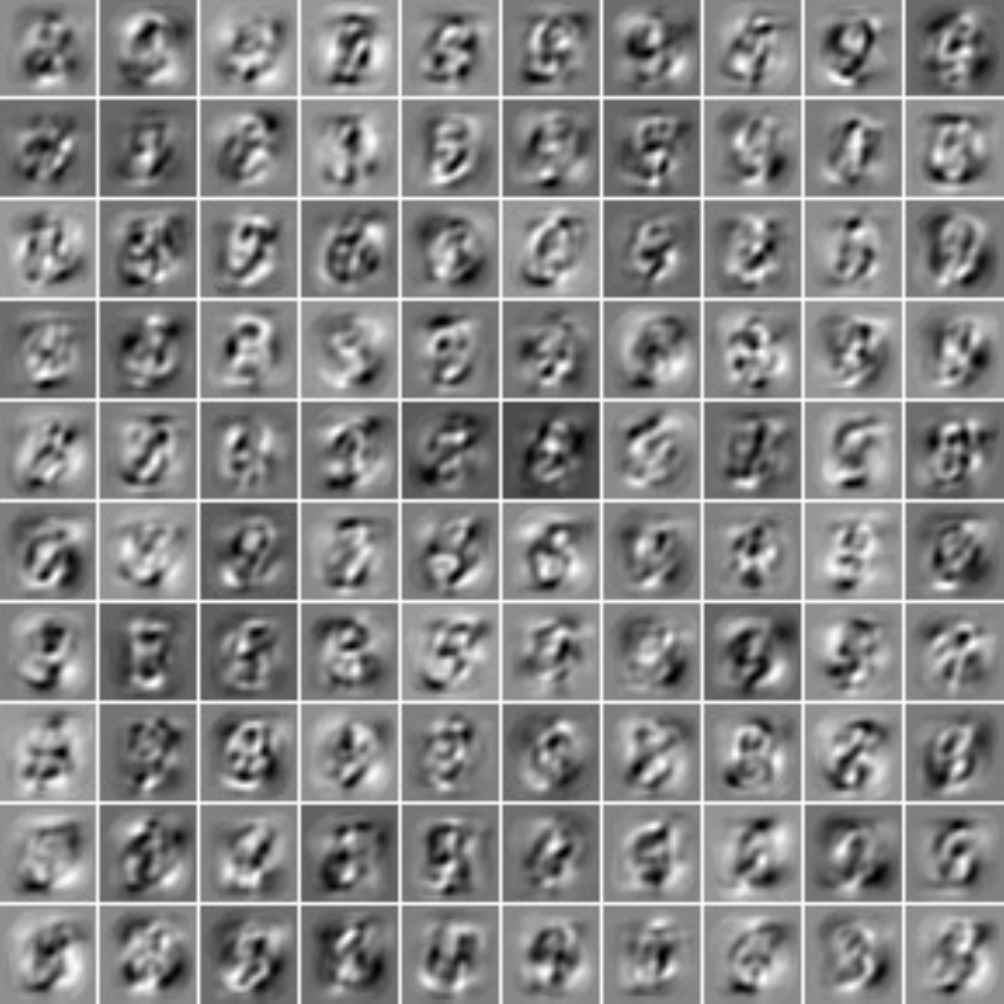}}	
	\subfigure[$000$ with pre-training]{\label{subfig:000_pre}
	\includegraphics[trim=0.0cm 0cm 0cm 0cm, clip=true,scale=0.7]{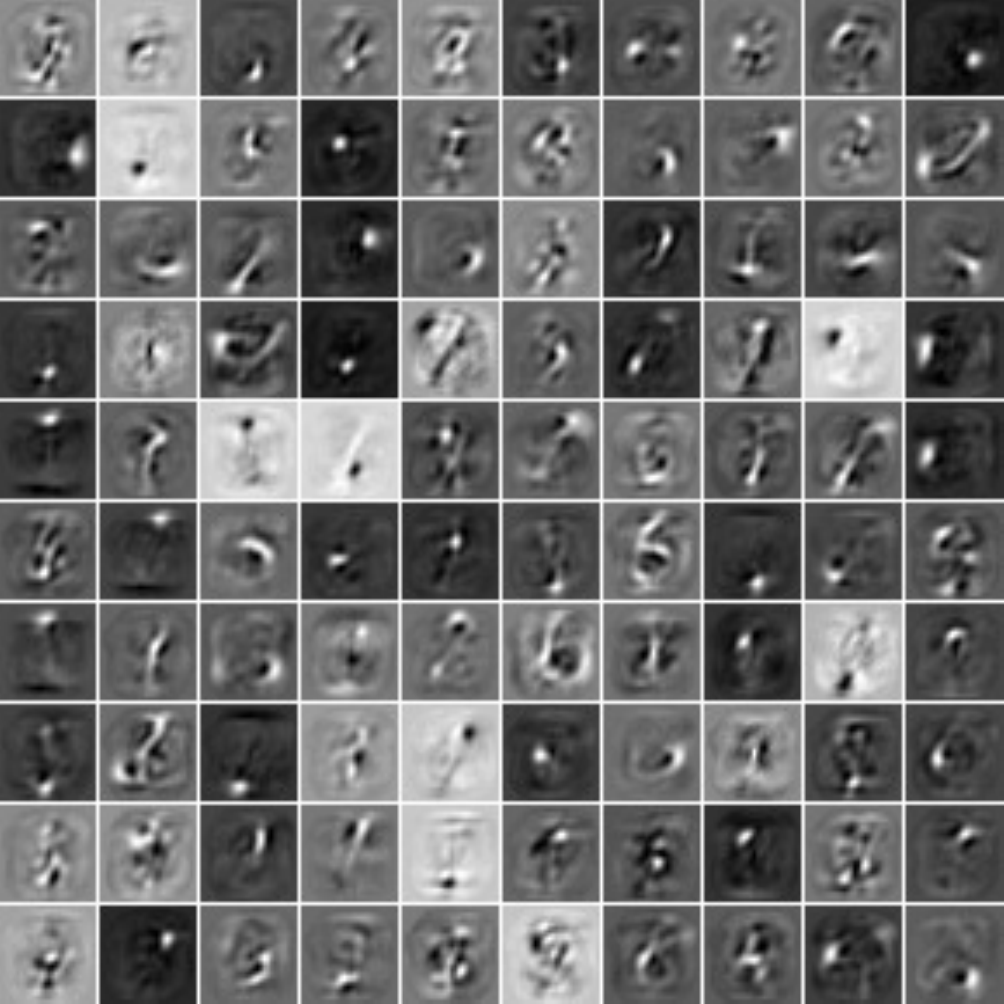}}
	\hspace{0.4cm}
	\subfigure[$ddd^b$ with pre-training]{\label{subfig:DDD_pre}
	\includegraphics[trim=0cm 0cm 0cm 0cm, clip=true,scale=0.7]{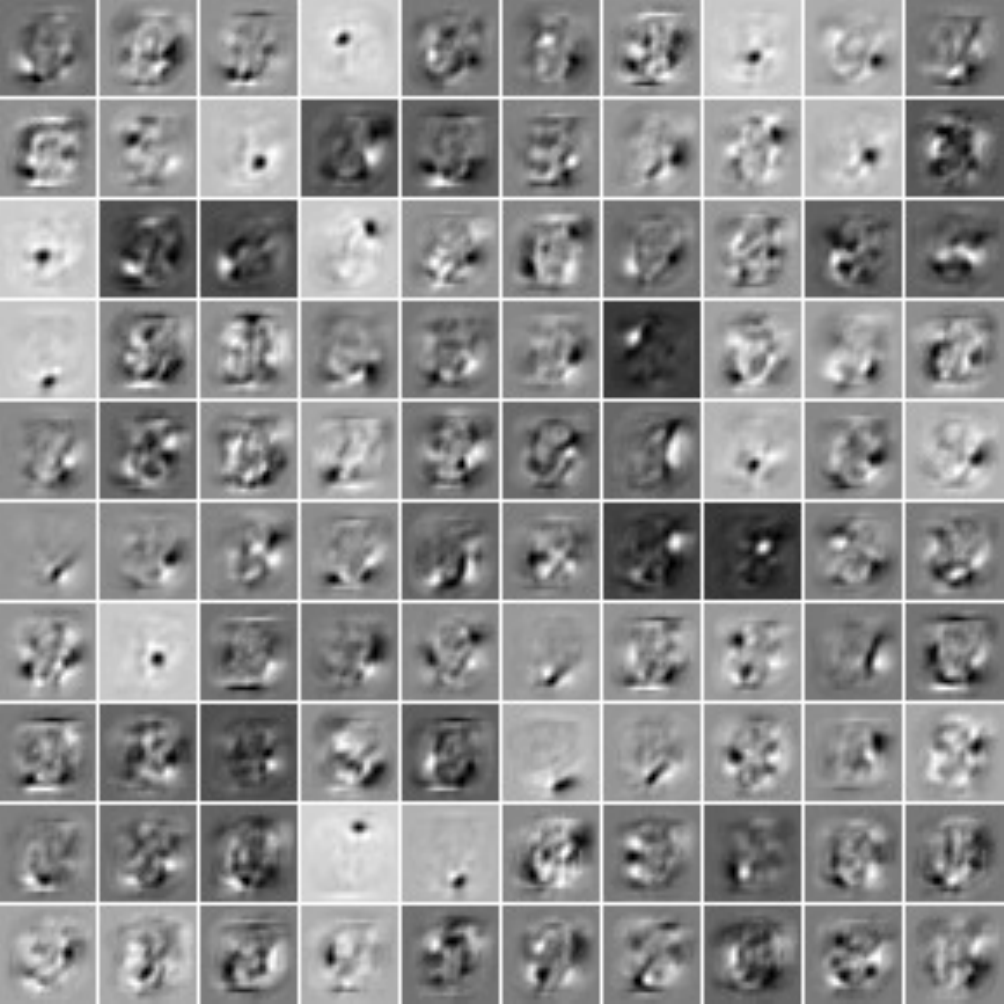}}	
\caption{Random selection of 100 linearly projected filters of the second hidden layer for (a) $000$ and (b) $ddd^b$ without pre-training and (c) $000$ and (d) $ddd^b$ without 200 epochs pre-training. The filters have been normalized independently such that the structure can be seen better.}
\label{fig:DBM_filter}
\end{figure} 

Again, we begin our analysis with the \emph{MNIST} dataset on which we trained normal and centered DBMs with $500$ hidden units in the first and $500$ units in the second hidden layer. Training was done using PCD-$1$ with a batch size of $100$, a learning rate of $0.001$ and in case of centering a sliding factor of $0.01$ for the extensive amount of $1000$ epochs ($600000$ gradient updates). The evolution of the average LL on the test data without pre-training is shown in Figure~\ref{fig:MNIST_DBM}(a) where the evolution of the average LL on the training data is not shown since it is almost equivalent. Both centered DBMs reach significantly higher LL values with a much smaller standard deviation between the trials than the normal DBMs (indicated by the error bars) and $ddd^b_s$ performs slightly better than $aaa^b_s$. These findings are different to the observations of \citet{MontavonMueller-2012} who reported only an improvement of the model through centering for locally connected DBMs. This might be due to the different training setup (e.g. learning rate, batch size, shorter training time or approximation of the data dependent part of the LL gradient by Gibbs sampling instead of optimizing the lower bound of the LL)
Figure~\ref{fig:MNIST_DBM}(b) shows the evolution of the average LL on the test data of the same models with pre-training for $120000$ gradient updates ($200$ epochs). The evolution of the average LL on the training data was again almost equivalent. $ddd^b_s$ has approximately the same performance with and without pre-training but $aaa^b_s$ now has similar performance as $ddd^b_s$. Pre-training allows $000$ to reach better LL than without pre-training, however it is still significantly worse compared to the centered DBMs with or without pre-training.
By comparing the results with the results of RBMs with $500$ hidden units trained on \emph{MNIST} shown in Figure~\ref{fig:MNIST500}(a) we see that all DBMs reach higher LL values than the corresponding RBM models. 

\begin{figure}[t]
\centering
\subfigure[Without pre-training]{
\includegraphics[trim=0.6cm 0cm 1.43cm 0cm, clip=true,scale=0.41]{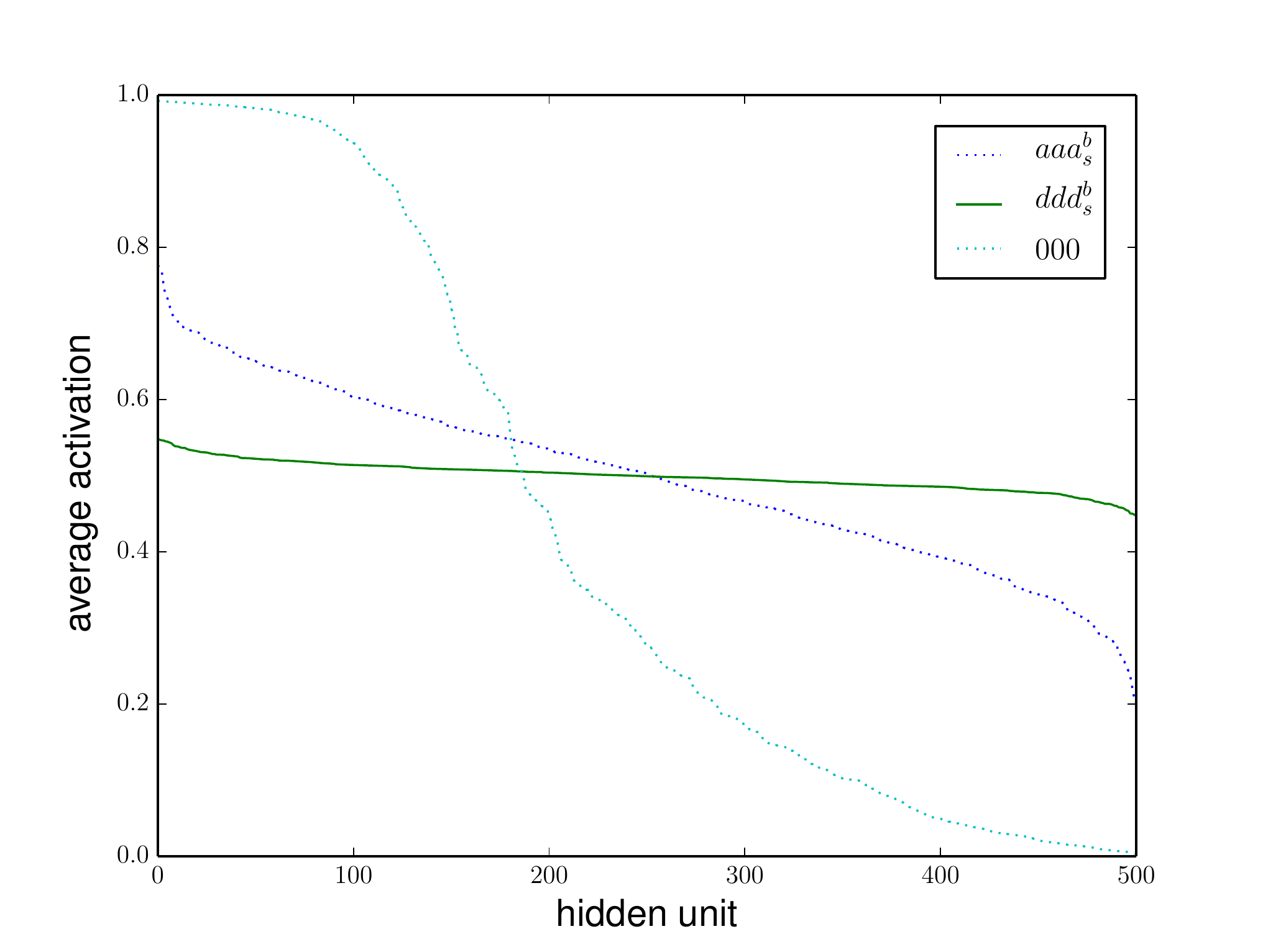}}
\subfigure[With pre-training]{
\includegraphics[trim=1.6cm 0cm 1.43cm 0cm, clip=true,scale=0.41]{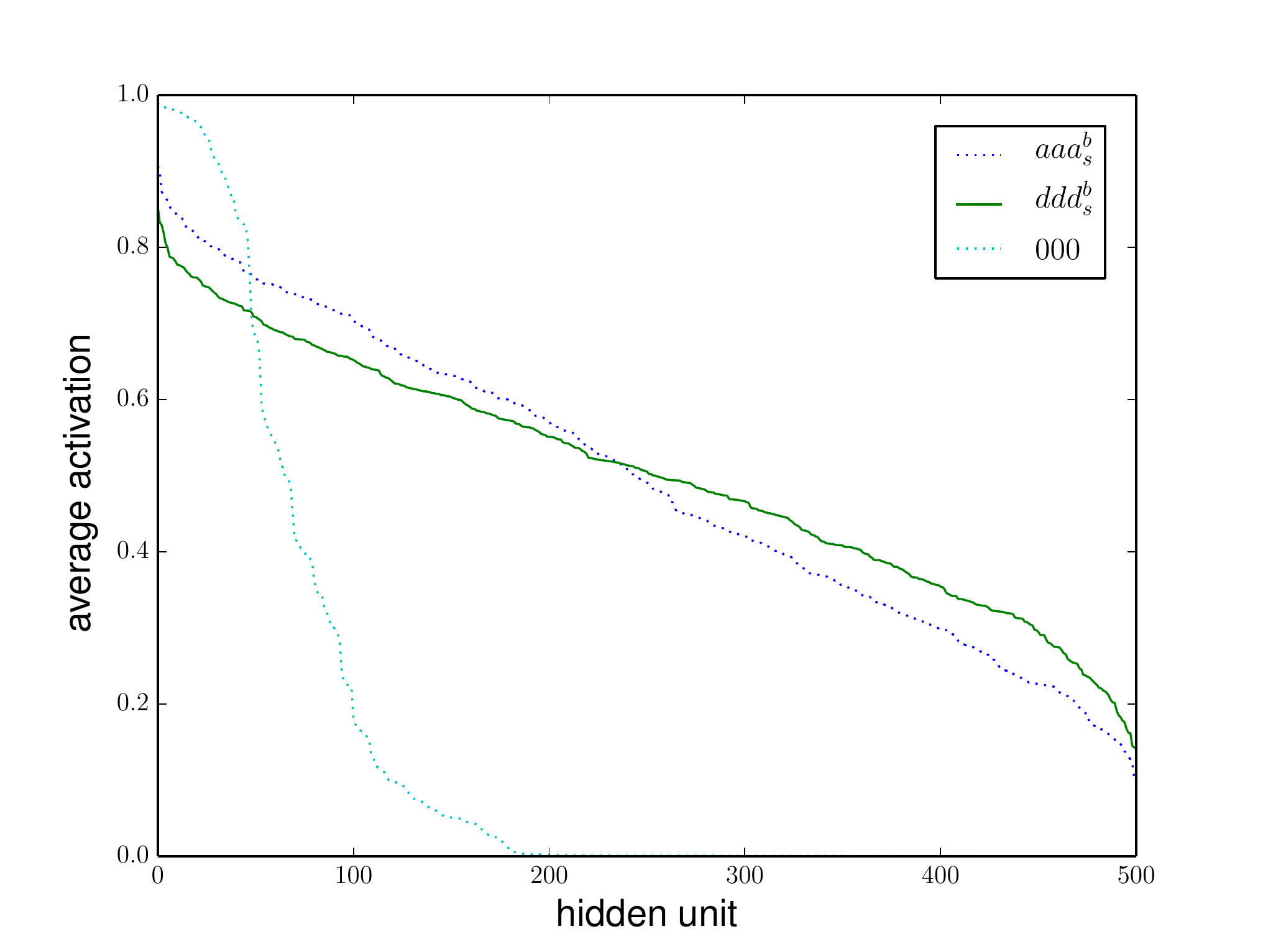}}
\caption{Decreasingly ordered average hidden activity on the training data for the different models (a) without pre-training and (b) with pre-training}
\label{fig:DBM_activity}
\end{figure} 

The higher layer representations in DBMs highly depend on the data driven lower layer representations. Thus we expect to see a qualitative difference between the second layer receptive fields or filters given by the columns of the weight matrices in centered and normal DBMs.
We did not visualize the filters of the first layer since all models showed the well known stroke like structure, which can be seen for RBMs in the review paper by \citet{fischer:14} for example. We visualized the filters of the second layer by linearly back projecting the second layer filters into the input space given by the matrix product of first and second layer weight matrix. The corresponding back projected second layer filters for $000$ and $ddd^b_s$ are shown Figure~\ref{fig:DBM_filter}(a) and (b), respectively. 
It can be seen that many second layer filters of $000$ are roughly the same and thus highly correlated. Moreover, they seem to represent some kind of mean information. Whereas
the filters for $ddd^b_s$ have much more diverse and less correlated structures than the filters of the normal DBM.
When pre-training is used the filters for $000$ become more diverse and the filters of both  $000$ and $ddd^b_s$ become more selective as can be seen in Figure~\ref{fig:DBM_filter}(c) and (d), respectively. 
The effect of the diversity difference of the filters can also be seen from the average activation of the second hidden layer. As shown in Figure~\ref{fig:DBM_activity}(a) without pre-training the average activation of the hidden units of $ddd^b_s$ given the training data is approximately 0.5 for all units while for $aaa^b_s$ it is a bit less balanced and for $000$ most of the units tend to be either active or inactive all the time. The results in Figure~\ref{fig:DBM_activity}(b) illustrate that the average activity for all models become less balanced when pre-training is used, which also reflects the higher selectivity of the filters as shown in Figure~\ref{fig:DBM_filter}(c) and (d). While the second layer hidden activities of $ddd^b_s$ and $aaa^b_s$ stay in a reasonable range, they become extremely selective for $000$ where 300 out of 500 units are inactive all the time. Therefore, the filters, average activation and  evolution of the LL indicate that that normal RBMs have difficulties in making use of the second hidden layer with and without pre-training.
\begin{figure}[ht]
\centering
\subfigure[Test LL without pre-training]{
\includegraphics[trim=0.6cm 0cm 1.43cm 0cm, clip=true,scale=0.41]
{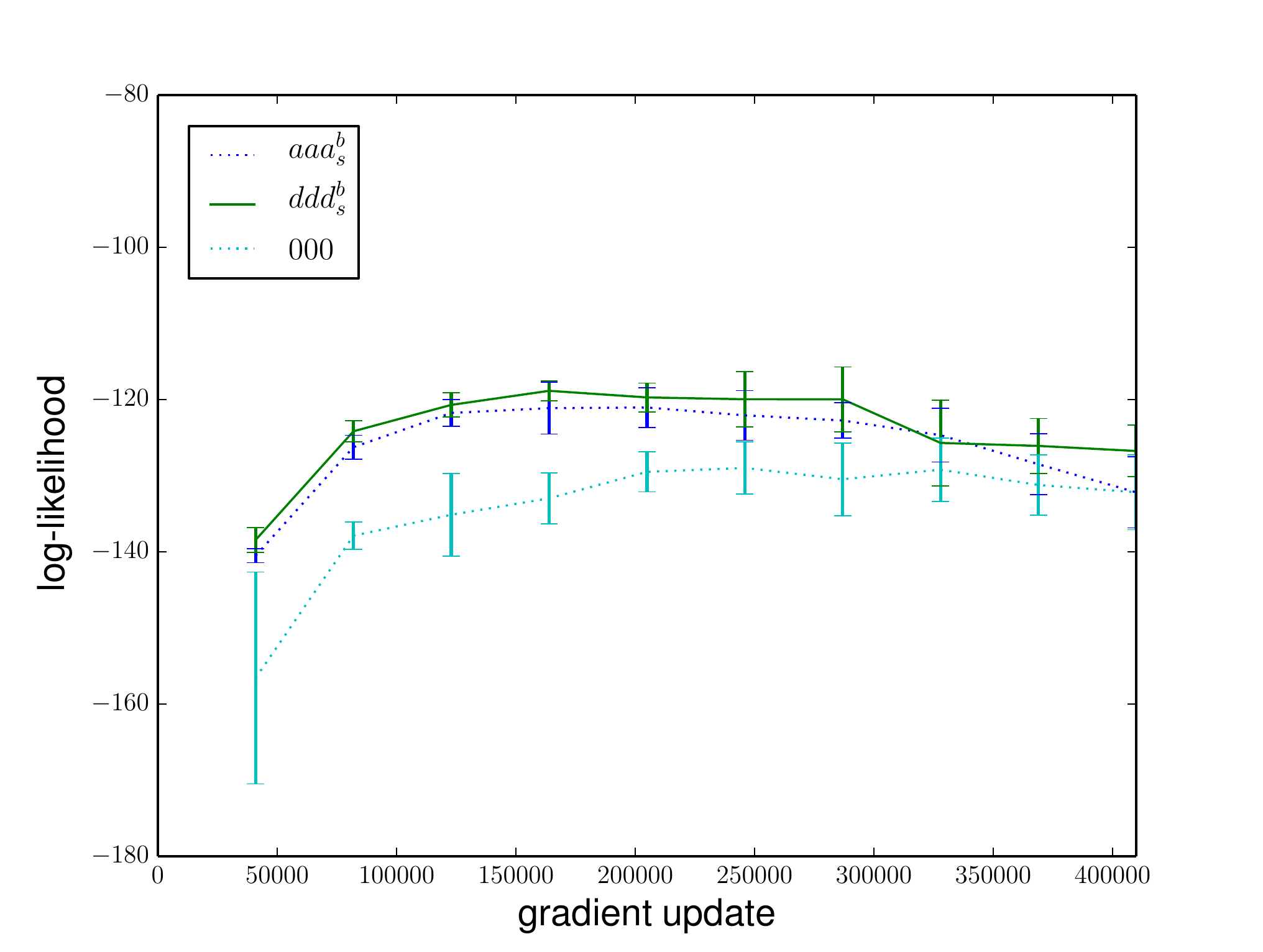}}
\subfigure[Test LL with 500 epochs pre-training]{
\includegraphics[trim=1.6cm 0cm 1.43cm 0cm, clip=true,scale=0.41]{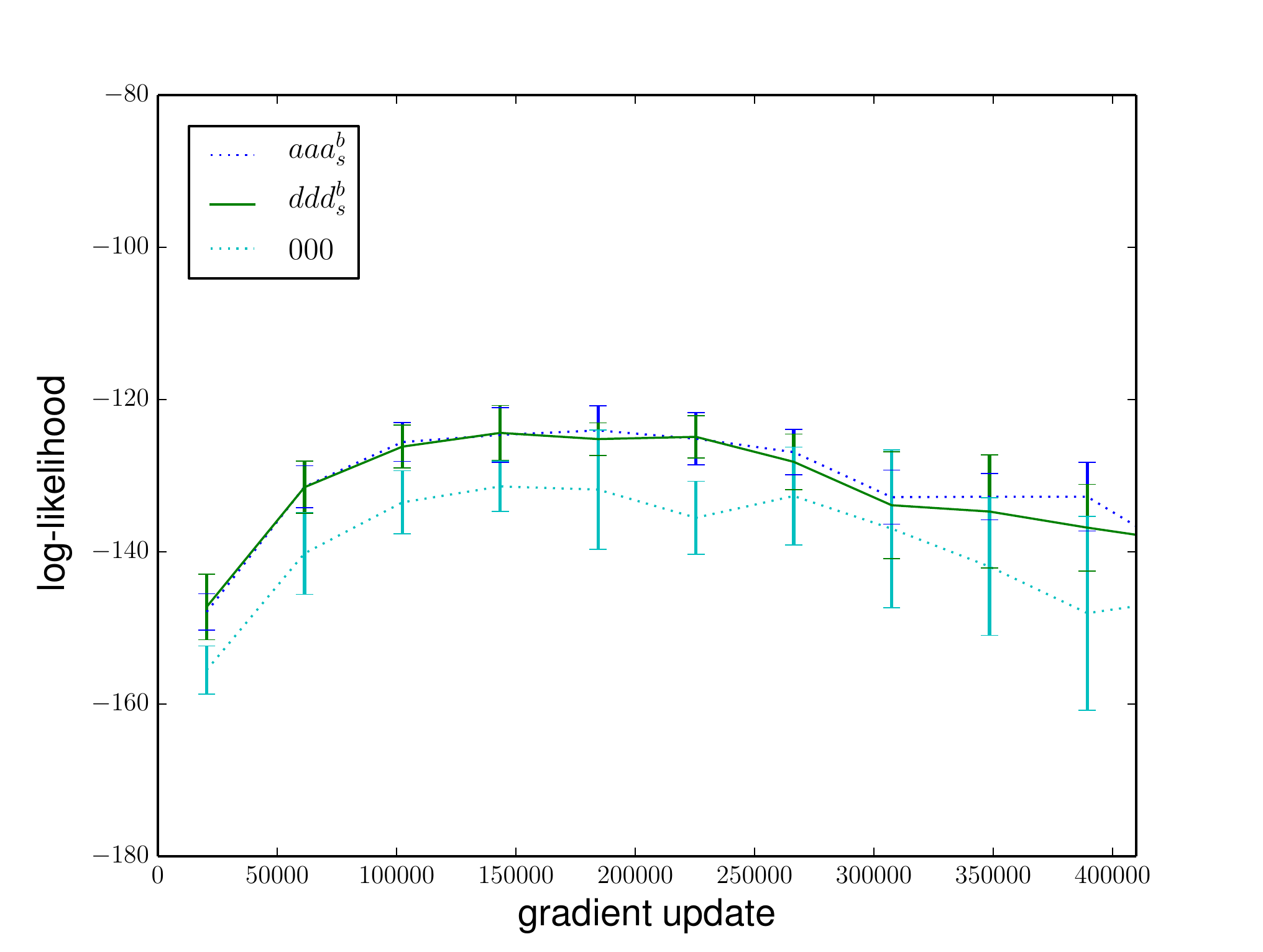}}
\subfigure[Train LL without pre-training]{
\includegraphics[trim=0.6cm 0cm 1.43cm 0cm, clip=true,scale=0.41]{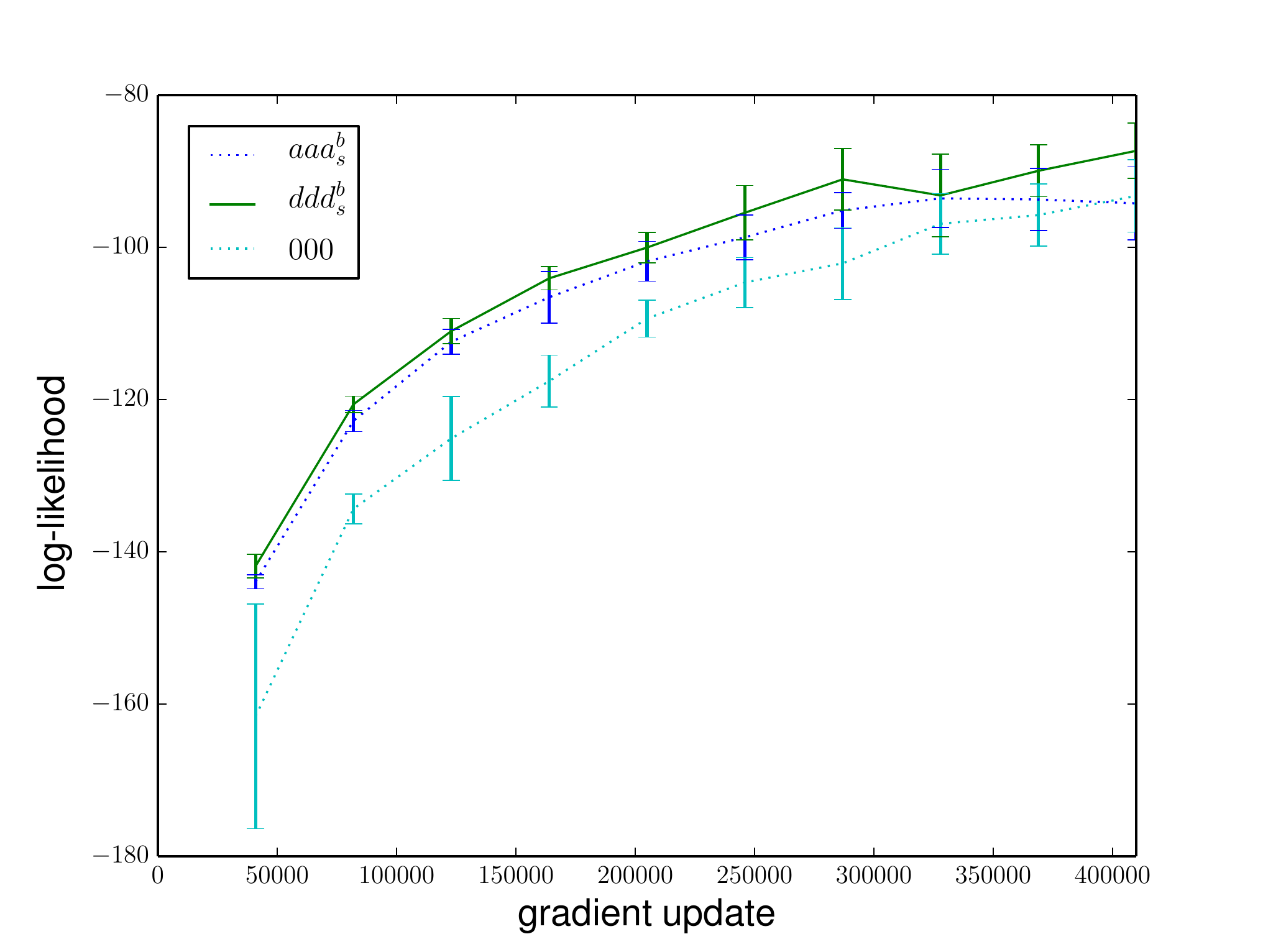}}
\subfigure[Train LL with 500 epochs pre-training]{
\includegraphics[trim=1.6cm 0cm 1.43cm 0cm, clip=true,scale=0.41]{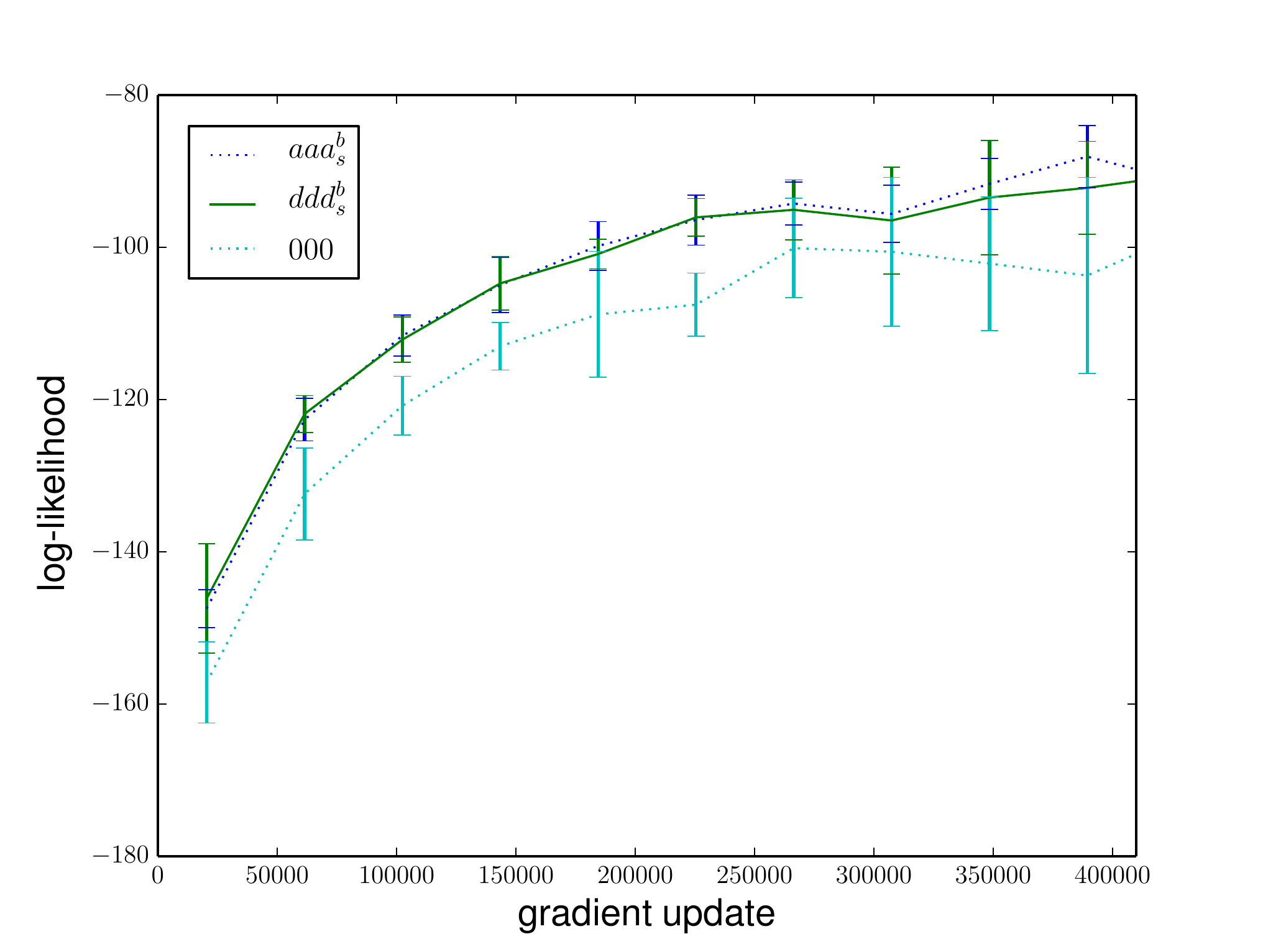}}
\caption{Evolution of the average LL on the \emph{Caltech} dataset for the DBMs $aaa^b_s$, $ddd^b_s$ and $000$ with 500x500 hidden units. (a) test LL (c) train LL without pre-training, (b) test LL and (d) train LL with $500$ epochs (20500 gradients updates) pre-training.
The models were trained using PCD-1 with a batch size of 100, a sliding factor of 0.01 and a learning rate of $\eta=0.001$ was used. The error bars indicate the standard deviation of the LL over the 10 trials. We skipped evaluating the initial model and (b) and (d) start after the 500 epochs of pre-training to roughly account for the computation overhead of pre-training}
\label{fig:CALTECH_DBM}
\end{figure} 

\begin{table}[t]
\begin{center}
\begin{small}
\begin{sc}
\begin{tabular}{l l l}
\hline
\abovespace\belowspace
Dataset & \multicolumn{1}{c}{$dd_s^b$} & \multicolumn{1}{c}{\emph{00}} \\
\hline
\multicolumn{3}{l}{No pre-training}
\abovespace\belowspace \\
adult  &  \underline{-15.54} $\pm$0.42 \,(-17.12)  & -18.44 $\pm$0.12 \,(-24.92)  \\
connect4  &  \underline{-15.09} $\pm$0.39 \,(-40.83)  & -18.15 $\pm$0.09 \,(-43.84)  \\
dna  &  \underline{-89.81} $\pm$0.13 \,(-92.57)  & -91.18 $\pm$0.10 \,(-95.12)  \\
mushrooms  &  \underline{-15.24} $\pm$0.60 \,(-19.68)  & -17.21 $\pm$0.82 \,(-27.71)  \\
nips  &  \underline{-270.35} $\pm$0.09 \,(-360.59)  & -275.43 $\pm$0.13 \,(-360.56)  \\
ocr letters  &  \underline{-30.37} $\pm$0.39 \,(-32.23)  & -31.56 $\pm$0.69 \,(-32.74)  \\
rcv1  &  -46.83 $\pm$0.08 \,(-47.26)  & \underline{-46.51} $\pm$0.16 \,(-47.88)  \\
web  &  \underline{-30.02} $\pm$0.59 \,(-72.88)  & -30.35 $\pm$0.58 \,(-79.36)  \\
\hline
\multicolumn{3}{l}{With pre-training}
\abovespace\belowspace \\
adult  &  \underline{-18.86} $\pm$2.74 \,(-21.43)  & -21.64 $\pm$0.59 \,(-40.42)  \\
connect4  &  \underline{-27.38} $\pm$1.52 \,(-32.13)  & -41.21 $\pm$1.61 \,(-52.04)  \\
dna  &  \underline{-89.87} $\pm$0.11 \,(-94.03)  & -91.06 $\pm$0.10 \,(-97.48)  \\
mushrooms  &  -24.23 $\pm$5.43 \,(-35.07)  & \underline{-21.42} $\pm$5.87 \,(-35.82)  \\
nips  &  \underline{-272.92} $\pm$0.16 \,(-404.11)  & -276.88 $\pm$0.21 \,(-378.76)  \\
ocr letters  &  -36.89 $\pm$1.44 \,(-39.76)  & \underline{-32.25} $\pm$2.04 \,(-35.01)  \\
rcv1  &  -47.79 $\pm$0.84 \,(-49.30)  & \underline{-46.90} $\pm$1.04 \,(-48.36)  \\
web  &  \underline{-31.10} $\pm$0.14 \,(-81.93)  & -32.43 $\pm$0.14 \,(-47.73)  \\
\hline
\end{tabular}
\end{sc}
\end{small}
\end{center}
\caption {Maximum average LL on test data on various datasets for DBMs with 200 hidden units on the first and second layer. For training (top) without pre-training and (bottom) with 200 epochs pre-training PCD-$1$, with a learning rate of $0.01$, batch size of 100 was used. (The best result is underlined).}
\label{tab:testLL_Various_DBM}
\end{table}

\begin{table}[t]
\begin{center}
\begin{small}
\begin{sc}
\begin{tabular}{l l l}
\hline
\abovespace
Dataset & \multicolumn{1}{c}{$dd_s^b$} & \multicolumn{1}{c}{\emph{00}} \\
\hline
\multicolumn{3}{l}{No pre-training}
\abovespace\belowspace \\
adult  &  \underline{-14.65} $\pm$0.37 \,(-15.49)  & -17.88 $\pm$0.12 \,(-23.04)  \\
connect4  &  \underline{-14.68} $\pm$0.38 \,(-39.74)  & -17.82 $\pm$0.07 \,(-42.78)  \\
dna  &  \underline{-62.00} $\pm$1.11 \,(-62.42)  & -62.48 $\pm$0.96 \,(-62.98)  \\
mushrooms  &  \underline{-14.74} $\pm$1.54 \,(-18.61)  & -16.62 $\pm$0.82 \,(-26.53)  \\
nips  &  \underline{-107.09} $\pm$0.21 \,(-107.09)  & -114.28 $\pm$0.21 \,(-114.28)  \\
ocr letters  &  \underline{-29.15} $\pm$0.67 \,(-30.91)  & -30.41 $\pm$0.69 \,(-31.57)  \\
rcv1  &  \underline{-45.75} $\pm$0.07 \,(-46.18)  & -45.80 $\pm$0.25 \,(-47.17)  \\
web  &  \underline{-29.37} $\pm$0.59 \,(-69.70)  & -29.68 $\pm$0.58 \,(-77.11)  \\
\hline
\multicolumn{3}{l}{With pre-training}
\abovespace\belowspace \\
adult  &  \underline{-17.11} $\pm$2.70 \,(-19.66)  & -21.42 $\pm$0.59 \,(-38.21)  \\
connect4  &  \underline{-26.62} $\pm$6.66 \,(-31.14)  & -40.53 $\pm$1.56 \,(-51.29)  \\
dna  &  \underline{-59.97} $\pm$0.49 \,(-60.37)  & -61.16 $\pm$1.51 \,(-61.72)  \\
mushrooms  &  -23.89 $\pm$5.27 \,(-34.11)  & \underline{-20.59} $\pm$5.85 \,(-34.86)  \\
nips  &  \underline{-114.51} $\pm$3.76 \,(-118.29)  & -118.90 $\pm$6.19 \,(-120.94)  \\
ocr letters  &  -35.39 $\pm$1.88 \,(-37.70)  & \underline{-30.90} $\pm$3.53 \,(-33.21)  \\
rcv1  &  -46.54 $\pm$0.87 \,(-48.07)  & \underline{-45.93} $\pm$1.04 \,(-47.39)  \\
web  &  \underline{-30.41} $\pm$0.15 \,(-78.45)  & -31.68 $\pm$0.15 \,(-44.42)  \\
\hline
\end{tabular}
\end{sc}
\end{small}
\end{center}
\caption {Maximum average LL on training data on various datasets for DBMs with 200 hidden units on the first and second layer. For training (top) without pre-training and (bottom) with 200 epochs pre-training PCD-$1$, with a learning rate of $0.01$, batch size of 100 was used. (The best result is underlined).}
\label{tab:trainLL_Various_DBM}
\end{table}

We continue our analysis with experiments on the \emph{Caltech} dataset on which we again trained normal and centered DBMs with $500$ hidden units on the first and second hidden layer. Training was also done using PCD-$1$ with a batch size of $100$, a learning rate of $0.001$ and in case of centering a sliding factor of $0.01$. Since the training data has only 41 batches the models were trained for the extensive amount of $10000$ epochs ($410000$ gradient updates).
Figure~\ref{fig:CALTECH_DBM} shows the average LL on the test data (a) without and (b) with 500 epochs pre-training. In addition, Figure~\ref{fig:CALTECH_DBM}(c) and (d) show the corresponding average LL on the training data demonstrating that all models overfitted to the \emph{Caltech} dataset. The results are consistent with the findings for \emph{MNIST}, that is $000$ performs worse than centering on training and test data independent of whether pre-training is used or not. Furthermore, $aaa^b_s$ seems to perform slightly worse than $ddd^b_s$ without pre-training, while the performance becomes equivalent if pre-training is used. But in contrast to the results of \emph{MNIST},  on \emph{Caltech} all methods perform worse  with pre-training.
This negative effect of pre-training becomes even worse when the number of pre-training epochs is increased. In the case of 2000 epochs of pre-training for example, $ddd^b_s$ and $aaa^b_s$ still perform better than $000$ but the maximal average LL among all models, which was reached by $ddd^b_s$ was only -98.1 for the training data and -141.4 for the test data, compared to -90.4 and -124.0 when 500 epochs of pre-training were used and -87.3 and -118.8 when no pre-training was used. Without pre-training the LL values are comparable to the results when an RBM with 500 hidden units is trained on \emph{Caltech} as shown in Figure~\ref{fig:CALTECH500}, illustrating that in terms of the LL a DBM does not necessarily perform better than an RBM. 
We also visualized the filters and plotted the average hidden activities for the training data of \emph{Caltech}, which lead to the same conclusions as the results for \emph{MNIST} and are not shown for this reason.  

Finally, we also performed experiments on the eight {\bf{\emph{additional binary datasets}}} described in Section~\ref{sec:Benchmark_problems} using the same training setup as for the corresponding RBM experiments. That is, 
the DBMs with 200x200 hidden units were trained for $5000$ epochs with PCD-$1$, a batch size of $100$, a learning rate of $0.01$, and in the case of centering a sliding factor of $0.01$. The LL was evaluated every 50th epoch and in the case of pre-training the models were pre-trained for 200 epochs. 
Table~\ref{tab:testLL_Various_DBM} shows the maximum average LL on the test data (top) without pre-training and (bottom) with pre-training. Without pre-training the results are consistent with the findings for RBMs  that $ddd^b$ performs better than $000$ on all datasets except for RCV1 where $000$ performs slightly better. 
The LL values for the DBMs are comparable but not necessarily better than the corresponding LL values for RBMs, which are shown in Table~\ref{tab:testLL_Various}. In the case of the web datasets for example the DBMs even perform worse than the RBM models.
When pre-training is used the performance of all models, centered or normal, is worse than the performance of the corresponding DBMs without pre-training. For completeness Table~\ref{tab:trainLL_Various_DBM} shows the maximum average LL for the training data leading to the same conclusion as the test data. 
To summarize, the experiments described in this section show that centering leads to higher LL values for DBMs. While pre-training leads to more selective filters in general it often is even harmful for the model quality.

\subsection{Experiments with Auto Encoders}

\begin{table}[htp]
\begin{center}
\begin{small}
\begin{sc}
\begin{tabular}{l l l l}
\hline
\abovespace
Dataset - Learning rate & \multicolumn{1}{c}{$dd_s^b$} & \multicolumn{1}{c}{\emph{00}} \\
\hline
\multicolumn{3}{l}{Test data}
\abovespace\belowspace \\
mnist - 0.1  &  \underline{50.21472} $\pm$0.0256  (5000) & 50.24859 $\pm$0.0200  (5000) \\
mnist - 0.5    &  \underline{50.01338} $\pm$0.0316  (5000) & 56.36068 $\pm$0.7578  (22) \\
caltech  - 0.01 &  \underline{44.38403} $\pm$0.2257  (2500) & 49.21274 $\pm$0.2119  (1968) \\
caltech - 0.1   &  \underline{44.45882} $\pm$0.1620  (246) & 48.59724 $\pm$0.3964  (206) \\
adult - 0.1   &  \underline{0.38837} $\pm$0.0229  (5000) & 0.47460 $\pm$0.0181  (4676) \\
adult - 0.5   &  \underline{0.36825} $\pm$0.0220  (1526) & 0.46086 $\pm$0.0155  (884) \\
connect4  - 0.1 &  \underline{0.03015} $\pm$0.0025  (5000) & 0.03467 $\pm$0.0040  (5000) \\
connect4 - 0.5 &  \underline{0.02357} $\pm$0.0019  (1431) & 0.02856 $\pm$0.0032  (1349) \\
dna - 0.01  &  \underline{34.34353} $\pm$0.1242  (2161) & 34.77299 $\pm$0.1948  (2105) \\
dna - 0.1  &  \underline{34.35117} $\pm$0.1245  (216) & 34.80547 $\pm$0.1823  (210) \\
mushrooms - 0.1 &  0.14355 $\pm$0.0117  (5000) & \underline{0.14226} $\pm$0.0081  (5000) \\
mushrooms - 0.5 &  \underline{0.08555} $\pm$0.0167  (3173) & 0.09960 $\pm$0.0169  (2936) \\
nips - 0.01  &  \underline{183.21045} $\pm$0.6355  (2261) & 188.65301 $\pm$0.7262  (2107) \\
nips - 0.1   &  \underline{183.31413} $\pm$0.6081  (226) & 189.10982 $\pm$0.6674  (212) \\
ocr letters - 0.1  &  \underline{5.41182} $\pm$0.1994  (5000) & 5.46969 $\pm$0.2138  (5000) \\
ocr letters - 0.5 &  \underline{4.91528} $\pm$0.1715  (3703) & 5.30343 $\pm$0.2737  (1945) \\
rcv1 - 0.1  &  12.93456 $\pm$0.1562  (5000) & \underline{12.55443} $\pm$0.3097  (5000) \\
rcv1 - 0.5 &  12.32545 $\pm$0.3296  (4953) & \underline{12.16340} $\pm$1.2188  (2946) \\
web - 0.1 &  \underline{1.07535} $\pm$0.0174  (1425) & 1.22756 $\pm$0.0121  (944) \\
\hline
\multicolumn{3}{l}{Training data}
\abovespace\belowspace \\
mnist - 0.1  &  \underline{50.03172} $\pm$0.0224  (5000) & 50.06083 $\pm$0.0167  (5000) \\
mnist - 0.5   &  \underline{49.84428} $\pm$0.0253  (5000) & 55.89110 $\pm$0.7516  (22) \\
caltech - 0.01 &  \underline{5.71437} $\pm$0.0235  (5000) & 6.42052 $\pm$0.0392  (5000) \\
caltech - 0.1   &  \underline{0.57507} $\pm$0.0030  (5000) & 0.62426 $\pm$0.0119  (5000) \\
adult - 0.1   &  0.04687 $\pm$0.0024  (5000) & \underline{0.03677} $\pm$0.0024  (5000) \\
adult - 0.5   &  \underline{0.03221} $\pm$0.0030  (1526) & 0.04053 $\pm$0.0051  (894) \\
connect4 - 0.1   &  0.01465 $\pm$0.0007  (5000) & \underline{0.01187} $\pm$0.0006  (5000) \\
connect4 - 0.5   &  0.01022 $\pm$0.0002  (1431) & \underline{0.00914} $\pm$0.0004  (1349) \\
dna - 0.01 &  \underline{17.80686} $\pm$0.0812  (5000) & 18.08306 $\pm$0.0665  (5000) \\
dna - 0.1   &  \underline{11.70726} $\pm$0.1018  (5000) & 12.08309 $\pm$0.0995  (5000) \\
mushrooms - 0.5   &  0.01995 $\pm$0.0007  (3173) & \underline{0.01768} $\pm$0.0013  (2936) \\
mushrooms - 0.1   &  0.06362 $\pm$0.0022  (5000) & \underline{0.05136} $\pm$0.0026  (5000) \\
nips - 0.01 &  \underline{21.45218} $\pm$0.0614  (5000) & 21.94275 $\pm$0.0509  (5000) \\
nips - 0.1   &  \underline{2.05141} $\pm$0.0067  (5000) & 2.06286 $\pm$0.0054  (5000) \\
ocr letters - 0.1   &  \underline{4.97384} $\pm$0.1865  (5000) & 5.00803 $\pm$0.2078  (5000) \\
ocr letters - 0.5   &  \underline{4.51704} $\pm$0.1642  (3703) & 4.88637 $\pm$0.2677  (1945) \\
rcv1 - 0.1   &  11.95633 $\pm$0.1226  (5000) & \underline{11.62695} $\pm$0.2792  (5000) \\
rcv1 - 0.5   &  11.36861 $\pm$0.2974  (4953) & \underline{11.27664} $\pm$1.1136  (2947) \\
web - 0.1   &  0.06872 $\pm$0.0007  (5000) & \underline{0.05554} $\pm$0.0003  (5000) \\
\hline
\end{tabular}
\end{sc}
\end{small}
\end{center}
\caption {Average maximal reached cost value with standard deviation on test and training data of various datasets for centered and normal three layer AEs with sigmoid units, cross entropy cost function and half the number of output dimensions than input dimensions. The average number of epochs till convergence is given in brackets.
}
\label{tab:AE_plain}
\end{table}

The benefit of centering in feed forward neural networks for supervised tasks has already been shown  by~\citet{Schraudolph98}.
In this section we want to analyze centering in a special kind of unsupervised feed forward neural networks, namely centered AEs as introduced in Section~\ref{sec:autoencoders}.
We therefore trained normal and centered three layer AEs on the ten big datasets described in Section~\ref{sec:Benchmark_problems}. To avoid trivial solutions we used half the number of output dimensions than input dimensions and tied weights. Since the datasets are binary we used the sigmoid non-linearity in encoder and decoder and the cross entropy cost function. Training was done using plain back propagation for a maximum number of $5000$ epochs without any further modification and early stopping with a look ahead of 5 epochs and we assumed convergence when the cost improvement on the validation set was less than $0.00001$. As for the RBM and DBM experiments, the weight matrices were initialized with random values sampled from a Gaussian with zero mean and a standard deviation of $0.01$ and the biases and offsets were initialized as described in Section~\ref{sec:initialization}. The batch size was set to 100  and the sliding factor to $0.01$.
We used a default learning rate of $0.1$ for all experiments. However, a second set of experiments was performed with a learning rate of $0.5$ when the AEs did not converge after $5000$ or with a learning rate of $0.01$ when the AEs converged rather quickly ($< 500$ epochs). 
Each experiment was repeated 10 times and we calculated the average  maximal reached cost value on test data, the corresponding standard deviation and the average number of epochs needed for convergence.

The results are given in Table~\ref{tab:AE_plain} showing that except for the RCV1 dataset centered AEs  perform clearly better in terms of the average reached cost value on the test data than normal AEs. On the training data normal AEs only perform slightly better on datasets where both models reached very small cost values anyway. We did no show the results for the validation sets since they are almost equivalent to the results for test data.

Interestingly, the result that centering only performs worse on the RCV1 dataset is fully consistent with the findings for RBMs and DBMs.
We inspected the RCV1 dataset and its first and second order statistics but did not find anything conspicuous compared to the other datasets that might explain why for this particular dataset centering is not beneficial.  However, learning is much slower for this dataset when centering is used, which can also be seen by comparing the results for learning rate $0.1$ an $0.5$ in Table~\ref{tab:AE_plain}. 

\section{Conclusion}\label{sec:conclusion}

This work discusses centering in RBMs and DBMs, where centering corresponds to subtracting offset parameters from visible and hidden variables.
Our theoretical analysis yields the following results 
\begin{itemize}
	\item[1.] Centered ANNs and normal ANNs are different parameterizations of the same model class, which justifies the use of centering in arbitrary ANNs (Section~\ref{sec:generalization}).
	\item[2.] The LL gradient of centered RBMs/DBMs is invariant under simultaneous flip of data and offsets, for any offset value in the range of zero to one. This leads to a desired invariance of the LL performance of the model under changes of data representation (Appendix~\ref{sec:ProofInv}).
	\item[3.] Training a centered RBM/DBM can be reformulated as training a normal RBM/DBM with a new parameter update, which we refer to as centered gradient (Section~\ref{sec:centeredGradient} and Appendix~\ref{sec:DerivationCentered}).
	\item[4.] From the new formulation follows that the enhanced gradient is a particular form of centering. That is, the centered gradient becomes equivalent to the enhanced gradient by setting the visible and hidden offsets to the average over model and data mean of the corresponding variable (Section~\ref{sec:centeredGradient} and Appendix~\ref{sec:CenteredToEnhanced}).  
\end{itemize}
Our numerical analysis yielded the following results 
\begin{itemize}
\item[1.]Optimal performance of centered DBMs/RBMs is achieved when both, visible and hidden variables are centered and the offsets are set to their expectations under data or model distribution. 
\item[2.] Centered RBMs/DBMs reach significantly higher LL values than normal binary RBMs/DBMs. As an example, centered RBMs with 500 hidden units achieved an average test LL of -76 on \emph{MNIST} compared to a reported value of -84 for carefully trained normal binary RBMs~\citep{Salakhutdinov-2008, SalakhutdinovMurray-2008, TangSutskever-2011, ChoRaikoEtAl-2013}.
\item[3.]Using the model expectation (as for the enhanced gradient for example) can lead to a severe divergence of the LL when PCD or PT$_c$ is used for sampling. This is caused by the correlation in offset and gradient approximation as discussed in Section~\ref{sec:divergence}. 
\item[4.] Initializing the bias parameters such that the RBM/DBM/AE is initially centered (i.e. $b_i = \sigma^{-1}(\langle x_i \rangle )$) can already improve the performance of a normal binary RBM. However, this initialization leads to a performance still worse than the performance of a centered RBM as shown in this work and is therefore no alternative to centering. 
\item[5.] The divergence can be prevented when an exponentially moving average for the approximations of the offset values is used, which also stabilized the training for other centering variants especially when the mini batch size is small.
\item[6.]Training centered RBMs/DBMs leads to smaller weight norms and larger bias norms compared to normal binary RBMs/DBMs. This supports the hypothesis that when using the standard gradient the mean value is modeled by both weights and biases, while when using the centered gradient the mean values are explicitly modeled by the bias parameters.
\item[7.]The direction of the centered gradient is closer to the natural gradient than that of the standard gradient and the natural gradient is extremely efficient for training RBMs if tractable.
\item[8.] Centered DBMs reach higher LL values than normal DBMs independent of whether pre-training is used or not. Thus pre-training cannot be considered as a replacement for centering.
\item[9.] While pre-training helped normal DBMs on \emph{MNIST} we did no observe an improvement through pre-training for centered DBMs. Furthermore, on all other datasets than \emph{MNIST} pre-training led to lower LL values and the results became worse as longer pre-training was performed for normal and centered DBMs.
\item[10.] The visual inspection of the learned filters, the average second hidden layer activities and reached LL values suggest that normal DBMs have difficulties in making use of the third and higher layers. 
\item[11.] Centering also improved the performance in terms of the optimized loss for AEs, which supports our assumption that centering is beneficial not only for probabilistic models like RBMs and DBMs.
\end{itemize}
Based on our results we recommend to center all units in the network using the data mean and to use an exponentially moving average if the mini-batch size is rather small ($<$~100 for stochastic models and $<$~10 for deterministic models). Furthermore, we do not recommend to use pre-training of DBMs since it often worsen the results.

All results clearly support the superiority of centered RBMs/DBMs and AEs, which we believe will also extend to other models.
Future work might focus on centering in other probabilistic models such as the neural auto-regressive distribution estimator ~\citep{LarochelleMurray-2011} or recurrent neural networks such as long short-term memory~\citep{HochreiterSchmidhuber-1997}.

\acks{We would like to thank Nan Wang for helpful discussions. We also would like to thank Tobias Glasmachers for his suggestions on the natural gradient part.
Asja Fischer was supported by the German Federal Ministry of Education and Research within the National 
Network Computational Neuroscience under grant number 01GQ0951 (Bernstein Fokus “Learning behavioral models: 
From human experiment to technical assistance”). }

\appendix

\section{Proof of Invariance for the Centered RBM Gradient}\label{sec:ProofInv}

In the following we show that the gradient of centered RBMs is invariant to flips of the variables if the corresponding offset parameters flip as well. 
Since training a centered RBM is equivalent to training a normal binary RBM using the centered gradient~(see Appendix~\ref{sec:DerivationCentered}), the proof also holds for the centered gradient.

We begin by formalizing the invariance property in the following definitions.

\begin{definition}
Let there be an RBM with visible variables $\vect X=(X_1,\dots, X_N)$ and hidden variables $\vect H=(H_1,\dots, H_M)$.
The variables $X_i$ and $H_j$ are called flipped if they take the values $\tilde x_i= 1- x_i$
	and $\tilde h_j= 1- h_j$ for any given states $x_i$ and $h_j$.
\end{definition}

\begin{definition}\label{def:invariance}
Let there be a binary RBM with parameters $\vect \theta$ and energy $E$ and another binary RBM with 
parameters $\tilde{\vect \theta}$ and energy $\tilde E$ where some of the variables are flipped, such that
\begin{equation} \label{eqn:sameE}
E(\vect x, \vect h)= \tilde E(\tilde{\vect x},\tilde{\vect h})\enspace,
\end{equation}
for all possible states $(\vect x,\vect h)$ and corresponding flipped states $(\tilde{\vect x},\tilde{\vect h})$, 
where $\tilde  x_i = 1-x_i$, $\tilde  h_j = 1-h_j$ if $X_i$ and $H_j$ are flipped and 
$\tilde  x_i =\ x_i$, $\tilde  h_j = h_j$ otherwise.
The gradient $\nabla \vect \theta $ is called \textbf{flip-invariant} or \textbf{invariant to the flips of the variables}
 if \eqref{eqn:sameE}
still holds after updating $\vect \theta$ and  $\tilde{\vect \theta}$
to $\vect \theta + \eta \nabla \vect \theta  $  and  $\tilde{\vect \theta} + \eta \nabla \tilde{\vect \theta}$, respectively, for an arbitrary learning rate $\eta$. 
\end{definition}

We can now state the following theorem.
\begin{theorem}\label{theorem1}
The gradient of centered RBMs  is invariant to flips of arbitrary variables 
$X_{i_1},\dots X_{i_r}$ and $H_{j_1}, \dots H_{j_s}$
with $\{ i_1, \dots, i_r\} \subset \{1,\dots, N \}$ and $\{ j_1, \dots, j_s\} \subset \{1,\dots, M \}$
 if the corresponding offset parameters $\mu_{i_1},\dots \mu_{i_r}$ and $\lambda_{j_1}, \dots \lambda_{j_s}$ flip as well that is if 
$\tilde{x}_i = 1 - x_i$ implies   $\tilde{\mu}_i = 1 - \mu_i$ and $  \tilde{h}_j =
1 - h_j$ implies $ \tilde{\lambda}_j = 1 - \lambda_j$ .\\
\end{theorem}

\begin{proof}
Let there be a centered RBM with parameters $\vect \theta $  and energy  $E$ and
another centered RBM where some of the variables are flipped with parameters $\tilde{\vect \theta}$ 
and energy $\tilde E$, such that $E(\vect x, \vect h) = \tilde E(\tilde{\vect x},\tilde{\vect h})$ 
for  any $(\vect x,\vect h )$ and corresponding $(\tilde{\vect x},\tilde{\vect h})$.
W.l.o.g. it is sufficient to show the invariance of the gradient when flipping 
only one visible variable $X_i$, one  hidden variable $H_j$, or both of them, since each
derivative with respect to a single parameter can only be affected by the flips of at most one 
hidden and one visible variable, which follows from the bipartite structure of the model.

We start by investigating how the energy changes when the variables are flipped. For this purpose we rewrite the energy in Equation~\eqref{eqn:energy} in summation notation given by
\begin{eqnarray}
  	E\left( \vect{x}, \vect{h} \right)
  	&\stackrel{ \eqref{eqn:energy}}{=}&
  	- \sum_i \left(x_i - \mu_i \right) b_i -  \sum_j \left(h_j - \lambda_j \right) c_j
  	- \sum_{ij} \left(x_i - \mu_i \right)  w_{ij}\left(h_j - \lambda_j \right)\enspace.
	\label{eqn:energyReform}
\end{eqnarray}

To indicate a variable flip we introduce the binary parameter $f_i$ that takes the value $1$ if the corresponding variable $X_i$ and the corresponding offset $\mu_i$ are flipped and $0$ otherwise.
Similarly, $g_j = 1$  if $H_j$ and $\lambda_j$ are flipped and $g_j = 0$ otherwise.
Now we use $\mathcal{E}^{f_i=1 \wedge g_j=1}$ to denote the terms of the energy \eqref{eqn:energyReform} that are affected by a flip of the variables $X_i$ and $H_j$.
Analogously, $\mathcal{E}^{f_i=1 \wedge g_j=0}$ and $\mathcal{E}^{f_i=0 \wedge g_j=1}$ denote the terms affected by a flip of either $X_i$ or $H_j$ respectively.
For flipped values $\tilde x_i$, $\tilde h_j$ these terms get
\begin{eqnarray}
\mathcal{E}^{f_i=1 \wedge g_j=1}   	&\stackrel{ \eqref{eqn:energyReform}}{=}& - (\tilde x_i - \tilde \mu_i ) b_i - (\tilde x_i - \tilde \mu_i ) \sum_{k \not = j} w_{ik}( h_k - \lambda_k ) \nonumber\\
  	&&- (\tilde h_j - \tilde \lambda_j ) c_j - (\tilde h_j - \tilde \lambda_j ) \sum_{u \not = i} w_{uj} (x_u - \mu_u ) \nonumber\\
  	&& - (\tilde x_i - \tilde \mu_i ) w_{ij}(\tilde h_j - \tilde \lambda_j )\nonumber\\
	&=&	- ((1-x_i) - (1-\mu_i) ) b_i - ((1-x_i) - (1-\mu_i) ) \sum_{k \not = j} w_{ik} (h_k - \lambda_k ) \nonumber\\
  	&&- ((1-h_j) - (1-\lambda_j) ) c_j - ((1-h_j) - (1-\lambda_j) ) \sum_{u \not = i} w_{uj} (x_u - \mu_u ) \nonumber\\
  	&& - ((1-x_i) - (1-\mu_i) ) w_{ij}((1-h_j) - (1-\lambda_j) )\nonumber\\
  	&=&  ( x_i -  \mu_i ) b_i + ( x_i -  \mu_i ) \sum_{k \not = j} w_{ik}( h_k - \lambda_k ) \nonumber\\
  	&& ( h_j -  \lambda_j ) c_j + ( h_j -  \lambda_j ) \sum_{u \not = i} w_{uj} (x_u - \mu_u ) \nonumber\\
  	&& - ( x_i -  \mu_i ) w_{ij}( h_j -  \lambda_j )\nonumber\enspace,
\end{eqnarray}
and analogously
\begin{eqnarray}
\mathcal{E}^{f_i=1 \wedge g_j=0}   	&\stackrel{ \eqref{eqn:energyReform}}{=}&
- \left(\tilde x_i - \tilde \mu_i \right) b_i - \left(\tilde x_i - \tilde \mu_i \right) \sum_{j} w_{ij}\left(h_j - \lambda_j \right) \nonumber \\
&=&
	\left(x_i -\mu_i \right) b_i + \left(x_i -\mu_i \right) \sum_{j} w_{ij}\left(h_j - \lambda_j \right)  \enspace,\nonumber
\end{eqnarray}
and
\begin{eqnarray}
\mathcal{E}^{f_i=0 \wedge g_j=1}   	&\stackrel{ \eqref{eqn:energyReform}}{=}&
	\left(h_j - \lambda_j \right) c_j + \left(h_j-\lambda_j \right) \sum_{i} w_{ij} \left(x_i - \mu_i \right) \enspace. \nonumber
\end{eqnarray}

From the fact that the terms differ from the corresponding terms in \eqref{eqn:energyReform}
only in the sign 
and that $E(\vect x, \vect h) = \tilde E(\tilde{\vect x},\tilde{\vect h})$ holds for any $(\vect x,\vect h )$ and corresponding $(\tilde{\vect x},\tilde{\vect h})$, it follows that the parameters $\tilde{\vect \theta}$ must be given by
\begin{eqnarray}
	\tilde{w}^{f_i \wedge g_j}_{ij} &=& (-1)^{f_i + g_j} w_{ij}\enspace,\label{eqn:change_W}\\
  \tilde{b}^{f_i \wedge g_j}_i &=& (-1)^{f_i } b_i \enspace,\label{eqn:change_b}\\
  \tilde{c}^{f_i \wedge g_j}_j &=& (-1)^{g_j} c_j \enspace,\label{eqn:change_c}\\
  \tilde{\mu}^{f_i \wedge g_j}_i &=& \mu_i\nonumber\enspace,\\
  \tilde{\lambda}^{f_i \wedge g_j}_j &=& \lambda_j\nonumber\enspace.
\end{eqnarray}
The LL gradient for the model without flips is given by Equations \eqref{eqn:grad_W} - \eqref{eqn:grad_c}. We now consider the LL gradients for the three possible flipped versions. 
If $X_i$ and $H_j$ are flipped the derivatives w.r.t. $w_{ij}, b_i,$, and $c_j$ are given by
\begin{eqnarray}
	\nabla \tilde{w}^{f_i=1 \wedge g_j=1}_{ij} &=& \langle (1-x_i-(1-\mu_i))(1-h_j-(1-\lambda_j)) \rangle_{d} \nonumber\\
	&& - \langle (1-x_i-(1-\mu_i))(1-h_j-(1-\lambda_j)) \rangle_{m}\nonumber\\
	&=& \langle (-x_i+\mu_i)(-h_j+\lambda_j) \rangle_{d}  - \langle (-x_i+\mu_i)(-h_j+\lambda_j) \rangle_{m}\nonumber\\
	&=& \langle (x_i-\mu_i)(h_j-\lambda_j) \rangle_{d}  - \langle (x_i-\mu_i)(h_j-\lambda_j) \rangle_{m}\nonumber\\
	&=& (-1)^{1 + 1} \nabla w_{ij}\nonumber\enspace,\\
	\nabla \tilde{b}^{f_i=1 \wedge g_j=1}_i &=& \langle 1-x_i - (1-\mu_i)\rangle_{d} - \langle 1-x_i - (1-\mu_i) \rangle_{m} \nonumber\\
	&=& -\langle x_i\rangle_{d} + \mu_i + \langle x_i \rangle_{m} -\mu_i \nonumber\\
    &=& (-1)^{1}\nabla b_i\nonumber \enspace,
	\label{eqn:gradbf}\\
	\nabla\tilde{c}^{f_i=1 \wedge g_j=1}_j &=& \langle 1 - h_j - (1-\lambda_j)\rangle_{d} - \langle 1 - h_j -(1-\lambda_j)\rangle_{m} \nonumber\\
	&=& -\langle h_j\rangle_{d} + \lambda_j + \langle h_j \rangle_{m} -\lambda_j \nonumber\\
    &=& (-1)^{1}\nabla c_j\nonumber \enspace.
	\label{eqn:gradcf}
\end{eqnarray}
If $X_i$ is flipped they are given by
\begin{eqnarray}
	\nabla \tilde{w}^{f_i=1 \wedge g_j=0}_{ij} &=& \langle (1-x_i-(1-\mu_i))(h_j-\lambda_j) \rangle_{d} \nonumber\\&&- \langle (1-x_i-(1-\mu_i))(h_j-\lambda_j) \rangle_{m}\nonumber\\
	&=& \langle (-x_i+\mu_i)(h_j-\lambda_j) \rangle_{d} - \langle (-x_i+\mu_i)(h_j-\lambda_j) \rangle_{m}\nonumber\\
	&=& -\left( \langle (x_i-\mu_i)(h_j-\lambda_j) \rangle_{d} - \langle (x_i-\mu_i)(h_j-\lambda_j) \rangle_{m}\right) \nonumber\\
	&=& (-1)^{1+0}\nabla w_{ij}\nonumber\enspace,\\
	\nabla \tilde{b}^{f_i=1 \wedge g_j=0}_i & = & \nabla \tilde{b}^{f_i=1 \wedge g_j=1}_i \nonumber\\ 
	&=& (-1)^{1}\nabla b_i\nonumber\enspace,\\
	\nabla \tilde{c}^{f_i=1 \wedge g_j=0}_j & = &\nabla \tilde{c}^{f_i=0 \wedge g_j=0}_i\nonumber\\ 
	&=& (-1)^{0}\nabla c_j \nonumber\enspace,
\end{eqnarray}
and due to the symmetry of the model the derivatives if $H_j$ is flipped are given by
\begin{eqnarray}
	\nabla \tilde{w}^{f_i=0 \wedge g_j=1}_{ij} 
	&=&  (-1)^{0+1}\nabla w_{ij}\nonumber\enspace,\\
	\nabla \tilde{b}^{f_i=0 \wedge g_j=1}_i 
	&=& (-1)^{0}\nabla b_i\nonumber\enspace,\\
	\nabla \tilde{c}^{f_i=0 \wedge g_j=1}_j 
	&=& (-1)^{1}\nabla c_j \nonumber\enspace.
\end{eqnarray}
Comparing the results with Equations~\eqref{eqn:change_W}~-~\eqref{eqn:change_c} shows that the gradient underlies the same sign changes under variable flips as the parameters. Thus, it holds for the updated parameters that
\begin{eqnarray}
  \tilde{w}^{f_i \wedge g_j}_{ij} + \eta \nabla \tilde{w}^{f_i \wedge g_j}_{ij}   
  &\stackrel{\eqref{eqn:change_W}}{=}& (-1)^{f_i + g_j} \left( w_{ij} + \eta \nabla  w_{ij} \right)\nonumber\enspace,\\
  \tilde{b}^{f_i \wedge g_j}_i + \eta \nabla  \tilde{b}^{f_i \wedge g_j}_i 
  &\stackrel{\eqref{eqn:change_b}}{=}& (-1)^{f_i + g_j} \left( b_{i} + \eta \nabla  b_{i} \right)\nonumber\enspace,\\
  \tilde{c}^{f_i \wedge g_j}_j + \eta \nabla  \tilde{c}^{f_i \wedge g_j}_j
  &\stackrel{\eqref{eqn:change_c}}{=}& (-1)^{f_i + g_j} \left( c_{j} + \eta \nabla  c_{j} \right)\nonumber\enspace,
\end{eqnarray}
showing that  $E(\vect x, \vect h) = \tilde E(\tilde{\vect x},\tilde{\vect h})$ is still guaranteed
and thus that the gradient of centered RBMs is flip-invariant according to Definition~\ref{def:invariance}. 
\end{proof}

Theorem~\ref{theorem1} holds for any value from zero to one for $\mu_i$ and $\lambda_j$, if it is guaranteed that the offsets flip simultaneously with the corresponding variables. 
In practice one wants the model to perform equivalently on any flipped version of the dataset without knowing which version is presented. 
This holds if we set the offsets to the expectation value of the corresponding variables under any distribution, since when $\mu_i = \sum_{x_i}
p\left(x_i\right)x_i$, flipping $X_i$ leads to $\tilde{\mu_i} = \sum_{x_i}
p\left(x_i\right)\left(1-x_i\right) = 1 - \sum_{x_i} p\left(x_i\right)x_i = 1 -
\mu_i$ and similarly for $\lambda_j$, $h_j$.

Due to the structural similarity this proof also holds for DBMs. By replacing $\vect{x}$ by $\vect{h^l}$ (which denotes the state of the variables in the $l$th hidden layer) and $\vect{h}$ by 
$\vect{h^{l+1}}$ (denoting the state of the variables in the $l+1$th hidden layer)
we can prove the invariance property for the derivatives of the parameters in layer $l$ and $l+1$.

\section{Derivation of the Centered Gradient}\label{sec:DerivationCentered}

In the following we show that the gradient of centered RBMs can be reformulated as an alternative update for the parameters of a normal binary RBM, which we name `centered gradient'.

A normal binary RBM with energy $E(\vect{x}, \vect{h}) = - \vect{x}^T\vect{b} - \vect{c}^T\vect{h} - \vect{x}^T\vect{W}\vect{h}$ can be transformed into a centered
RBM with energy $\tilde{E}(\vect{x}, \vect{h}) = 	- \left(\vect{x} - \vect{\mu} \right)^T \tilde{\vect{b}}
  	- \tilde{\vect{c}}^T\left(\vect{h} - \vect{\lambda} \right)
  	- \left(\vect{x} - \vect{\mu} \right)^T\tilde{\vect{W}}\left(\vect{h} - \vect{\lambda} \right)$ by
  	the following parameter transformation  
\begin{eqnarray}
	\tilde{\vect{W}} &\stackrel{\eqref{eqn:transform_W}}{=}& \vect{W} \label{eqn:transformN_W}\enspace,\\
	\tilde{\vect{b}} &\stackrel{\eqref{eqn:transform_b}}{=}&
	\vect{b} +  \vect{W} \vect{\lambda} \label{eqn:transformN_b}\enspace,\\
  \tilde{\vect{c}} &\stackrel{\eqref{eqn:transform_c}}{=}&
  \vect{c} + \vect{W}^T \vect{\mu}\label{eqn:transformN_c}\enspace,
  \end{eqnarray}
which guarantees that $E(\vect{x}, \vect{h}) =  \tilde{E}(\vect{x}, \vect{h}) + const$ 
for all $(\vect x, \vect h) \in \{0,1\}^{n+m}$ and thus that the modeled distribution stays the same.

Updating the parameters of the centered RBM according to  Eq.~\eqref{eqn:grad_W} -- \eqref{eqn:grad_c} with a learning rate $\eta$ leads to an updated set of parameters $\tilde{\vect W}_u, \tilde{\vect b}_u, \tilde{\vect c}_u$ given by
 \begin{eqnarray} 
	\tilde{\vect{W}_u} &\stackrel{\eqref{eqn:grad_W}}{=}&  \tilde{\vect{W}}  + \eta (  \langle ( \vect{x} - \vect{\mu} ) (\vect{h} - \vect{\lambda})^T \rangle_{d}  
	  - \langle ( \vect{x} - \vect{\mu} ) (\vect{h} - \vect{\lambda})^T \rangle_{m})  \label{eqn:update_W}\enspace,\\
	\tilde{\vect{b}}_u &\stackrel{\eqref{eqn:grad_b}}{=}&
	 \tilde{\vect{b}} + \eta (\langle  \vect{x} \rangle_{d} - \langle  \vect{x} \rangle_{m} )  \label{eqn:update_b} \enspace,\\
  \tilde{\vect{c}}_u &\stackrel{\eqref{eqn:grad_c}}{=}&
    \tilde{\vect{c}}+ \eta ( \langle \vect{h}  \rangle_{d} - \langle \vect{h} \rangle_{m}) \label{eqn:update_c} \enspace .
  \end{eqnarray}
One can now transform the updated centered RBM back to a normal RBM by applying the inverse transformation
 to the updated parameters, which finally leads to the centered gradient.
 \begin{eqnarray}
\vect{W}_u &\stackrel{\eqref{eqn:transformN_W}}{=}& \tilde{\vect{W}}_u \nonumber \\
&\stackrel{ \eqref{eqn:transformN_W},\eqref{eqn:update_W}}{=}& \vect{W}  + \eta \underbrace{(  \langle ( \vect{x} - \vect{\mu} ) (\vect{h} - \vect{\lambda})^T \rangle_{d})  
	  - \langle ( \vect{x} - \vect{\mu} ) (\vect{h} - \vect{\lambda})^T \rangle_{m})}_{\stackrel{\eqref{eqn:cgradW}}{=}\nabla_c{\vect{W}}}\label{eqn:new_W} \enspace,\\	   
\vect{b}_u &\stackrel{\eqref{eqn:transformN_b}}{=}& \tilde{\vect{b}}_u -  \vect{W}_u \vect{\lambda} \nonumber\\
	&\stackrel{\eqref{eqn:update_b},\eqref{eqn:new_W}}{=}& \tilde{\vect{b}} + \eta (\langle  \vect{x} \rangle_{d} - \langle  \vect{x} \rangle_{m} ) -  (\vect{W} + \eta \nabla_c{\vect{W}} ) 
	 \vect{\lambda} \nonumber\\
	 &\stackrel{\eqref{eqn:transformN_b}}{=}& \vect{b} +  \vect{W} \vect{\lambda} + \eta (\langle  \vect{x} \rangle_{d} - \langle  \vect{x} \rangle_{m} ) -  \vect{W}  \vect{\lambda}  - \eta \nabla_c{\vect{W}}  \vect{\lambda} 
	\nonumber\\
	 &=& \vect{b}  + \eta \underbrace{(\langle  \vect{x} \rangle_{d} - \langle  \vect{x} \rangle_{m}  - \nabla_c{\vect{W}}  \vect{\lambda )}}_{\stackrel{\eqref{eqn:cgradb}}{=}\nabla_c{\vect{b}}} \enspace,
	\label{eqn:new_b}\\
   \vect{c}_u &\stackrel{\eqref{eqn:transformN_c}}{=}&  \tilde{\vect{c}}_u - \vect{W}_u^T \vect \mu \nonumber \\
  &\stackrel{\eqref{eqn:update_c},\eqref{eqn:new_W}}{=}& \tilde{\vect{c}}+ \eta ( \langle \vect{h}  \rangle_{d} - \langle \vect{h} \rangle_{m}) - (\vect{W} + \eta \nabla_c{\vect{W}} ) 
	 \vect{\mu} \nonumber \\
	 &\stackrel{\eqref{eqn:transformN_c}}{=}& \vect{c} +  \vect{W} \vect{\mu} + \eta (\langle  \vect{h} \rangle_{d} - \langle  \vect{h} \rangle_{m} ) -  \vect{W}  \vect{\mu}  - \eta \nabla_c{\vect{W}}  \vect{\mu} 
	\nonumber\\
	 &=& \vect{c}  + \eta \underbrace{(\langle  \vect{h} \rangle_{d} - \langle  \vect{h} \rangle_{m}  - \nabla_c{\vect{W}}  \vect{\mu )}}_{\stackrel{\eqref{eqn:cgradc}}{=}\nabla_c{\vect{c}}} \enspace.\label{eqn:new_c}
  \end{eqnarray}
The braces in Equation~\eqref{eqn:new_W}~-~\eqref{eqn:new_c} mark the centered gradient given by Equations~\eqref{eqn:cgradW}~-~\eqref{eqn:cgradc}.

\section{Enhanced Gradient as Special Case of the Centered Gradient} \label{sec:CenteredToEnhanced}
In the following we show that the enhanced gradient can be derived as a special case of the centered gradient. By 
setting $\vect{\mu} = \frac{1}{2} \left(\langle  \vect{x} \rangle_{d} + \langle  \vect{x} \rangle_{m} \right)$ and $\vect{\lambda} = \frac{1}{2} \left(\langle  \vect{h} \rangle_{d} + \langle  \vect{h} \rangle_{m} \right)$ we get
\begin{eqnarray}
\nabla_c \vect{W} &\stackrel{\eqref{eqn:cgradW}}{=}&  \langle ( \vect{x} - \vect{\mu} ) (\vect{h} - \vect{\lambda})^T \rangle_{d} - \langle ( \vect{x} - \vect{\mu} ) (\vect{h} - \vect{\lambda})^T \rangle_{m} \nonumber \\
&=& \langle \vect x \vect h^T  \rangle_d  
- \langle  \vect x  \rangle_d  \vect \lambda^T 
- \vect \mu \langle  \vect h^T  \rangle_d  
+ \vect{\mu} \vect \lambda^T  
- \langle \vect x \vect h^T  \rangle_m 
+\langle  \vect x  \rangle_m \vect \lambda^T 
+ \vect \mu \langle  \vect h^T  \rangle_m 
- \vect{\mu} \vect \lambda^T \nonumber\\
&=&  \langle \vect x \vect h^T  \rangle_d  
- \frac{1}{2}\langle  \vect x  \rangle_d   \left(\langle  \vect{h} \rangle_{d} + \langle  \vect{h} \rangle_{m} \right)^T 
- \frac{1}{2} \left(\langle  \vect{x} \rangle_{d} + \langle  \vect{x} \rangle_{m} \right) \langle  \vect h^T  \rangle_d  \nonumber\\
&&
- \langle \vect x \vect h^T  \rangle_m 
+ \frac{1}{2}\langle  \vect x  \rangle_m \left(\langle  \vect{h} \rangle_{d} + \langle  \vect{h} \rangle_{m} \right)^T 
+ \frac{1}{2} \left(\langle  \vect{x} \rangle_{d} + \langle  \vect{x} \rangle_{m} \right) \langle  \vect h^T  \rangle_m  \nonumber\\
\nonumber\enspace
&=&  \langle \vect x \vect h^T  \rangle_d  
- \frac{1}{2}\langle  \vect x  \rangle_d \langle  \vect{h}^T  \rangle_{d} - \frac{1}{2} \langle  \vect x  \rangle_d \langle  \vect{h}^T  \rangle_{m}
- \frac{1}{2} \langle  \vect{x} \rangle_{d} \langle  \vect h^T  \rangle_d - \frac{1}{2} \langle  \vect{x} \rangle_{m} \langle  \vect h^T  \rangle_d  \nonumber\\
&&
- \langle \vect x \vect h^T  \rangle_m 
+ \frac{1}{2}\langle  \vect x  \rangle_m \langle  \vect{h}^T  \rangle_{d} + \frac{1}{2}\langle  \vect x  \rangle_m \langle  \vect{h}^T  \rangle_{m}
+ \frac{1}{2} \langle  \vect{x} \rangle_{d} \langle  \vect h^T  \rangle_m + \frac{1}{2} \langle  \vect{x} \rangle_{m}  \langle  \vect h^T  \rangle_m  \nonumber\\
\nonumber\enspace
&=&  \langle \vect x \vect h^T  \rangle_d  
- \langle  \vect x  \rangle_d \langle  \vect{h}^T  \rangle_{d} 
- \langle \vect x \vect h^T  \rangle_m 
+ \langle  \vect x  \rangle_m \langle  \vect{h}^T  \rangle_{m}   \nonumber\\
\nonumber\enspace
&=&  \langle ( \vect{x} - \langle  \vect x  \rangle_d ) (\vect{h} - \langle  \vect h \rangle_d)^T \rangle_{d} - \langle ( \vect{x} - \langle  \vect x  \rangle_m ) (\vect{h} - \langle  \vect h  \rangle_m)^T \rangle_{m} \nonumber \\
&\stackrel{\eqref{enhanced_W}}{=}&\nabla_e \vect{W} \nonumber\enspace ,
\end{eqnarray}
and for the derivatives with respect to the bias parameters follows directly that
\begin{eqnarray}
\nabla_c \vect{b} &\stackrel{\eqref{eqn:cgradb}}{=}&
\langle  \vect{x} \rangle_{d} - \langle  \vect{x} \rangle_{m} - \nabla_e \vect{W} \vect \lambda \nonumber\\
&=& \langle  \vect{x} \rangle_{d} - \langle  \vect{x} \rangle_{m} - \nabla_e \vect{W} \frac{1}{2} \left(\langle  \vect{h} \rangle_{d} + \langle  \vect{h} \rangle_{m} \right) \nonumber\\
&\stackrel{\eqref{enhanced_b}}{=}&\nabla_e \vect{b} \nonumber\enspace,\\
\nabla_c \vect{c} &\stackrel{\eqref{eqn:cgradc}}{=}&
 \langle \vect{h}  \rangle_{d} - \langle \vect{h} \rangle_{m} - {\nabla_e \vect{W}}^T \vect \mu \nonumber\\
&=&\langle \vect{h}  \rangle_{d} - \langle \vect{h} \rangle_{m} - {\nabla_e \vect{W}}^T \frac{1}{2} \left(\langle  \vect{x} \rangle_{d} + \langle  \vect{x} \rangle_{m} \right)\nonumber\\
&\stackrel{\eqref{enhanced_c}}{=}&\nabla_e \vect{c} \nonumber\enspace .
\end{eqnarray} 

\bibliography{centering_RBM}

\end{document}